%% file: main.tex
\documentclass{article}
\usepackage{geometry}
\usepackage{times}
\usepackage{amsmath, amsfonts, amssymb, amsthm}
\usepackage{graphicx, subfigure}
\usepackage{algorithm}
\usepackage[noend]{algpseudocode}
\usepackage{hyperref}
\usepackage{url}

\newtheorem{Theorem}{Theorem}
\newtheorem*{Theorem*}{Theorem}
\newtheorem{Definition}{Definition}
\newtheorem{Lemma}{Lemma}

\newtheorem*{Remark}{Remark}
\newtheorem*{Proof}{proof}

\newtheorem{claim}{Claim}
\newcommand{\tu}{\mathbf{u}}
\newcommand{\tv}{\mathbf{v}}
\newcommand{\ty}{\mathbf{y}}
\newcommand{\tw}{\mathbf{w}}
\newcommand{\tp}{\mathbf{p}}
\newcommand{\ts}{\mathbf{s}}
\newcommand{\e}{\mathbf{e}}
\newcommand{\w}{\mathbf{w}}
\newcommand{\poly}{\mathrm{poly}}
\newcommand{\act}{h}

\title{Making Method of Moments Great Again? -- How can GANs learn distributions}

\author{Yuanzhi Li$^1$, Zehao Dou$^2$\\
{$^1$ Machine Learning Department, Carnegie Mellon University}\\
{\texttt{yuanzhil@andrew.cmu.edu}}\\
{$^2$ Department of Statistics and Data Science, Yale University}\\
{\texttt{zehao.dou@yale.edu}}
}

\date{}
\begin{document}
\maketitle
\allowdisplaybreaks[4]
\begin{abstract}
\input{abstract.tex}
\end{abstract}
\newpage
\tableofcontents
\newpage

% \begin{keywords}%
%   Method of Moments, GAN, Convergence, Distribution, Tensor Decomposition.
% \end{keywords}
% \vspace{-6mm}
\section{Introduction}
%\vspace{-3mm}
\input{intro.tex}
\section{Related works}
\input{related.tex}
%\vspace{-2mm}
\section{Preliminary}
\input{preliminary.tex}
%\vspace{-2mm}
\section{Statement of the main result}
\input{statement.tex}

%\vspace{-2mm}

\section{Warm Up: two-layer Generator with Cubic Activations}
% \vspace{-2mm}
To present the simplest form of our result, we first consider the situation where both the learner and target generators are two-layer networks with cubic activations (or cubic generators for short). Denote $G,G^{*}:\mathbb{R}^{d}\rightarrow\mathbb{R}^{D}$ be the learner and target generators with the following form:
\[G_{k}(\mathbf{\omega})=\sum_{i=1}^{r}\alpha_{k}^{(i)}(\tv_{k}^{(i)}\cdot \mathbf{\omega})^{3},~ G^{*}_{k}(\mathbf{\omega})=\sum_{i=1}^{r}\alpha_{k}^{*(i)}(\tv_{k}^{*(i)}\cdot \mathbf{\omega})^{3}~~(k=1,2,\cdots,D)\]
Here, both $\tv_{k}^{(i)}~(i=1,2,\cdots,r)$ and $\tv_{k}^{*(i)}~(i=1,2,\cdots,r)$ are orthonormal vector groups, and the weights of the ground truth generator $\alpha_{k}^{*(i)}>0$. (This assumption is reasonable since we can shift the direction of $\tv$ if its corresponding $\alpha<0$). The input vector of generators is drawn from normal Gaussian distribution: $\omega\sim \mathcal N(0,I_{d})$. Firstly, we introduce a concept of $(\tau,\mathcal A)$-robustness.
\begin{Definition}
[$(\tau,\mathcal A)$-robustness] A rank-$r$ cubic generator $G_{k}^*$ is called $(\tau,\mathcal A)$-robust if its weights $\alpha_{k}^{*(i)}~i\in [r], k\in[D]$ satisfy the following conditions:
\begin{equation*}
\begin{aligned}
% & \forall k\in[D]:~~\alpha_{k}^{*(1)}<\alpha_{k}^{*(2)}<\cdots<\alpha_{k}^{*(r)}\\
& \forall i\in [r], k\in [D]:~~\tau < \alpha_{k}^{*(i)} < \mathcal A, \quad \forall i\neq j\in [r], k\in [D]:~~ |\alpha_{k}^{*(i)}-\alpha_{k}^{*(j)}|>\tau
\end{aligned}
\end{equation*}
\end{Definition}
In the following parts, we assume that the ground truth generator satisfies the $(\tau,\mathcal A)$-robustness, and furthermore we prove that the learner generator can uniquely determine the target generator only by polynomial discriminators. We will also give the corresponding sample complexity which means the number of inputs we need to sample to make the learner generator creates a distribution that is $\epsilon$-approximated to the ground truth:
\[W_{1}\left(G_{\#}\mathcal N(0,I_{d}), ~G^{*}_{\#}\mathcal N(0,I_{d})\right)<\epsilon\]
Here, $W_{1}$ stands for the Wasserstein-1 distance between distributions and $\#$ stands for the pushforward measure.
% $\alpha_{k}^{*(i)}$ are sampled i.i.d from Gaussian distribution $\mathcal N(0,\sigma^{2})$ and $V_{k}^*=(\tv_{k}^{*(1)},\tv_{k}^{*(2)},\cdots,\tv_{k}^{*(r)})$ are sampled i.i.d from uniform distribution on the space:
% \[H^{r}_{d} = \{H\in\mathbb{R}^{d\times r}|~H^{T}H=I_{r}\}\]

% \[\alpha_{k}^{(1)}\geqslant\alpha_{k}^{(2)}\geqslant\cdots\geqslant\alpha_{k}^{(r)}\geqslant 0, ~\alpha_{k}^{*(1)}\geqslant\alpha_{k}^{*(2)}\geqslant\cdots\geqslant\alpha_{k}^{*(r)}\geqslant 0\]
% \vspace{-4mm}
\subsection{Identifiability with Polynomial Discriminator}
% \vspace{-1mm}
\subsubsection{Intro-component Moment Analysis}
% \vspace{-2mm}
\begin{Theorem}
\label{intro-component}
For $\forall k\in[D]$, if $M^{i}(G_{k})=M^{i}(G_{k}^{*})$ for $i = 2,4,\cdots,2r$, then there exists a permutation $\sigma:[r]\rightarrow[r]$ such that: $\alpha_{k}^{(\sigma(i))}=\alpha_{k}^{*(i)}$ holds for $\forall i\in[r]$.
Here: 
\[M^{i}(G_{k})=\mathop{\mathbb{E}}\limits_{\omega\sim \mathcal N(0,I)}[G_{k}(\omega)]^{i},~M^{i}(G_{k}^{*})=\mathop{\mathbb{E}}\limits_{\omega\sim \mathcal N(0,I)}[G_{k}^{*}(\omega)]^{i}\]
are the order-$i$ moments of variables $G_{k}(\omega)$ and $G_{k}^{*}(\omega)$.
\end{Theorem}
% \vspace{-5mm}
\begin{Remark}
The theorem above means that for any single node, if its moments of order $2,4,\cdots,2r$ are identified, then its coefficients $\alpha_{k}^{(i)}~(i=1,2,\cdots,r)$ are also uniquely determined. It's worth mentioned that we only use moments of even number order because for each odd number $i$, it's obvious that $M^{i}(G_{k})=M^{i}(G^{*}_{k})=0$ because both $[G_{k}(x)]^{i}$ and $[G_{k}^{*}(x)]^{i}$ are homogeneous polynomials with order $3i$, which is an odd number. It leads to the zero expectation of variables $G_{k}(\omega)$ and $G_{k}^{*}(\omega)$. Therefore, only when $i$ is even, the equation $M^{i}(G_{k})=M^{i}(G_{k}^{*})$ is nontrivial.
\end{Remark}
% \vspace{-5mm}
\subsubsection{Inter-component Moment Analysis}
In the section above, we conduct moment analysis for each component, which proves the identifiability of weights $\alpha_{k}^{(i)}~k\in[D],i\in[r]$. In this section, we will show the relationship between $\tv_{k}^{(i)}$ and $\tv_{k}^{*(i)}$ by calculating the inter-component moments. Before that, we introduce a concept of Cubic Expectation Matrix (CE-Matrix) and give an assumption of its invertibility. 
\begin{Definition}
[CE-Matrix] Given a dimension $r$, $r$ coefficients $\lambda_{1},\lambda_{2},\cdots,\lambda_{r}$ and $r$ independent random variables $\omega_{1},\omega_{2},\cdots,\omega_{r}$ which follows the standard Gaussian distribution. Denote: \[Q=\lambda_{1}\omega_{1}^{3}+\lambda_{2}\omega_{2}^{3}+\cdots+\lambda_{r}\omega_{r}^{3}\]
Then we list all the 3-order and 1-order monomials with variables $\omega_{1},\cdots,\omega_{r}$. It's not difficult to know that there are altogether $r+\binom{r+2}{3}=\frac{1}{6}r(r^{2}+3r+8)\triangleq K_r$ different monomials. We denote them as: $P_{1},P_{2},\cdots,P_{K_r}$. 
\[P_{1}=\omega_{1}^{3},\cdots,P_{r}=\omega_{r}^{3},P_{r+1}=\omega_{1}^{2}\omega_{2},\cdots, P_{K_r}=\omega_{r}\]
The CE-Matrix is a $K_r\times K_r$ polynomial matrix in the following form:
\[T_{ij}=\mathop\mathbb{E}\limits_{\omega\sim \mathcal N(0,1)}\left[P_{i}\cdot Q^{2j-1}\right]~~i,j\in[K_r]\]
We mark this CE-Matrix as: $T=CE[\lambda_{1},\lambda_{2},\cdots,\lambda_{r}]$.
\end{Definition}

\paragraph{Invertibility of CE-Matrix}
For any dimension $2\leqslant r\leqslant 50$, we have verified by computer that the polynomial matrix above is invertible. Or in other words, the determinant of CE-Matrix isn't equal to zero polynomial.
\begin{equation}\label{eqn:generic}
\mathrm{det}\left(CE[\lambda_{1},\lambda_{2},\cdots,\lambda_{r}]\right)\neq 0
\end{equation}
We believe that CE-Matrix is always invertible for every $r \geq 2$. When the inequality above holds, we say that this $r$ satisfies the 3-\textbf{generic condition}. 
% \vspace{-2mm}
\begin{Theorem}
According to the invertibility of CE-Matrix, once $\alpha_{k}^{(i)}=\alpha_{k}^{*(i)}$ holds for $\forall k\in[D],i\in[r]$, assume for $\forall i\neq j\in[D]$:
\[\mathop\mathbb{E}\limits_{\omega\sim\mathcal N(0,I)}\left(G_{i}(\omega)\right)^{K}G_{j}(\omega)=\mathop\mathbb{E}\limits_{\omega\sim\mathcal N(0,I)}\left(G_{i}^*(\omega)\right)^{K}G_{j}^*(\omega)\]
holds for $K=1,3,\cdots,2K_{r}-1$. Then we have:
$\tv_{i}^{(a)}\cdot\tv_{j}^{(b)}=\tv_{i}^{*(a)}\cdot\tv_{j}^{*(b)}$
holds for $a,b\in[r]$.
\end{Theorem}
% \vspace{-4mm}
\subsection{Sample Complexity}
\input{sample_complexity.tex}
% \vspace{-5mm}
\section{Extension of Cubic Activations to Higher Order Activations}
% \vspace{-2mm}
Actually, we can use very similar approaches to prove the identifiability and sample complexity of higher degree generators. Firstly, we list $r$ intro-component moments for each component and uniquely determine the parameters $\alpha_{k}^{(i)}$ by Vieta's Theorem. Then, we list $K_{r}^{(d)}$ inter-component moments for each component pair. Here, the definition of $K_{r}^{(d)}$ is similar to the one above, and:
\[K_{r}^{(d)}=\sum_{i=1}^{d}\binom{r+2i-2}{2i-1}\]
Then we can get a tensor decomposition form of the inner products of each vector pairs (like the $\tv_{k}^{(a)}$ and $\tv_{l}^{(b)}$) above. And finally, due to the uniqueness of tensor decomposition and Lemma \ref{inner_product}. We can uniquely determine the learner generator to be exactly the same as target generator.

After extending our generators to be with higher order, we consider the situation where both the learner and target generators are order-$p$ rank $r < d$ orthogonal tensor multiplication functions. Denote $G,G^{*}:\mathbb{R}^{d}\rightarrow\mathbb{R}^{D}$ be the learner and target generator functions with the following form:
\[G_{k}(\mathbf{\omega})=\sum_{i=1}^{r}\alpha_{k}^{(i)}(\tv_{k}^{(i)}\cdot \mathbf{\omega})^{p},~ G^{*}_{k}(\mathbf{\omega})=\sum_{i=1}^{r}\alpha_{k}^{*(i)}(\tv_{k}^{*(i)}\cdot \mathbf{\omega})^{p}~~(k=1,2,\cdots,D)\]
Here, both $\tv_{k}^{(i)}~(i=1,2,\cdots,r)$ and $\tv_{k}^{*(i)}~(i=1,2,\cdots,r)$ are orthonormal vector groups, and the weights of the ground truth generator $\alpha_{k}^{*(i)}>0$. The input vector of generators is drawn from normal Gaussian distribution: $\omega\sim \mathcal N(0,I_{d})$. Similarly, we assume $G^*$ satisfies the $(\tau,\mathcal A)$-robustness. With almost the same mathematical techniques, we can guarantee its identifiability with polynomial discriminators (or in other words, under moment analysis). Just like cubic generators, we can also calculate its corresponding sample complexity.

\begin{Theorem}
[Intro-component Moment Analysis] For each $k\in[D]$, we list $r$ equations:
\[M^{i}(G_k)=\frac{1}{N}\sum_{t=1}^{N}\left(G_{k}^*(\omega^{(t)})\right)^{i}=\overline{M^{i}}(G_{k}^*)~~~(i=2,4,\cdots,2r)\]
After solving these equations, we can get a unique learner generator $G_{k}$ which satisfies:
\[\left(\sum_{i=1}^{r}|\alpha_{k}^{(i)}-\alpha_{k}^{*(i)}|^{2}\right)^{1/2}<\tilde{O}\left(\frac{1}{\sqrt{N}}\cdot \poly\left(p,r,\frac{\mathcal A}{\tau},\frac{1}{\delta}\right)^{\poly(p,r)}\right)\]
holds for $\forall k\in [D]$ with probability at least $1-\delta$ over the choice of samples $\omega^{(t)}$.
\end{Theorem}

\begin{Theorem}
[Inter-component Moment Analysis] For each $i\neq j\in[D]$, we list $K_{r}^{(p)}$ equations:
\[M^{k,1}(G_{i},G_{j})=\frac{1}{N}\sum_{t=1}^{N}(G_{i}^*(\omega^{(t)}))^{k}\cdot(G_{j}^*(\omega^{(t)}))=\overline{M^{k,1}}(G_{i}^*,G_{j}^*)~~(k=1,3,\cdots,2K_{r}^{(p)}-1)\]
After solving these equations, we can get a learner generator $G_{k}$, such that: with probability larger than $1-\delta$ over the choice of $\alpha_{k}^{*(i)}, \tv_{k}^{*(i)}~k\in[D],i\in[r]$, it holds that:
\[|\tv_{i}^{(a)}\cdot\tv_{j}^{(b)}-\tv_{i}^{*(a)}\cdot\tv_{j}^{*(b)}|<\tilde{O}\left(\frac{\poly(d)}{\sqrt{N}}\cdot \poly\left(p,r,\frac{\mathcal A}{\tau},\frac{1}{\delta},\frac{1}{\sigma}\right)^{\poly(p,r)}\right)\]
\end{Theorem}
\begin{Theorem}
[Main theorem for Polynomial Generators] With probability at least $1-D^{2}\delta$ over the choice of the target generator $G^*$, we can efficiently obtain a learner generator $G$ only by polynomial discriminators, such that:
\[W_{1}(G_{\#}\mathcal N(0,I), G^*_{\#}\mathcal N(0,I))
< \tilde{O}\left(\sqrt{D}\cdot \frac{\poly(d)}{N^{1/4}}\cdot \poly\left(p,r,\frac{\mathcal A}{\tau},\frac{1}{\delta},\frac{1}{\sigma}\right)^{\poly(p,r)}\right)\]
Here, just like the cubic occasion, each $\alpha_{k}^{*(i)}$ is sampled independently by distribution $\mathcal N(M,\sigma)$. Each $\tv_{k}=(\tv_{k}^{*(1)}, \tv_{k}^{*(2)},\cdots,\tv_{k}^{*(r)})\in\mathbb{R}^{d\times r}$ is the first $r$ column of an arbitrary orthogonal matrix. $M$ is a constant positive integer. We also assume that the target generator satisfies the $(\tau,\mathcal A)$-robustness. This theorem leads to the sample complexity estimation we need:
\[N=\tilde{\Omega}\left(\poly\left(p,r,\frac{\mathcal A}{\tau},\frac{1}{\delta},\frac{1}{\sigma}\right)^{\poly(p,r)}\cdot \poly\left(D,d,\frac{1}{\epsilon}\right)\right)\]
\end{Theorem}
% \vspace{-7mm}
\section{Experimental Results}
% \vspace{-2mm}
\input{experiments.tex}
% \vspace{-5mm}
\section{Discussion}
\input{discussion.tex}

\bibliographystyle{plain}
\bibliography{reference}

\newpage
\appendix
\section{Other Forms of Generators and Discriminators}
\input{generator.tex}

\section{Omitted Proofs in Section 6}
\subsection{Some Simple Lemmas and Properties}
Before we start to prove our main theorems, we give some useful properties and lemmas first.
\begin{Lemma}
\label{gaussian-elementary}
Assume random variable $\omega$ is sampled from the standard Gaussian distribution with dimension $n$, then for any orthogonal matrix $G\in\mathbb{R}^{n\times n}$, distribution of random variable $G\cdot\omega$ is also the standard Gaussian. Or in other words,
\[G_{\#}N(\mathbf{0},I_{n})=N(\mathbf{0},I_{n})\]
Here, $\#$ is the pushforward measure.
\end{Lemma}
\begin{Lemma}
\label{elementary-1}
For orthonormal vectors $\alpha_{1},\alpha_{2},\cdots,\alpha_{k}\in\mathbb{R}^{n}~~(k\leqslant n)$, there exists an orthogonal matrix $G$, such that:
\[\forall i\in[k], ~G\alpha_{i}=e_{i}\]
Here, $e_{i}\in\mathbb{R}^{n}$ is a unit vector with the $i$-th component 1 and the others 0.
\end{Lemma}
\begin{Proof}
Let $\{\beta_{1},\beta_{2},\cdots,\beta_{n-k}\}$ be a set of orthonormal basis of linear space $[span(\alpha_{1},\alpha_{2},\cdots,\alpha_{k})]^{\perp}$. Then matrix $G = (\alpha_{1},\cdots,\alpha_{k},\beta_{1},\cdots,
\beta_{n-k})^{T}$ is orthogonal and it's easy to verify that: $G\alpha_{i}=e_{i}$ holds for all $i\in[k]$, which comes to the conclusion.
\end{Proof}
\begin{Lemma}
\label{elementary-2}
Vectors $\alpha_{i},\beta_{i}\in R^{m}~(i=1,2,\cdots,n)~(n>m)$ satisfy the following conditions:\\
(1) Both $A = (\alpha_{1},\alpha_{2},\cdots,\alpha_{n})$ and $B = (\beta_{1},\beta_{2},\cdots,\beta_{n})$ are full rank matrices.\\
(2) $\forall i,~ \|\alpha_{i}\|_{2}=\|\beta_{i}\|_{2}$\\
(3) $\forall i\neq j,~ \alpha_{i}\cdot\alpha_{j} = \beta_{i}\cdot\beta_{j}$\\
Then, there exists an orthonormal matrix $G\in\mathbb{R}^{m\times m}$, such that: $\forall i\in[n], \alpha_{i}=G\beta_{i}$.
\end{Lemma}
\begin{Proof}
Without loss of generality, assume that $\alpha_{1},\alpha_{2},\cdots,\alpha_{m}$ are linearly independent. Denote $A=(X,A_{1})$ and $B=(Y,B_{1})$. Here, $X,Y\in\mathbb{R}^{m\times m}, A_{1},B_{1}\in\mathbb{R}^{m\times(n-m)}$. According to condition (2) and (3), we have: $A^{T}A=B^{T}B$ and we need to prove that there exists an orthogonal matrix such that: $A=GB$. 
\begin{align}
A^{T}A=B^{T}B&\Rightarrow \left(\begin{matrix}X^{T}\\A_{1}^{T}\end{matrix}\right)(X,A_{1})=\left(\begin{matrix}Y^{T}\\B_{1}^{T}\end{matrix}\right)(Y,B_{1})\notag\\
&\Rightarrow X^{T}X=Y^{T}Y, X^{T}A_{1}=Y^{T}B_{1}
\end{align}
According to the polar decomposition of matrices, we know that there exists orthogonal matrices $G_{1}$ and $G_{2}$, such that:
\[X=G_{1}\sqrt{X^{T}X}, Y=G_{2}\sqrt{Y^{T}Y}=G_{2}\sqrt{X^{T}X}\]
which means for orthogonal matrix $G=G_{1}G_{2}^{-1}$, $X=GY$. Since we have assumed that $X$ is invertible, $Y$ is also invertible and moreover:
\[X^{T}A_{1}=Y^{T}B_{1}\Rightarrow Y^{T}G^{T}A_{1}=Y^{T}B_{1}\Rightarrow A_{1}=GB_{1}\Rightarrow A=GB\]
which comes to our conclusion.
\end{Proof}

\begin{Lemma}\label{inner_product}
Given $k$ unit vectors $x_{1},x_{2},\cdots,x_{k}\in\mathbb{R}^{d}$, then there exists another unit vector $y\in\mathbb{R}^{d}$, such that for $\forall i\in[k]$,
\[|\langle x_{i}, y\rangle|\geqslant \frac{2}{\pi dk}\]
\end{Lemma}
\begin{Proof}
Denote $t=\frac{2}{\pi dk}$ and $K_{i}=\{y\in\mathbb{R}^{r}:\|y\|_{2}=1, |\langle x_{i}, y\rangle|\geqslant t\}$. Also, we denote $\e_{i}$ be the $r$-length vector with its $i$-th element 1 and others 0. Next we will calculate the surface area of $K_{i}$. 
\begin{align}
A(K_{i})&=\int_{\w\in S^{d-1},|\w\cdot x_{i}|\leqslant t}dS = \int_{\w\in S^{d-1},|\w\cdot \e_{d}|\leqslant t}dS\notag\\
&=\int_{\substack{x_{1}^{2}+\cdots+x_{d}^{2}=1,\\ |x_{d}|\leqslant t}}\frac{1}{|x_{d}|}dx_{1}dx_{2}\cdots dx_{d-1}\notag\\
&=\int_{1-t^{2}\leqslant x_{1}^{2}+\cdots+x_{d-1}^{2}\leqslant 1}\frac{1}{\sqrt{1-x_{1}^{2}-\cdots-x_{d-1}^{2}}}dx_{1}\cdots dx_{d-1}\notag\\
&=\int_{\sqrt{1-t^{2}}}^{1}d\rho\int_{0}^{\pi}d\phi_{1}\cdots\int_{0}^{2\pi}d\phi_{d-2}\frac{1}{\sqrt{1-\rho^{2}}}\rho^{d-2}\sin^{d-3}{\phi_{1}}\cdots\sin{\phi_{d-3}}\notag\\
&=\int_{\sqrt{1-t^{2}}}^{1}\frac{\rho^{d-2}}{\sqrt{1-\rho^{2}}}d\rho\int_{0}^{\pi}~\sin^{d-3}{\phi_{1}}d\phi_{1}~\cdots\int_{0}^{2\pi}d\phi_{d-2}\notag
\end{align}
Combining the surface area of $S^{d-1}$:
\begin{equation*}
\begin{aligned}
A(S^{d-1})&= \int_{0}^{1}\frac{\rho^{d-2}}{\sqrt{1-\rho^{2}}}d\rho\int_{0}^{\pi}~\sin^{d-3}{\phi_{1}}d\phi_{1}~\cdots\int_{0}^{2\pi}d\phi_{d-2}
\end{aligned}
\end{equation*}
We know that:
\begin{equation}
\begin{aligned}
\frac{A(K_{i})}{A(S_{d-1})}=\frac{\int_{\sqrt{1-t^{2}}}^{1}\frac{\rho^{d-2}}{\sqrt{1-\rho^{2}}}d\rho}{\int_{0}^{1}\frac{\rho^{d-2}}{\sqrt{1-\rho^{2}}}d\rho}\leqslant \frac{\int_{0}^{\arcsin t}(\cos\theta)^{d-2}d\theta}{\int_{0}^{1}\rho^{d-2}d\rho}\leqslant (d-1)\arcsin{t}<\frac{\pi}{2}td
\end{aligned}
\end{equation}
Therefore:
\[A(K_{1}\cup K_{2}\cup\cdots\cup K_{k})\leqslant \sum_{i=1}^{k}A(K_{i})<\frac{\pi}{2}tdk\cdot A(S_{d-1})=A(S_{d-1})\]
which means there exists a unit vector $y\in\mathbb{R}^{d}$, such that:
\[\forall i\in[k], ~~|\langle x_{i},y\rangle|\geqslant t=\frac{2}{\pi dk}\]
which comes to our conclusion.
\end{Proof}

\subsection{Concentration Inequality for Polynomials}
\input{concentration.tex}

\subsection{Properties and Conclusions about Tensor Decomposition}
\input{tensor_decomposition.tex}

\subsection{Properties and Conclusions about Matrix Perturbation Theory}
\input{matrix_perturbation.tex}

\subsection{Proof of Theorem 1}
\input{proof1_intro.tex}

\subsection{Proof of Theorem 3}
\input{proof3_intro.tex}

\subsection{Proof of Theorem 2}
\input{proof2_intro.tex}

\subsection{Proof of Theorem 4}
\input{proof4_intro.tex}

\subsection{Proof of Theorem 5}
\input{proof5_intro.tex}

\section{An Efficient Algorithm to Determine Generators by Moment Analysis}
\input{algorithm.tex}
\end{document}

%% file: abstract.tex
Generative Adversarial Networks (GANs) are widely used models to learn complex real-world distributions. In GANs, the training of the generator usually stops when the discriminator can no longer distinguish the generator's output from the set of training examples. A central question of GANs is that when the training stops, whether the generated distribution is actually close to the target distribution, and how the training process reaches to such configurations efficiently? In this paper, we established a theoretical results towards understanding this generator-discriminator training process. We empirically observe that during the earlier stage of the GANs training, the discriminator is trying to force the generator to match the low degree moments between the generator's output and the target distribution. Moreover, only by matching these empirical moments over polynomially many training examples, we prove that the generator can already learn notable class of distributions, including those that can be generated by two-layer neural networks.

%In GANs, the training of the generator usually stops when the discriminator can no longer distinguish the generator's output from the set of training examples. A central question of GANs is that when the training stops, whether the generated distribution is actually close to the target distribution.  Previously, it was found that such closeness can only be achieved when there is a strict capacity trade-off between the generator and discriminator: Neither of the two models can be too powerful than the other. In this paper, we established one of the first theoretical results in explaining this trade-off. We show that when the generator is a class of two-layer neural networks, then it is necessary and sufficient for the discriminator to be a one-layer network with ReLU-type activation functions. With this trade-off, using polynomially many training examples, when the training stops, the generator will indeed output a distribution that is inverse-polynomially close to the target. Our result also sheds light on how GANs training can find such a generator efficiently. 

%% file: intro.tex
Generative Adversarial Network (GAN)~\cite{goodfellow2014generative} is one of the most popular models for generating real-life data. Due to its great success, many variants have been proposed to improve the training and generalization of GANs~\cite{gulrajani2017improved,odena2017conditional,karras2017progressive,miyato2018spectral,zhang2017stackgan,radford2015unsupervised,zhang2018self}. The goal of the GANs is to train a generator (usually a deep neural network) $G$ whose input $x$ follows from the standard Gaussian distribution, and whose output $G(\omega)$ is (in distribution) as close to the target distribution $\mathcal{D}^\star$ as possible. To achieve the goal, a discriminator network $\mathcal{D}$ is simultaneously trained to distinguish  $G(\omega)$ from $\mathcal{D}^\star$. The training process stops when a sufficiently trained discriminator can not distinguish $G(\omega)$ from $\mathcal{D}^\star$ better than random guessing. In this case, we also call the generator $G(\omega)$ wins the game~\cite{arora2017generalization}.

In this paper we focus on a popular variant of GANs called the Wasserstein GAN~\cite{arjovsky2017wasserstein}. The Wasserstein distance is a measurement between two probability distributions $p, q$, which is defined as: $W_{\mathcal{F}}(p, q) = \sup_{f \in \mathcal{F}} |\mathbb{E}_{x \sim p} f(x) - \mathbb{E}_{x \sim q} f(x)|$
where $\mathcal{F}$ is some certain set of functions. In Wasserstein GAN, the set of functions $\mathcal{F}$ is usually taken to be a set of certain structured neural networks (the discriminators $ \mathcal{D}$). Therefore, the generator will win the game when: $$ \sup_{\mathcal{D} \in \mathcal{F}} |\mathbb{E}_{\omega \sim N(0, I)} \mathcal{D}(G(\omega)) - \mathbb{E}_{X \sim \mathcal{\mathcal{D}^\star}} \mathcal{D}(X)| \approx 0$$

Despite the great empirical success of Wasserstein GANs, the following fundamental question still does not have a satisfying theoretical answer: 
% \vspace{-2mm}
\begin{center}
\emph{What does the generator $G(\cdot)$ learn after it wins the game? How can it reach this winning configuration efficiently?}
\end{center}

Towards answering this question, many theoretical works have been proposed. Notably, the works~\cite{arora2018gans,arora2017generalization,bai2018approximability} studied the learnt distribution of the generator when the generator and discriminator reach the global optimal solution (or a global Nash Equilibrium). However, the on the training side, the theoretical works~\cite{nagarajan2017gradient,heusel2017gans,mescheder2017numerics,daskalakis2018last,daskalakis2018limit,gidel2018negative,liang2018interaction,mokhtari2019unified,lin2019gradient} have only established the local convergence property of the GANs' training process, or the convergence to such a global optimal solution when the objective is convex-concave, leaving a huge gap between the \textbf{local} convergence results of GANs and the learnt distribution of GANs after \textbf{global} convergence. Moreover, to the best of our knowledge, up to now, there remains no theoretical results supporting the quality of a local optimal solution of the GANs' training objective. Even worse, it is a well-known empirical fact that there can be an enormous amount of bad local optimal solutions for the GANs' training objective, and a careful balance between the generator and discriminator must be maintained during the training to arrive at a good final solution. On the other hand, unlike in supervised training where the training objective involving a (over-parameterized) deep neural network can be approximately convex~\cite{als18dnn,li2018learning,al19-rnngen,als18,du2018gradient,arora2019finegrained,arora2019exact,zou2018stochastic,du2018gradient2,dfs16,jacot2018neural,ghorbani2019linearized,li2019towards,hanin2019finite,yang2019scaling,cao2019generalization,zou2019improved,cao2019generalization}, to the best of our knowledge, there are also no theoretical results showing that the GANs' training objective can be convex-concave, when neural networks with non-linear activation functions are involved.

To gradually bridge this gap between the convergence analysis of GANs' training and how can GANs reach a good final solution, in this paper, instead of directly studying the equilibrium conditions at the late stage of the training, we propose to first \emph{study the early phase of the GANs' training} -- when the generator and discriminator are not sufficiently trained yet. In particular, we ask the following fundamental question:
\begin{center}
\emph{What are the initial signals pick up by the generator $G(\omega)$ at the beginning of the training? What are the initial functions learnt by the discriminator to distinguish $G(\omega)$ from the true distribution?}
\end{center}

``Well begun is half done.'' - Aristotle.

The recent advance in deep learning theory~\cite{all18,arora2019finegrained} indicates that during the earlier phase of supervised training, a (over-parameterized) neural network is capable of learning a specific class of functions: The class of \textbf{low degree polynomials}. Motivated by this line of research, we consider the case when the discriminator falls into this function class. When the discriminators are the class of low degree polynomials, the GANs' learning process reduces to one of the most famous distribution learning framework: \textbf{The Method of Moments Learning}. Specifically, the discriminator now seeks for a mismatch between the low degree moments of the generator's output and that of the target distribution. The important observation is that:
%\vspace{-8mm}
\begin{claim}[Method of Moment Learning]
For every integer $q \geq 0$, when $\mathcal{F}$ is the class of degree $q$ polynomials with bounded coefficients,
$\sup_{\mathcal{D} \in \mathcal{F}} |\mathbb{E}_{\omega \sim N(0, I)} \mathcal{D}(G(\omega)) - \mathbb{E}_{X \sim \mathcal{\mathcal{D}^\star}} \mathcal{D}(X)| = 0$
if and only if for all non-negative integers $c \leq q$:
$\mathbb{E}_{\omega \sim N(0, I)} [G(\omega)^{\otimes c}] = \mathbb{E}_{X \sim \mathcal{\mathcal{D}^\star}} [X^{\otimes c}]$.
\end{claim}
Following this framework, we ask the following fundamental question:
% \vspace{-2mm}
\begin{center}
\emph{What signals can the generator network pick by simply matching the moments of the true distribution?}
\end{center}
% \vspace{-2mm}
\begin{figure}[hbt!]
\centering
\includegraphics[width=0.45\linewidth]{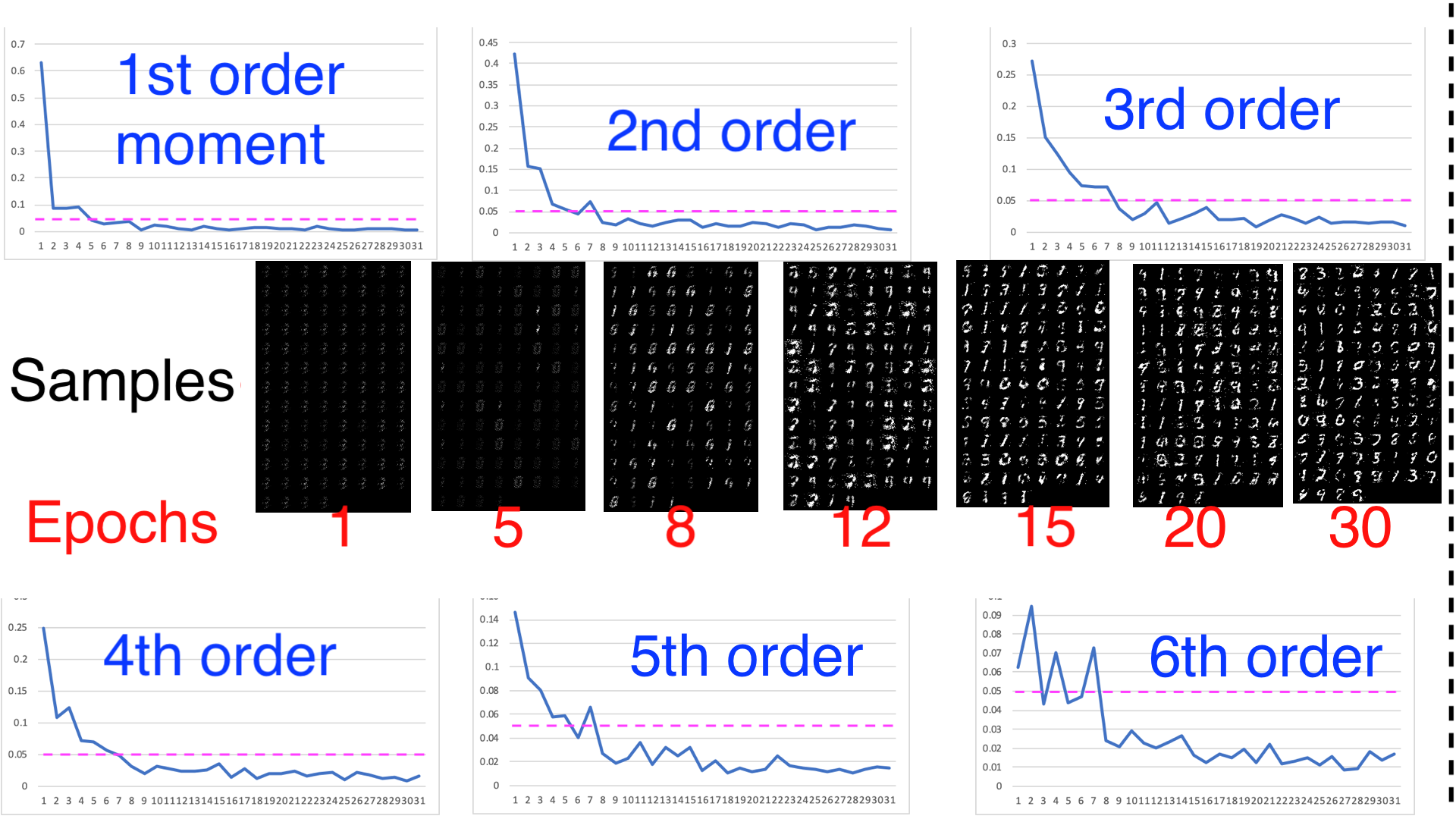}
\includegraphics[width=0.45\linewidth]{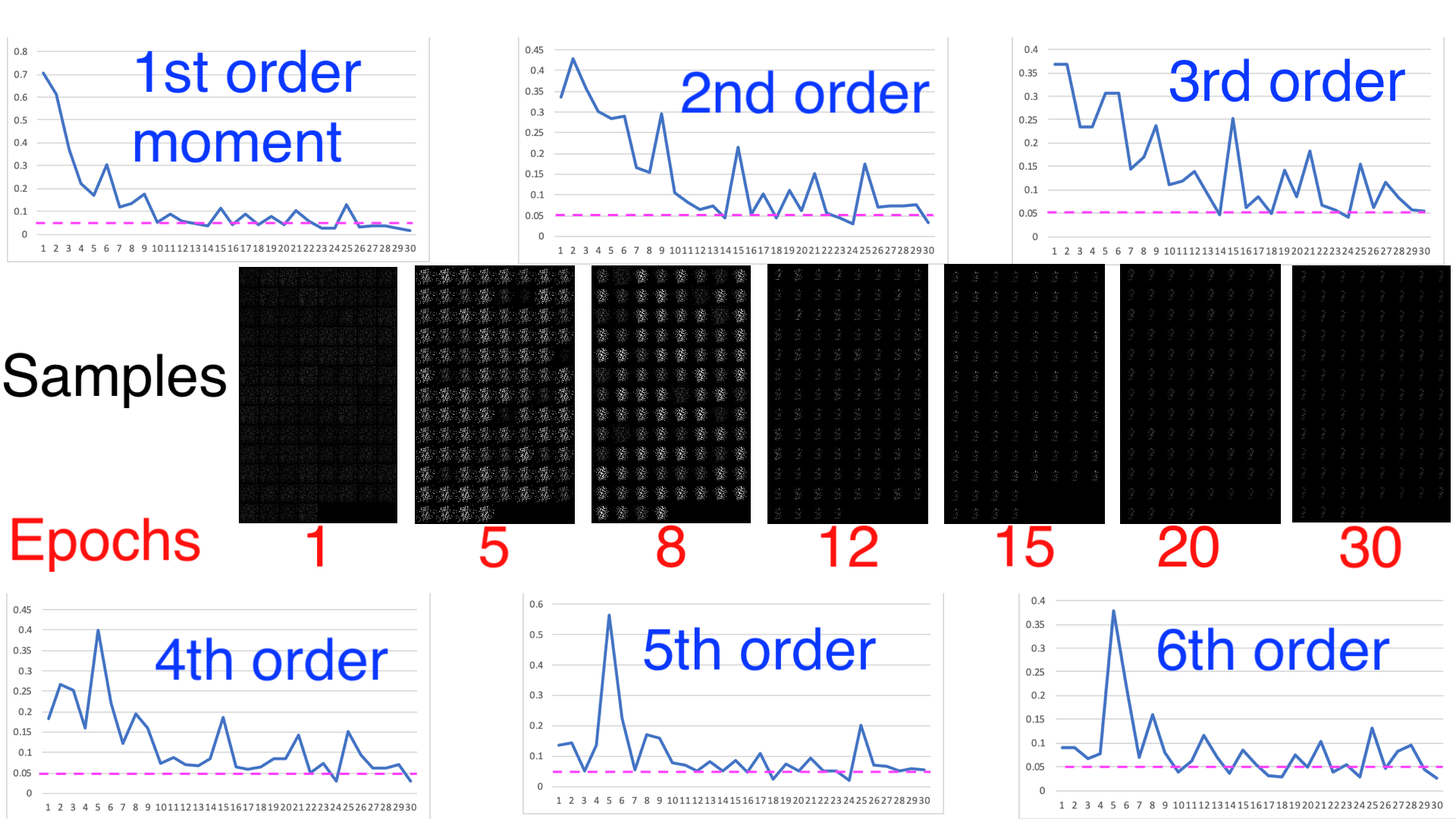}
% \vspace{-2mm}
\caption{Illustration of how the generator learn to match the moments between the output and the true distribution on MNIST data set. X-axis: The average mismatch ($\ell_2$ error) between the moments. Y-axis: The number of training epochs. Top: A successful training where the generator indeed learn to match the moments. Bottom: A unsuccessful training where the generator get stuck and fail to match the moments. The two cases are generated under the same training configuration and different random seeds. Purple line: the error reaches $0.05$.}
\label{fig:1}
% \vspace{-8mm}
\end{figure}

This question has been widely studied in the case when $c = 2$. In this case, the classical frameworks of PCA and CCA~\cite{allen2017doubly,allen2016lazysvd,allen2017first} indicate that the generator will actually learn to \emph{match the variance of its output comparing to the target distribution, in every single direction}. However, for many more complicated distributions, matching the variance would be insufficient to recover the  true target. Towards this end, the case when $c > 2$ has also received a notable amount of attention, and the tensor decomposition framework~\cite{anandkumar2014tensor} is proposed to show that matching moments of $c > 2$ can \emph{efficiently learn} certain classes of target distributions given by \emph{linear generative} models, such as sparse coding, non-negative matrix factorization etc. In this work, we extend the scope of method of moments to formally study the signals learnt by a generator as a \emph{non-linear neural network}. In particular, we show the following theory:
% \vspace{-2mm}
\begin{Theorem*}[Main, Method of Moments, Sketched]
Suppose the target distribution $\mathcal{D}^\star$ is generated by $\mathcal{D}^\star = G_{\#}^*\mathcal N(0,I_d)$ for an unknown (poly-weight) two-layer neural network $G^\star$ with polynomial activations, then for a learner two-layer neural network $G$ with the same or larger width, for every $\varepsilon > 0$, given $N = \poly(d/\varepsilon)$ training examples $X_1, \ldots,X_N \sim \mathcal{D}^\star $, by matching the first $O(1) $ order moments between the empirical observations and $G_{\#}(\mathcal{N}(0, I_d))$, it ensures that the learned network $G$ must satisfy:
\[W(G_{\#}\mathcal N(0,I_d), \mathcal{D}^\star) \leq \varepsilon.\]

Moreover, such a network $G$ can be found efficiently by a tensor decomposition algorithm.
\end{Theorem*}

% \vspace{-2mm}
In other words, the theorem says that the learner generator can \emph{sample efficiently} recover the true distribution generated by an unknown two-layer neural network $G^\star$, by matching the low-order moments between $G(\omega)$ and $\mathcal{D}^\star$. Moreover, motivated by the tensor decomposition algorithm, we also include an algorithm at the end to show how such generator $G$ can be found \emph{computationally efficiently} as well. Moreover, given the recent advance in training neural network using gradient descent~\cite{li2020learning,allen2020towards,allen2020feature} to perform tensor decomposition, such generator could potentially be found by simply doing gradient descent starting from random initialization. In figure~\ref{fig:2}, we also demonstrate empirically that one-hidden layer generator with polynomial activation functions can already learn to generate non-trivial images by matching the low degree moments.

\begin{figure}[hbt!]
\centering
\includegraphics[width=0.3\linewidth]{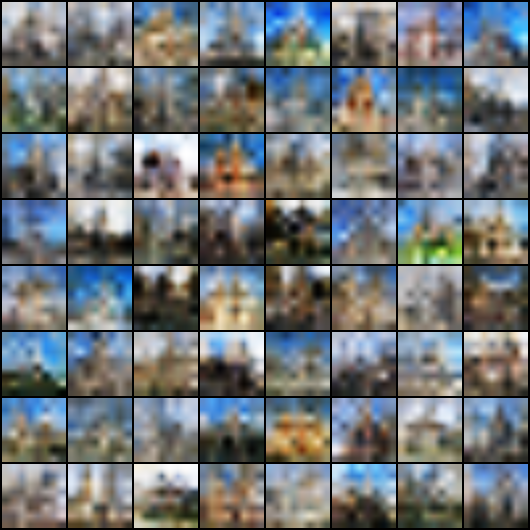} \quad
\includegraphics[width=0.3\linewidth]{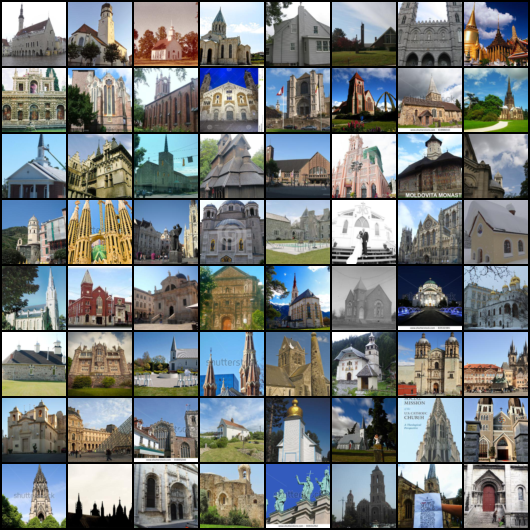}
\caption{Left: Samples generated by a two-layer generator neural network with \textbf{quadratic activation functions} trained on LSUN church data set (right). The discriminator is a three-hidden-layer network with quadratic activation functions. Thus, \textbf{the low-degree polynomial generator can already learn highly non-trivial images by merely matching the (local) moments up to order 8}.}
\label{fig:2}
% \vspace{-10mm}
\end{figure}

%% file: related.tex
% \vspace{-2mm}
\paragraph{Generative Adversarial Networks}

The works~\cite{nagarajan2017gradient,heusel2017gans,mescheder2017numerics,daskalakis2018last,daskalakis2018limit,gidel2018negative,liang2018interaction,mokhtari2019unified,lin2019gradient} considered either local convergence property of the GANs' training, or the global convergence property when the objective is convex-concave However, there is no clear evidence of how the local property can be transferred to learning the real distribution $\mathcal{D}^*$, and it is also not clear when the training objective of GAN is convex-concave. The works~\cite{feizi2017understanding,lei2019sgd} studied the global convergence of Wasserstein GAN when the generator network is one-layer neural network and the discriminator  is a (specially designed) quadratic/linear function. Since the discriminators are no more than quadratic, the training objective reduces to solving an Cholesky decomposition of a matrix $M$. In our case, to learn a good generator, a high degree discriminator is required, and the training objective involves tensor decomposition.  The work~\cite{bai2018approximability} focused on the situation when the generator is invertible (or injective). In this case, the image of both $G(\omega)$ and $G^*(\omega)$ each have density functions $p$ and $q$ that cover the entire space. Thus, one can use discriminator $\mathcal{D}$ to compute $\log p - \log q$ and minimize the KL-divergence between $p, q$. In this special case, it can also minimize the Wasserstein distance. Our work does not require $G, G^*$ to be injective, so there could be no density function for distributions $G_{\#}\mathcal N(0,I_d)$ and $G^*_{\#}\mathcal N(0,I_d)$. Our result relies on a different technique to match the moments between $G_{\#}\mathcal N(0,I_d)$ and $G^*_{\#}\mathcal N(0,I_d)$ instead of minimizing the KL-divergence. The work~\cite{zhang2017discrimination} showed that one-layer neural networks with ReLU activation are dense in the space of Lipschitz functions. However, it does not provide any efficient sample complexity guarantee.

% \vspace{-2mm}
\paragraph{Learning one hidden layer neural network with Gaussian inputs}
Learning neural networks with one hidden layer under standard Gaussian inputs is a popular research direction. However, most of the existing works focus on the setting of supervised learning. We refer to~\cite{li2017convergence,glm17} and the citations therein. In particular, the recent work~\cite{li2020learning} shows that two-layer ReLU neural networks can \emph{efficiently} learn a tensor-decomposition-like objective by simply doing gradient descent from random initialization.
% \vspace{-2mm}
\paragraph{Method of Moments}

Our work is also related to the Method of Moments, a well-known approach to learn an underlying distribution. Method of Moments has been used in many other machine learning problems such as mixture of Gaussian distribution~\cite{vempala2004spectral,moitra2010settling}, topic models~\cite{anandkumar2012spectral}, hidden Markov models~\cite{anandkumar2014tensor}, dictionary learning~\cite{arora2014new}, mixture of linear regression~\cite{li2018learning} and so on, and has been used to design practical GANs~\cite{li2017mmd} as well. To the best of our knowledge, the method of moments (beyond order two) has not been explicitly studied to understand the performance of Wasserstein GANs.

%% file: preliminary.tex
% \vspace{-2mm}
We begin by introducing some notations: For a function $f: \mathbb{R}^d \to \mathbb{R}$, we define the Lipschitz constant of the function $f$ as:
$\| f\|_{\mathrm{Lip}} = \sup_{x \not= y} \frac{|f(x) - f(y)|}{\| x - y\|_2}$.

For random variables $X, Y$ in $\mathbb{R}^d$, we define the Wasserstein-1 distance $W_1(X, Y)$ as:
$$W_1(X, Y) = \sup_{f: \mathbb{R}^d \to R,~ \| f\|_{\mathrm{Lip}} \leq 1} \mathbb{E}[f(X)] - \mathbb{E}[f(Y)]$$

We use $G: \mathbb{R}^d \to \mathbb{R}^D$ to denote the  generator network in the learner, and $\mathcal{D}: \mathbb{R}^D \to \mathbb{R}$ to denote the discriminator network in the learner. 
We use $\mathcal{D}^*$ to denote the true distribution where the training data $X$ are sampled from. We use $N$ to denote the number of training examples. For a (finite) set $\mathcal{Z}$ of training examples, we use $X \sim \mathcal{Z}$ to denote $X$ is a uniformly at random sample from the set $\mathcal{Z}$. We use $\mathcal{N}(\mu, \Sigma)$ to denote a normal distribution with mean $\mu$ and covariance $\Sigma$. We use $I_d$ (or sometimes $I$ for simplicity) to denote the identity matrix in dimension $d$. We use $G_{\#}\mathcal N(0,I_d)$ to denote the random variable $G(\omega)$ where $\omega \sim \mathcal{N}(0, I_{d}))$. We use $\lambda(X), \sigma(X)$ to denote the set of eigenvalues and the singular values of matrix $X$ respectively. Also:
$\lambda_{\min}(X):= \min_{\lambda\in\lambda(X)}|\lambda|,~~\sigma_{\min}(X):= \min_{\sigma\in\sigma(X)}|\sigma|.$
Last but not least, we list some notations on the tensor decomposition theory. For vector $\alpha\in\mathbb{R}^{d}$, we use $\alpha^{\otimes 3}$ to denote a $d\times d\times d$ tensor $\mathcal T$ where:
$\mathcal T_{ijk}=\alpha_{i}\alpha_{j}\alpha_{k}$. Also, for a tensor $\mathcal T\in\mathbb{R}^{d\times d\times d}$ and matrix $G\in\mathbb{R}^{d\times d}$, we denote $\mathcal T'=T[G,G,G]$ as a new tensor with the same shape as $\mathcal T$ where:
$\mathcal T'_{ijk}=\sum_{x,y,z=1}^{d}\mathcal T_{xyz}G_{ix}G_{jy}G_{kz}$.

\paragraph{Our Model}

Our work focuses on the Wasserstein GAN model. In this model, the set of all Lipschitz-1 functions is replaced by a set of (norm bounded) discriminator neural networks.  Given $N$ i.i.d. samples $\mathcal{Z} = \{X_i\}_{i = 1}^N$ from distribution $\mathcal{D}^*$ and $N$ i.i.d. samples $\mathcal{S} = \{ \omega_i \}_{i = 1}^N$ from $\mathcal{N}(0, I_d)$, the Wasserstein GAN's objective function is given as:
$L(G, D) = \mathbb{E}_{X \sim \mathcal{Z}}[D(X)] -\mathbb{E}_{\omega \sim \mathcal{S} }[D(G(\omega))]. $
In this paper, we focus on the realizable case where the distribution $\mathcal{D}^*$ is realizable by a two-layer neural network, given as:
$$G^{*}_{k}(\omega)=\sum_{i=1}^{r}\alpha_{k}^{*(i)} \act(\tv_{k}^{*(i)}\cdot \omega)~~(k=1,2,\cdots,D)$$

Where $\act$ is an activation function. We mostly focus on the case when $r$ is a constant, as a standard choice of the generator as a deconvolution neural network. We shall parameterize the learner network as 
$$G_{k}(\omega)=\sum_{i=1}^{r}\alpha_{k}^{(i)} \act(\tv_{k}^{(i)}\cdot \omega)~~(k=1,2,\cdots,D)$$

Where $\alpha_k^{(i)} \in \mathbb{R}$ and $\tv_k^{(i)} \in \mathbb{R}^d$ are trainable parameters. As mentioned in the intro, we consider discriminator networks consists of $O(1)$-degree polynomials of the input. In the end, we also extend the result to the discriminator network as $D(X) = \text{ReLU}(\langle \mathbf{g}, X \rangle + b)$, where $\mathbf{g} \in \mathbb{R}^d$ and $b \in \mathbb{R}$ are trainable parameters. We shall summarize our main results in the following section.

%% file: statement.tex
% \vspace{-2mm}
Our main result shows that when the weights in the target generator network $\alpha_{k}^{*(i)}$ and $\tv_{k}^{*(i)}$ is sampled from a smooth distribution, %then using the aforementioned generator and discriminator as the learner networks, the Wasserstein GAN training objective can 

Then by matching the empirical moment between the samples from $\mathcal{D}^*$ and the generator's output, it minimizes the Wasserstein distance between $G(\omega)$ and $\mathcal{D}^*$ as well. In our result, we consider the special case when there is a $\sigma > 0$ such that each $V_k = (\tv_{k}^{*(1)}, \tv_{k}^{*(2)}, \cdots,\tv_{k}^{*(r)})$ is an column orthogonal matrix with $\sigma_{\min} ((V_k, V_{k'})) \geq \sigma$ for every $k \not= k'$. To formally prove our main theory, we also consider a smooth distribution of the parameters of the ground-truth network, as $\alpha_{k}^{*(i)} \sim \mathcal{N}([\alpha_{k}^{*(i)}]_0, \sigma^2)$ for $|[\alpha_{k}^{*(i)}]_0 |\in [1, \poly(d)]$ . Our analysis belongs to the seminal framework of \emph{smoothed analysis} widely used in Theorem 5.1 of \cite{SA}, Theorem 1.1 of \cite{KMEANS},  Theorem 1.4 of \cite{TENSOR},  Theorem 5 of \cite{TENSOR2}, Theorem 5 of \cite{LASSSS} and so on, where $\sigma > 0$ can be arbitrarily small. We consider this smoothing for theoretical reason to ensure the non-degenerate condition of certain matrices, as we shall demonstrate in the proof. 

Given ground truth distribution $\mathcal{D}^*=G^*_{\#}\mathcal N(0,I_d)$, our main theory can be summarized as the follows:
\vspace{-2mm}
\begin{Theorem*}[Main; Method of Moment]

For every $\sigma, \varepsilon > 0$, for every constant $r$ satisfies the $p$-generic condition, for every $h(z)$ being a degree-$p$ polynomial of $z$, for $N_0 = poly\left(\frac{1}{\delta},\frac{1}{\sigma}\right)\cdot poly(D,d,\frac{1}{\epsilon})$, given $N = \Omega(poly(N_0))$
training examples $\{X_i \}_{i = 1}^N$ sampled from $\mathcal{D}^*$, with probability at least $1-D^{2}\delta$, for every generator network $G$ with $|\alpha_k^{(i)}|, \| \tv_k^{(i)} \|_2 \leq {poly}(d)$ such that 
%$$\max_{\|\mathbf{g}\|_{1}+|b| \leqslant 1} L(G, \mathcal{D}) \leq \varepsilon$$
$$\left\|\frac{1}{N} \sum_{i \in [N]} G(\omega_i)^{\otimes r} - \frac{1}{N} \sum_{i \in [N]} X_i^{\otimes r} \right\|_F^2 \leq \frac{poly(\varepsilon)}{N_0}$$

We must have that 
$$W_1(G_{\#}\mathcal N(0,I), \mathcal{D}^*) \leq \varepsilon$$
where $G_{\#}\mathcal N(0,I)$ stands for the pushforward measure, which means the distribution sampled by $G(z),~z\sim\mathcal N(0,I)$.
\end{Theorem*}

We defer the definition of the generic condition to Equation (\ref{eqn:generic}) for the special case when $p = 3$. For the general version with other $p$, we defer to Section A of Appendix. The generic condition asserts invertibility for some fixed $\poly(r) \times \poly(r)$ size matrix (whose entries are absolute constants that do not depend on $G^\star$ or $G$). Moreover, we point out that we have verified numerically that the generic condition holds for every constant $p \leq 10$ and $r \leq 50$.

% \vspace{-2mm}
\paragraph{Our approach: Method of Moments and tensor decomposition}

Our approach is based on the recent technique of tensor decomposition~\cite{anandkumar2014tensor}. We first show that when the moment of one coordinate: $G_k(\omega)$ matches the moment of $G_k^\star(\omega)$, then for all $i$, the coefficient $\alpha_{k}^{(i)}$ must matches $\alpha_{k}^{*(i)}$. After that, we consider the moment between the coordinates of $G_k$ and $G_k^\star$. Using a careful reduction to tensor decomposition, we can show that when the joint moment of $G_k, G_{k'}$ matches $G_k^\star, G_{k'}^\star$, then the vectors $\tv_{k}^{(i)}$ and $\tv_{k}^{\star (i)}$ must be close as well, hence we conclude that the Wasserstein distance between $G$ and $G^\star$ is small.

% \vspace{-2mm}
\paragraph{Extensions to ReLU discriminators}

We also observe that a one-hidden-layer ReLU discriminator can be used to efficiently simulate any low degree polynomial, i.e. For any $p(z) = z^q$ and $\varepsilon > 0$, there must be weights $w_1, \cdots, w_{C}, b_1, \cdots, b_{C}$ for some value $C $ depends on $\log \frac{1}{\varepsilon}$ and $q$ such that $ \sum_i \text{ReLU}(w_i z + b_i) \in [p(z) - \varepsilon, p(z) + \varepsilon]$ for every $z$ with $|z| \leq 1$. Thus, if the generator wins against the ReLU discriminator, then the generator wins against any low degree polynomial discriminator. In this case, we then show that the generator must (approximately) match all the low degree moment between $G(\omega)$ and $G^\star(\omega)$ as well. 

% \vspace{-2mm}

\paragraph{How Wasserstein GAN minimizes Wasserstein distance}
Our method also sheds light on how Wasserstein GAN can learn the true distribution: Instead of directly learning the true distribution of $G^\star(\omega)$, the discriminator will simply try to find a mismatch of the lower order moment between $G(\omega)$ and the true distribution. Thus, the generator will simply update its weights to match all the lower order moment of $\mathcal{D}^*$. By doing so, we actually show that the generator is already learning the distribution: In fact, the generator needs to minimize a objective consists of certain tensor Frobenius norm difference between $a, \tv$ and $a^{\star}, \tv^{\star}$. Our main theorem shows that such an objective must imply that $\alpha, \tv$ is close to $\alpha^{\star}, \tv^{\star}$. Moreover,  Prior works have also shown how such a tensor Frobenius minimization problem can be solved efficiently using gradient descent in certain cases~\cite{ge2015escaping}.

%% file: sample_complexity.tex
% \vspace{-2mm}
After the identifiability we discuss above, there is another important property which is the sample complexity. Firstly, we state our conclusion on the sample complexity of intro-component moment analysis. With loss of generality, assume $\alpha_{k}^{(1)}<\alpha_{k}^{(2)}<\cdots<\alpha_{k}^{(r)}, ~\alpha_{k}^{*(1)}<\alpha_{k}^{*(2)}<\cdots<\alpha_{k}^{*(r)}$ holds for $\forall k\in[D]$.  
% \vspace{-2mm}
\begin{Theorem}
For each $k\in[D]$, we list $r$ equations:
\[M^{i}(G_k)=\frac{1}{N}\sum_{t=1}^{N}\left(G_{k}^*(\omega^{(t)})\right)^{i}=\overline{M^{i}}(G_{k}^*)~~~(i=2,4,\cdots,2r)\]
here the $\overline{M^{i}}(G_{k}^*)$ stands for the empirical mean since we can't obtain the knowledge about the moment of the target generator in advance. After solving these equations, we can get a unique learner generator $G_{k}$ which satisfies:
\[\left(\sum_{i=1}^{r}|\alpha_{k}^{(i)}-\alpha_{k}^{*(i)}|^{2}\right)^{1/2}<\left(\frac{C_{1}(6r)^{3r}\cdot\mathcal A^{r+1}}{\tau}\right)^{3r}\cdot \sqrt{\frac{\log(Dre^{2}/\delta)}{N}}\]
holds for $\forall k\in [D]$ with probability at least $1-\delta$ over the choice of samples $\omega^{(t)}$. Here $C_{1}$ is an absolute constant.
\end{Theorem}
% \vspace{-2mm}
Next, we state the sample complexity of inter-component moment analysis. 
\begin{Theorem}
For each $i\neq j\in[D]$, we list $K_{r}$ equations:
\[M^{k,1}(G_{i},G_{j})=\frac{1}{N}\sum_{t=1}^{N}(G_{i}^*(\omega^{(t)}))^{k}\cdot(G_{j}^*(\omega^{(t)}))=\overline{M^{k,1}}(G_{i}^*,G_{j}^*)~~(k=1,3,\cdots,2K_{r}-1)\]
After solving these equations, we can get a learner generator $G_{k}$, such that: with probability larger than $1-5\delta$ over the choice of $\alpha_{k}^{*(i)}, \tv_{k}^{*(i)}~k\in[D],i\in[r]$, it holds that:
\[|\tv_{i}^{(a)}\cdot\tv_{j}^{(b)}-\tv_{i}^{*(a)}\cdot\tv_{j}^{*(b)}|<\left(\frac{C'r^{23}\mathcal A^{6}}{\tau\delta\sqrt{\sigma}}\right)^{9r^6}\cdot d^{8}\sqrt{\frac{\log(e^{2}(Dr+2r^{3})/\delta)}{N}}+O\left(\frac{1}{N}\right)\triangleq \mu_{3}\]
where $C'$ is an absolute constant.
\end{Theorem}
After we analyze the intro-component and inter-component moments between learner generator and target generator, we are finally able to estimate their Wasserstein distance.
\[W_{1}(G_{\#}\mathcal N(0,I), G^*_{\#}\mathcal N(0,I))=\sup_{Lip(D)\leqslant 1}\left[\mathop\mathbb{E}_{\omega\sim\mathcal N(0,I)}D(G(\omega))-\mathop\mathbb{E}_{\omega\sim\mathcal N(0,I)}D(G^*(\omega))\right]\]
\begin{Theorem}
[Main theorem] With probability at least $1-\frac{5D^{2}}{2}\delta$ over the choice of the target generator $G^*$, we can efficiently obtain a learner generator $G$ only by polynomial discriminators, such that:
\[W_{1}(G_{\#}\mathcal N(0,I), G^*_{\#}\mathcal N(0,I))
< \sqrt{D}\cdot d^{4}\left(\frac{Cr^{23}\mathcal A^{6}}{\tau\delta\sqrt{\sigma}}\right)^{5r^6}\sqrt[4]{\frac{\log(e^{2}(Dr+2r^{3})/\delta)}{N}}+o\left(\frac{1}{N^{1/4}}\right)\]
Here, each $\alpha_{k}^{*(i)}$ is sampled independently by distribution $\mathcal N(M,\sigma)$. Each $\tv_{k}=(\tv_{k}^{*(1)}, \tv_{k}^{*(2)},\cdots,\tv_{k}^{*(r)})\in\mathbb{R}^{d\times r}$ is the first $r$ column of an arbitrary orthogonal matrix. $M$ is a constant positive integer. We also assume that the target generator satisfies the $(\tau,\mathcal A)$-robustness. This theorem also tells us that in order to make the Wasserstein distance smaller than $\epsilon$, the number of samples we need is:
\[N=\tilde{\Omega}\left(\poly\left(r,\frac{\mathcal A}{\tau},\frac{1}{\delta},\frac{1}{\sigma}\right)^{\poly(r)}\cdot d^{16}D^{2}\left(\frac{1}{\epsilon}\right)^{4}\right)\]
\end{Theorem}

%% file: experiments.tex
In this section, we conduct some simple experiments to test and verify our conclusions. In our experiment, we use quadratic functions, 4-degree functions or 6-degree functions with rank 2 as both learner and target generator $G,G^*:\mathbb{R}^{4}\rightarrow\mathbb{R}^{8}$. 
\[G_{k}(\omega)=\sum_{i=1}^{r}(\alpha_{k}^{(i)}\cdot\omega)^{p},~G_{k}^*(\omega)=\sum_{i=1}^{r}(\alpha_{k}^{*(i)}\cdot\omega)^{p}~~~k\in[D]\]
In the following figures, we use linear functions, quadratic functions, 1-layer ReLU networks, 2-layer ReLU networks as discriminators respectively. The left one shows the change of Wasserstein loss $L(G,D)$ by iterations while the right one shows the change of the parameter distance between the learner generator and target generator. We calculate the actual parameter distance by directly comparing their parameters:
$d(G,G^*)=\|K^{T}K-K^{*T}K^*\|_{F}$.

Here, $K=(\alpha_{1}^{(1)},\alpha_{1}^{(2)},\cdots,\alpha_{D}^{(1)},\alpha_{D}^{(2)}),~ K^*=(\alpha_{1}^{*(1)},\alpha_{1}^{*(2)},\cdots,\alpha_{D}^{*(1)},\alpha_{D}^{*(2)})$. As our results show, the parameter distance can well approximate the actual Wasserstein distance between $G_{\#}\mathcal N(0,I)$ and $G^*_{\#}\mathcal N(0,I)$.
While training, we use the Wasserstein loss with gradient penalty as loss function. We use 1e-3 and 1e-4 as the learning rate of the generator optimization step and discriminator optimization step. From the results, we can see that, when using quadratic and cubic discriminators, the parameter distance does not converge to 0. That's because quadratic and cubic discriminators are too weak to distinguish different generators, or in other words, using moments $\leq 3$ are not enough. When using quartic functions or higher complexity function (like 1-layer and 2-layer ReLU networks) as discriminators to match higher moments, the Wasserstein loss converges and the parameter distance also converges to 0 very quickly. 

% \begin{figure}[!ht]
% \centering
% \subfigure[Wasserstein Loss]{
% \label{fig:loss_linear}
% \begin{minipage}[t]{0.27\linewidth}
% \centering
% \includegraphics[width=\linewidth]{loss_linear.png}
% \end{minipage}%
% }%
% \subfigure[Parameter Distance]{
% \label{fig:distance_linear}
% \begin{minipage}[t]{0.27\linewidth}
% \centering
% \includegraphics[width=\linewidth]{distance_linear.png}
% \end{minipage}%
% }%
% \caption{Discriminators as Linear Functions}
% \end{figure}

\begin{figure}[!ht]
\centering
\subfigure[quadratic]{
\begin{minipage}[t]{0.20\linewidth}
\centering
\includegraphics[width=\linewidth]{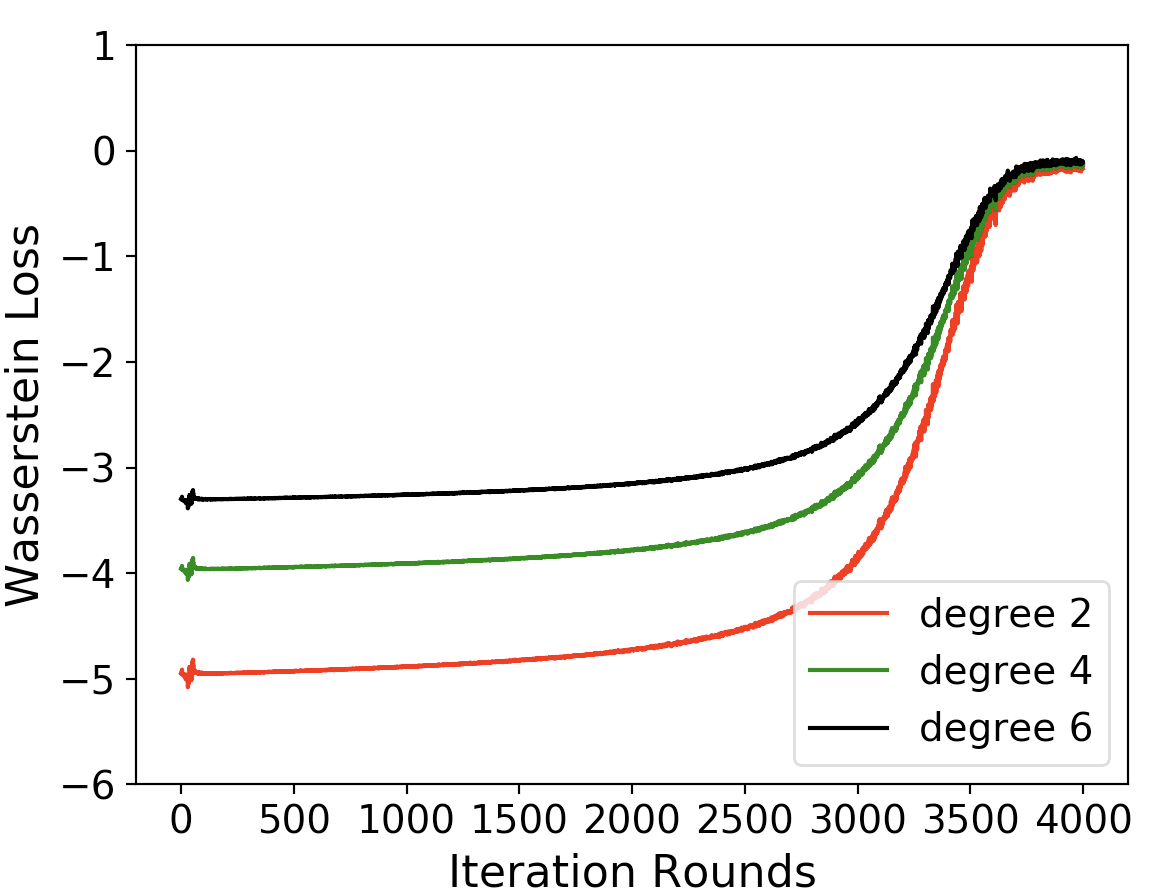}
\includegraphics[width=\linewidth]{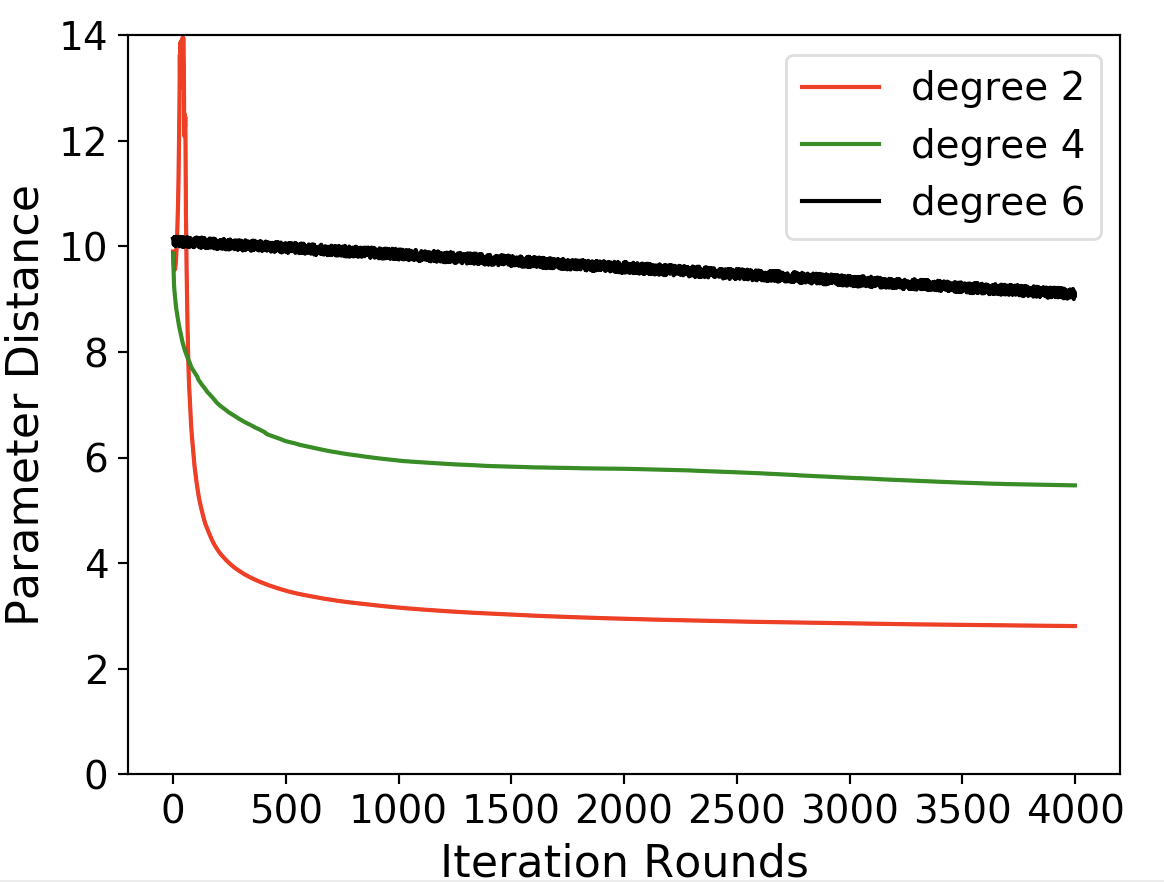}
\end{minipage}%
}%
\subfigure[cubic]{
\begin{minipage}[t]{0.20\linewidth}
\centering
\includegraphics[width=\linewidth]{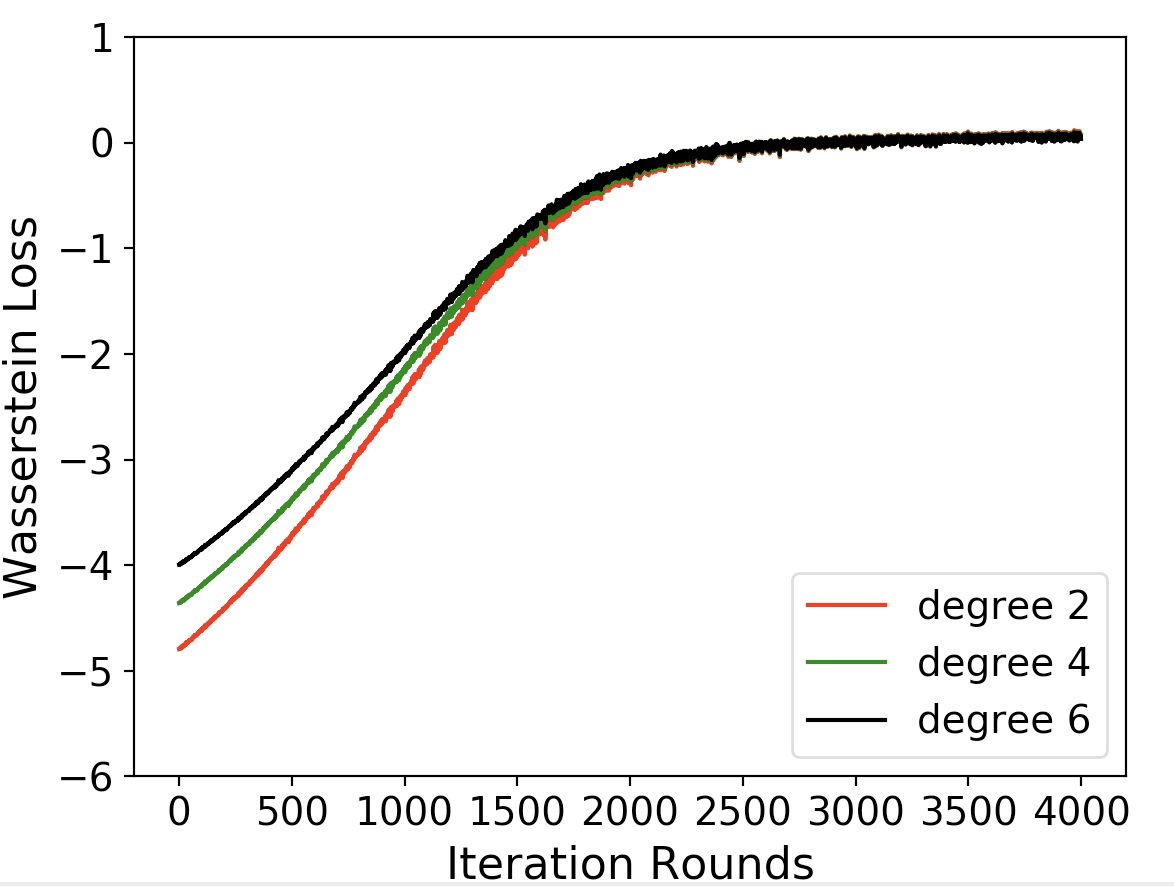}
\includegraphics[width=\linewidth]{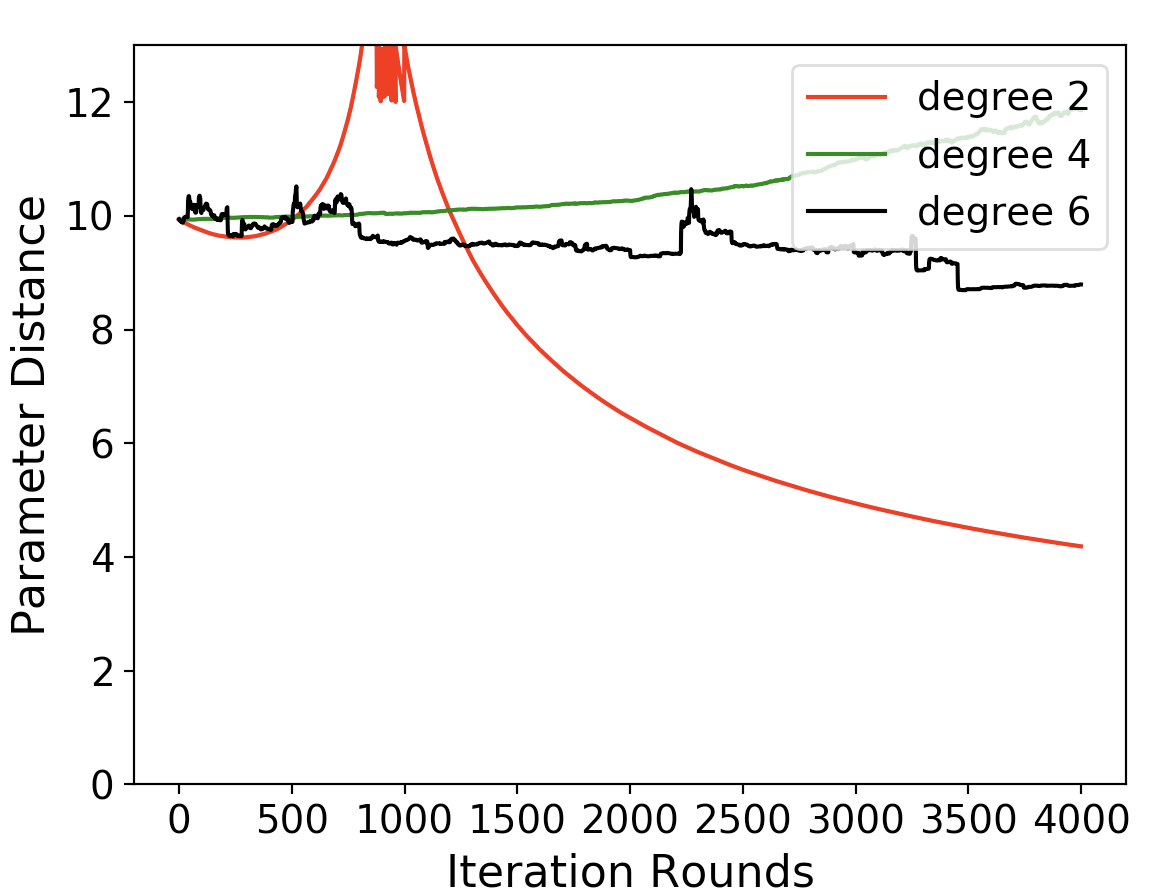}
\end{minipage}%
}%
\subfigure[quartic]{
\begin{minipage}[t]{0.20\linewidth}
\centering
\includegraphics[width=\linewidth]{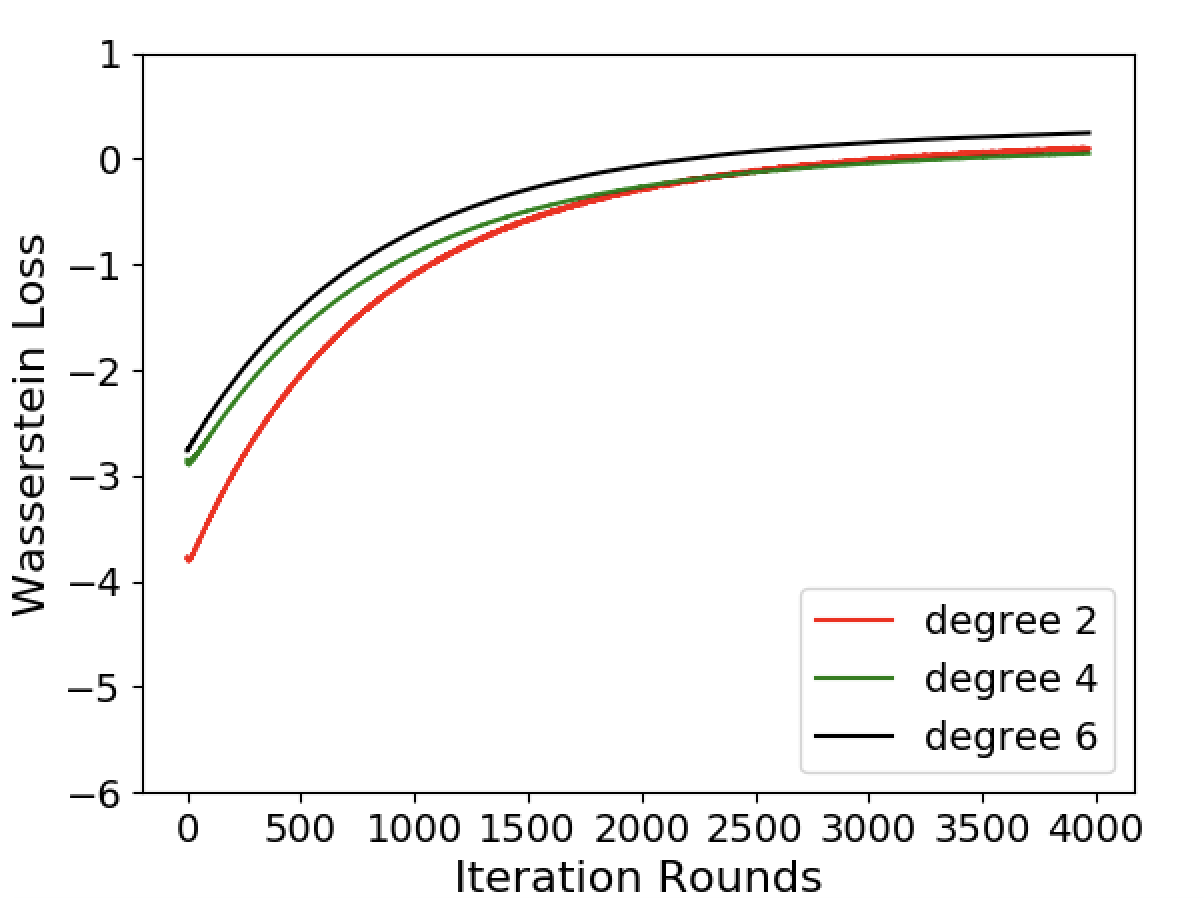}
\includegraphics[width=\linewidth]{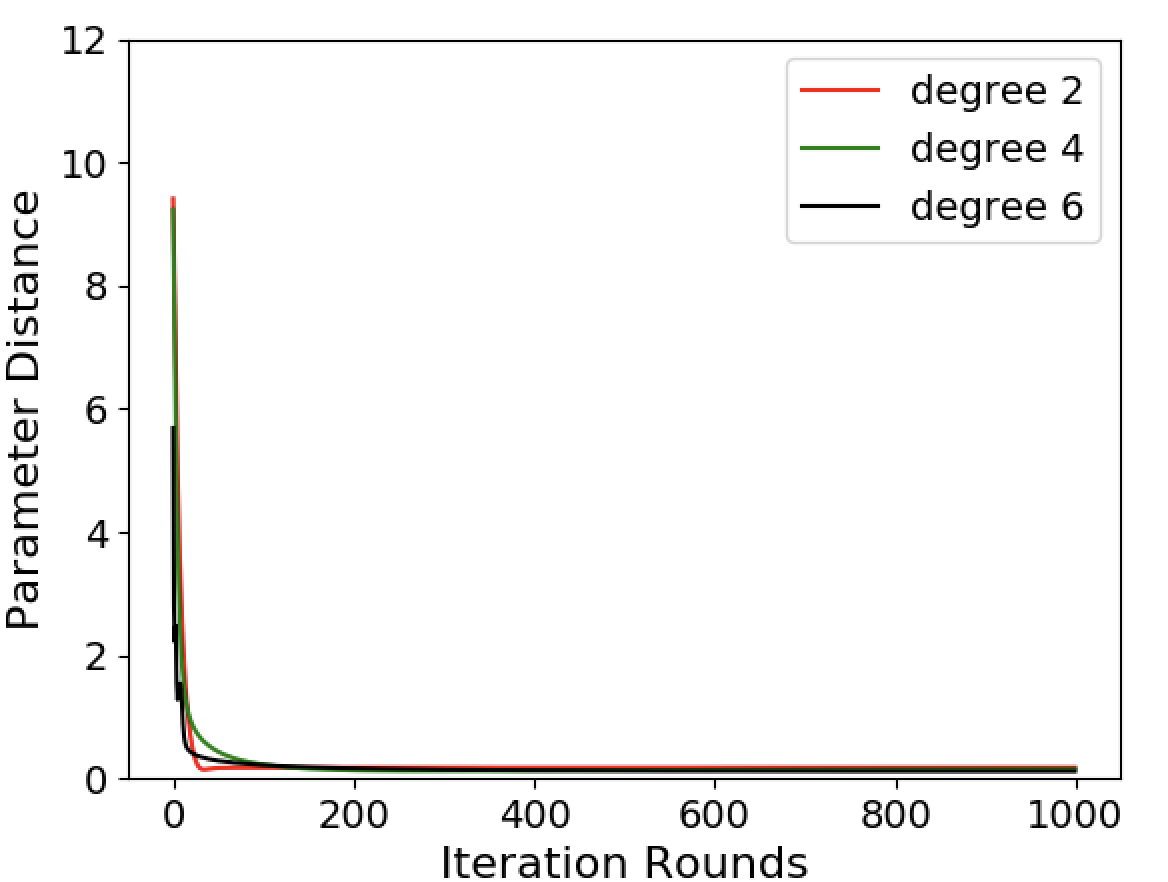}
\end{minipage}%
}%
\subfigure[1-layer ReLU]{
\begin{minipage}[t]{0.20\linewidth}
\centering
\includegraphics[width=\linewidth]{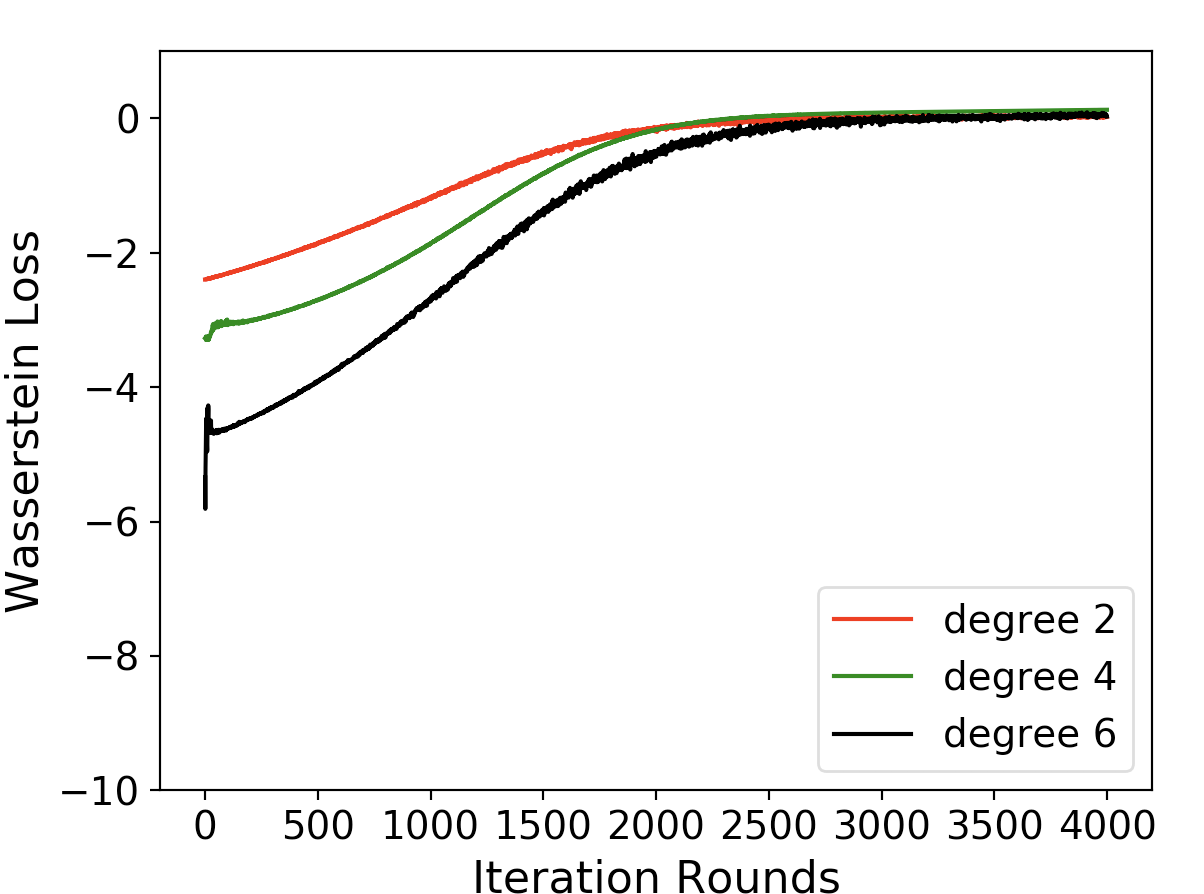}
\includegraphics[width=\linewidth]{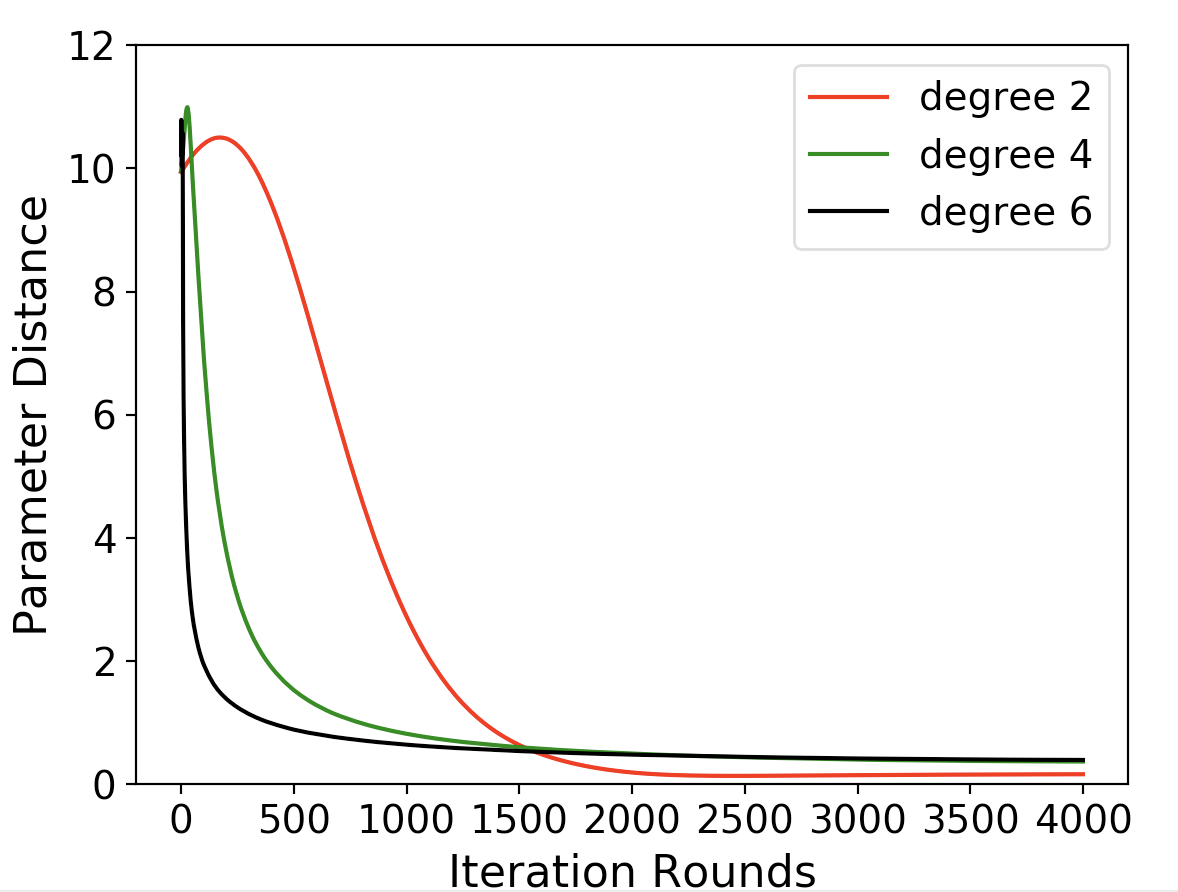}
\end{minipage}%
}%
\subfigure[2-layer ReLU]{
\begin{minipage}[t]{0.20\linewidth}
\centering
\includegraphics[width=\linewidth]{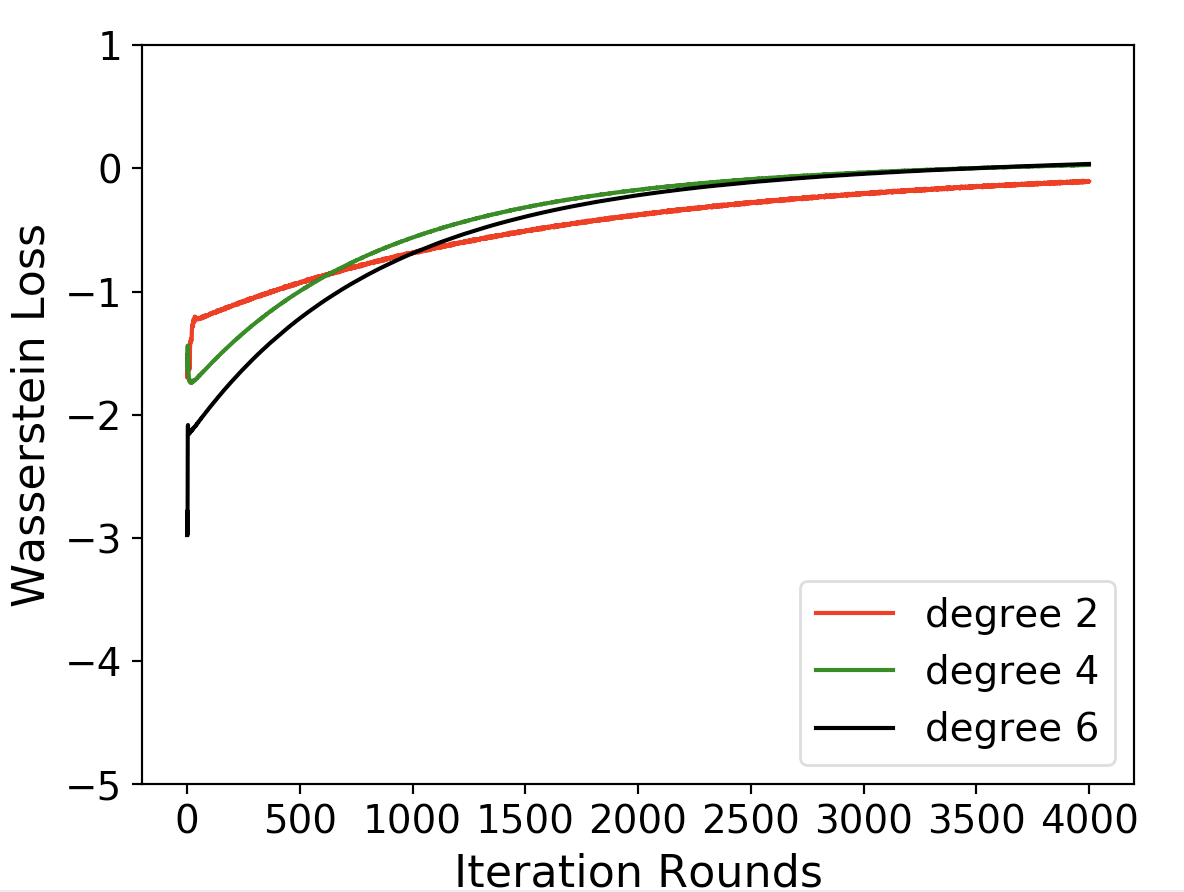}
\includegraphics[width=\linewidth]{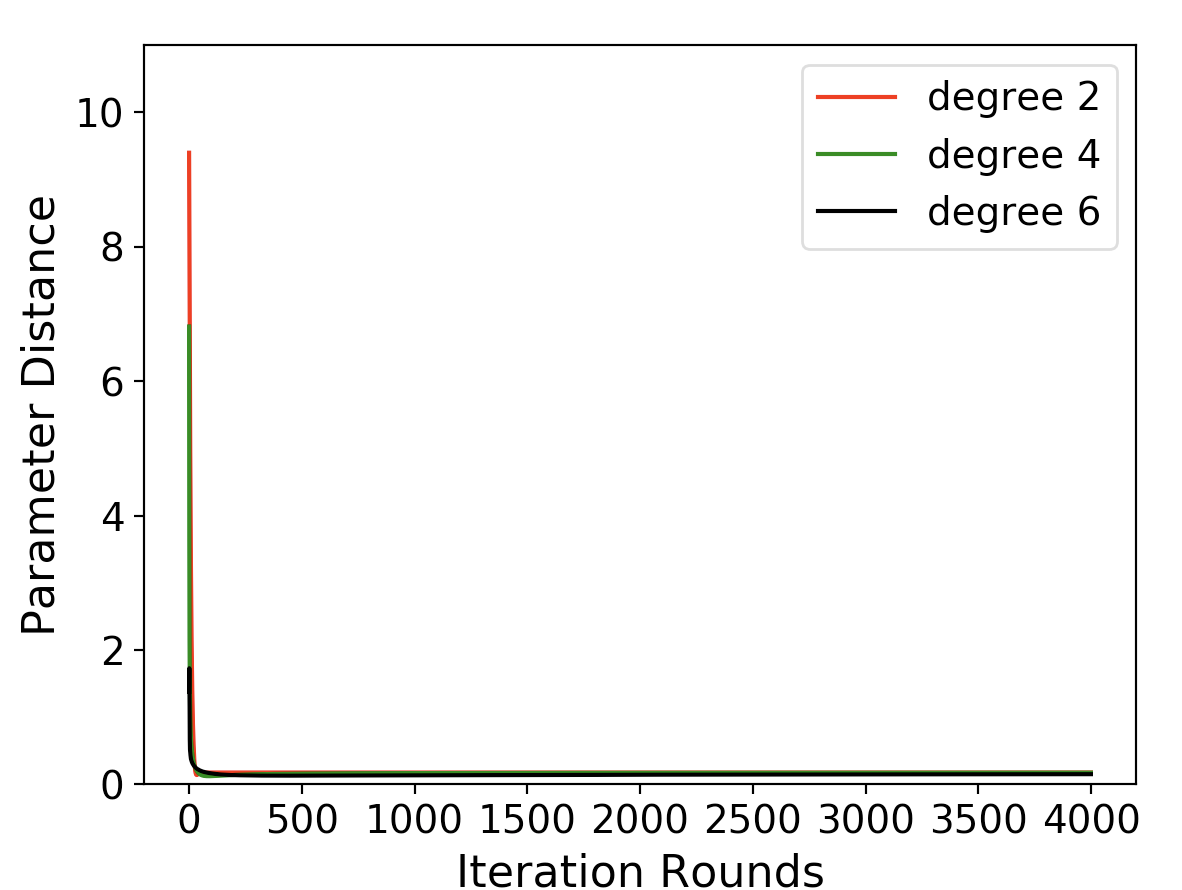}
\end{minipage}%
}%
\caption{These five figures show us the Wasserstein loss (top) and the parameter distance (bottom) curvature when discriminators are set as quadratic function, cubic function, quartic function, 1-layer ReLU network and 2-layer ReLU network respectively.}
% \vspace{-8mm}
\end{figure}

%% file: discussion.tex
In this paper we made a preliminary step understanding how the generator-discriminator training framework in GANs can learn the true distribution. We immediately see a sequence of open directions following our work: First, one can study the later phase of the training, when the discriminator starts to learn more complicated functions than low degree polynomials. Second, one can hope to extend our framework to generator neural networks with more layers. The recent advance in deep learning~\cite{AL2019-resnet,allen2020backward} that extends the study of two-layer neural networks~\cite{kawaguchi2016deep,soudry2016no,xie2016diversity,ge2017learning,soltanolkotabi2017theoretical,tian2017analytical,brutzkus2017globally,zhong2017recovery,li2017convergence,boob2017theoretical,li2017algorithmic,vempala2018polynomial,bakshi2018learning,zhang2018learning,li2017provable,li2016recovery} might make this direction a not-so-remote goal.  Another possible direction is to extend our study to other popular target distributions such as sparse coding/non-negative matrix factorization~\cite{arora2016latent,arora2018linear,li2016recovery,li2017provable,allen2020feature}.

%% file: generator.tex
In this section, we will show that cubic generator can be extended to higher degree generator, and the polynomial-type discriminator can be extended to a more general two-layer ReLU network. 

\subsection{Extension of Discriminators to 2-layer Networks}
On the other hand, we will discuss the extension from polynomial type discriminator to two-layer ReLU discriminator. According to the existing results \cite{ma2019priori, barron1993universal}, we know that any Lipschitz function can be approximated by two-layer ReLU networks under $L_{\infty}$, which obviously include the polynomial functions with finite degree. Therefore, a more general two-layer ReLU networks can overlap the polynomial functions, which makes our conclusion verified.

\subsection{Extension of Cubic Generators to 2-layer Generators with Polynomial Activation}
In the cubic generators, we have the following form:
\[G_k^*(\omega)=\sum_{i=1}^{r}\alpha_k^{*(i)}h(\tv_k^{*(i)}\cdot \omega)~~(k=1,2,\cdots,D).\]
Here, $h(z)=z^3$ is the cubic function. In this section, we introduce a more general setting where the activation function $h$ can be any high-degree polynomial with generic condition. Under this new setting, we name our generators to be polynomial generators, which is much more general than cubic ones. For polynomial generators, we have similar results as that of cubic generators above in Section 5. 

For intro-component moment analysis, the only difference we have is the moments of random variables $Y=h(\omega)$ and $Z=\omega^3$. It will change some coefficients in the proofs of Theorem 1 and Theorem 3, but won't change the proof lines. Similarly, we have:
\begin{Theorem}
[Intro-component Moment Analysis for Polynomial Generators] For each $k\in[D]$, we list $r$ equations:
\[M^{i}(G_k)=\frac{1}{N}\sum_{t=1}^{N}\left(G_{k}^*(\omega^{(t)})\right)^{i}=\overline{M^{i}}(G_{k}^*)~~~(i=2,4,\cdots,2r)\]
where, $G_k(\omega)=\sum_{i=1}^{r}\alpha_k^{(i)}h(\tv_k^{(i)}\cdot \omega)$ is the learner generator,  $G_k^*(\omega)=\sum_{i=1}^{r}\alpha_k^{*(i)}h(\tv_k^{*(i)}\cdot \omega)$ is the target generator, and $h$ is a degree $p$ polynomial activation function. After solving these equations, we can get a unique learner generator $G_{k}$ which satisfies:
\[\left(\sum_{i=1}^{r}|\alpha_{k}^{(i)}-\alpha_{k}^{*(i)}|^{2}\right)^{1/2}<\tilde{O}\left(\frac{1}{\sqrt{N}}\cdot poly\left(\kappa,p,r,\frac{\mathcal A}{\tau},\frac{1}{\delta}\right)^{poly(\kappa, p,r)}\right)\]
holds for $\forall k\in [D]$ with probability at least $1-\delta$ over the choice of samples $\omega^{(t)}$. Here, $\kappa$ is the degree of polynomial $h$.
\end{Theorem}

For inter-component analysis, the situation is also very similar. Before we begin to introduce the generalized moment analysis, we extend the CE-Matrix in Section 5 to a $p$-th order expectation matrix and set up a $p$-generic condition.
\begin{Definition}[$p$-th order Expectation Matrix]
Given a dimension $r$, $r$ coefficients $\lambda_{1},\lambda_{2},\ldots,\lambda_{r}$ and $r$ independent random variables $\omega_{1},\omega_{2},\ldots,\omega_{r}$ which follows the standard Gaussian distribution. Denote: \[Q=\lambda_{1}\omega_{1}^{p}+\lambda_{2}\omega_{2}^{p}+\cdots+\lambda_{r}\omega_{r}^{p}\]
Then we list all the odd-degree monomials with variables $\omega_{1},\cdots,\omega_{r}$ whose degrees are no larger than $p$. There are altogether $K_r^{(p)}=\sum_{1\leq i\leq (p+1)/2}\binom{r+2i-2}{2i-1}$ different monomials. We denote them as: $P_{1},P_{2},\cdots,P_{K_r^{(p)}}$. 
The $p$-th order Expectation Matrix is a $K_r^{(p)}\times K_r^{(p)}$ polynomial matrix $T$ in the following form:
\[T_{ij}=\mathop\mathbb{E}\limits_{\omega\sim \mathcal N(0,1)}\left[P_{i}\cdot Q^{2j-1}\right]~~i,j\in[K_r^{(p)}].\]
If the determinant of $T$ (which is a polynomial over $\lambda_{1},\lambda_{2},\ldots,\lambda_{r}$) is not the zero polynomial, we say that $r$ satisfies the $p$-generic condition.
\end{Definition}

\begin{Theorem}
[Inter-component Moment Analysis for Polynomial Generators] We assume that $r$ satisfies the $p$-generic condition. For each $i\neq j\in[D]$, we list $K$ equations:
\[M^{k,1}(G_{i},G_{j})=\frac{1}{N}\sum_{t=1}^{N}(G_{i}^*(\omega^{(t)}))^{k}\cdot(G_{j}^*(\omega^{(t)}))=\overline{M^{k,1}}(G_{i}^*,G_{j}^*)~~(k=1,2,\cdots,K),\]
where, $K=\sum_{1\leq i\leq (p+1)/2}\binom{r+2i-2}{2i-1}$. After solving these equations, we can get a learner generator $G_{k}$, such that: with probability larger than $1-\delta$ over the choice of $\alpha_{k}^{*(i)}, \tv_{k}^{*(i)}~k\in[D],i\in[r]$, it holds that:
\[|\tv_{i}^{(a)}\cdot\tv_{j}^{(b)}-\tv_{i}^{*(a)}\cdot\tv_{j}^{*(b)}|<\tilde{O}\left(\frac{poly(d)}{\sqrt{N}}\cdot poly\left(\kappa, p,r,\frac{\mathcal A}{\tau},\frac{1}{\delta},\frac{1}{\sigma}\right)^{poly(\kappa, p,r)}\right).\]
\end{Theorem}
Finally, we obtain the main theorem on the sample complexity upper bound.
\begin{Theorem}
[Main theorem for Polynomial Generators] With probability at least $1-D^{2}\delta$ over the choice of the target generator $G^*$, we can efficiently obtain a learner generator $G$ only by polynomial discriminators, such that:
\[W_{1}(G_{\#}\mathcal N(0,I), G^*_{\#}\mathcal N(0,I))
< \tilde{O}\left(\sqrt{D}\cdot \frac{poly(d)}{N^{1/4}}\cdot poly\left(\kappa, p,r,\frac{\mathcal A}{\tau},\frac{1}{\delta},\frac{1}{\sigma}\right)^{poly(\kappa, p,r)}\right)\]
Here, just like the cubic occasion, each $\alpha_{k}^{*(i)}$ is sampled independently by distribution $\mathcal N(M,\sigma)$. Each $\tv_{k}=(\tv_{k}^{*(1)}, \tv_{k}^{*(2)},\cdots,\tv_{k}^{*(r)})\in\mathbb{R}^{d\times r}$ is the first $r$ column of a random orthogonal matrix (under Haar Measure). $M$ is a constant positive integer. We also assume that the target generator satisfies the $(\tau,\mathcal A)$-robustness. This theorem leads to the sample complexity estimation we need:
\[N=\tilde{\Omega}\left(poly\left(\kappa,p,r,\frac{\mathcal A}{\tau},\frac{1}{\delta},\frac{1}{\sigma}\right)^{poly(\kappa,p,r)}\cdot poly\left(D,d,\frac{1}{\epsilon}\right)\right)\]
\end{Theorem}

%% file: concentration.tex
In this section, we introduce some definitions and properties in the field of multilinear polynomial and the concentration inequality for polynomial of independent random variables which are originally proposed in \cite{schudy2012concentration}. 
\begin{Definition}
A random variable $X$ is called moment bounded with parameter $L>0$, if for any integer $n$, it holds that:
\[\mathbb{E}(|X|^{n})\leqslant nL\cdot \mathbb{E}(|X|^{n-1})\]
\end{Definition}
Then, we show that normal Gaussian distribution is moment bounded with parameter $L=1$. 
\begin{Lemma}\label{gauss-L}
Normal Gaussian distribution $\mathcal N(0,1)$ is moment bounded with $L=1$.
\end{Lemma}
\begin{Proof}
Actually, we can calculate the moments:
\begin{equation*}
\mathop\mathbb{E}\limits_{X\sim\mathcal N(0,1)}|X|^{n}=
\left\{
\begin{aligned}
& (n-1)!!&~~~(n\text{~is even.})\\
& \sqrt{\frac{2}{\pi}}\cdot (n-1)!!&~~~(n\text{~is odd.})
\end{aligned}
\right.
\end{equation*}
Therefore, we have:
\begin{equation}
\begin{aligned}
&\frac{\mathbb{E}|X|^{2k}}{2k\cdot\mathbb{E}|X|^{2k-1}}= \sqrt{\frac{\pi}{2}}\cdot\frac{(2k-1)(2k-3)\cdots 1}{(2k)(2k-2)\cdots 2} \leqslant \sqrt{\frac{\pi}{2}}\cdot\frac{1}{2}<1\\
&\frac{\mathbb{E}|X|^{2k+1}}{(2k+1)\cdot\mathbb{E}|X|^{2k}}= \sqrt{\frac{2}{\pi}}\cdot\frac{(2k)(2k-2)\cdots 2}{(2k+1)(2k-1)\cdots 1}\leqslant \sqrt{\frac{2}{\pi}}<1
\end{aligned}
\end{equation}
which means that:
\[\forall n\in\mathbb{N},~~\mathbb{E}(|X|^{n})\leqslant n\cdot \mathbb{E}(|X|^{n-1})\]
and it comes to our conclusion.
\end{Proof}
Next, we further introduce some basic concepts about hypergraph and its relationship with polynomials.
\begin{Definition}~\\
~~(1) Concepts of hypergraphs:
A powered hypergraph $H$ consists of a vertex set $\mathcal V(H)$ and a powered hyperedge set $\mathcal H(H)$. A powered hyperedge $h$ consists of a set $\mathcal V(h)\subseteq \mathcal V(H)$ of $|\mathcal V(h)|=\eta(h)$ vertices and its corresponding power vector $\tau(h)$ with length $\eta(h)$. The power vector $\tau(h)$ has elements $\tau(h)_{v}=\tau_{hv}~\forall v\in\mathcal V(h)$. All these $\tau_{hv}$ are strictly positive integers. Besides, for each hyperedge $h$, define $q(h)=\sum_{v}\tau_{hv}$. For the whole hypergraph $H$, define $\Gamma = \max_{h\in
\mathcal H(H),v\in\mathcal V(H)}\tau_{hv}$ as the maximal power in $H$.\\
~~(2) Priority relationship between hypergraphs:
For powered hyperedges $h_{1},h_{2}\in \mathcal H(H)$, we write $h_{1}\succeq h_{2}$ if $\mathcal V(h_{1})\supseteq\mathcal V(h_{2})$ and $\tau_{h_{1}v}=\tau_{h_{2}v}$ holds for $\forall v\in\mathcal V(h_{2})$.\\
~~(3) Relationship with polynomials:
For powered hypergraph $H$, and real-valued weights $w_{h}$ for each hyperedge $h\in\mathcal H(H)$, we can define a polynomial $f:\mathbb{R}^{|\mathcal V(H)|}\rightarrow \mathbb{R}$.
\[f(x)=\sum_{h\in\mathcal H(H)}w_{h}\prod_{v\in\mathcal V(h)}x_{v}^{\tau_{hv}}\]
each hyperedge corresponds to a monomial with weight $w_{h}$.\\
~~(4) Order-$r$ Parameters: Assume $X_{1},\cdots,X_{|\mathcal V(H)|}$ be independent random variables. We define:
\[\mu_{r}(w,X)=\max_{h_{0}:\mathcal V(h_{0})\subseteq \mathcal V(H),q(h_{0})=r}\left(\sum_{h\in\mathcal H(H):h\succeq h_{0}}|w_{h}|\prod_{v\in \mathcal V(h)\setminus\mathcal V(h_0)}\mathbb{E}(|X_{v}|^{\tau_{hv}})\right)\]
\end{Definition}
\begin{Theorem}\label{thm-concent}
Given $n$ independent moment bounded random variables $X=(X_{1},X_{2},\cdots,X_{n})$ with the same parameter $L$. For a general polynomial $f(x)$ with total power $q$ and maximal variable power $\Gamma$, then:
\[Pr(|f(X)-\mathbb{E}f(X)|\geqslant t)\leqslant e^{2}\cdot\max\left(\max_{1\leqslant r\leqslant q}e^{-\frac{t^{2}}{\mu_{0}\mu_{r}\cdot L^{r}\Gamma^{r}\cdot R^{q}}} ,\max_{1\leqslant r\leqslant q}e^{-\left(\frac{t}{\mu_{r}L^{r}\Gamma^{r}\cdot R^{q}}\right)^{1/r}} \right)\]
where $R>1$ is an absolute constant. 
\end{Theorem}

%% file: tensor_decomposition.tex
In this section, we review the famous tensor decomposition algorithm, Jenrich's Algorithm, as well as the whitening and un-whitening process which are used in tensor decomposition. Assume tensor $\mathcal L\in\mathbb{R}^{r\times r\times r}$ has the following decomposition.
\[\mathcal L=\tu_{1}^{\otimes 3}+\tu_{2}^{\otimes 3}+\cdots+\tu_{r}^{\otimes 3}\]
where $\tu_{1},\tu_{2},\cdots,\tu_{r}$ are linearly independent. The following algorithm will output the $r$ vectors $\tu_{1},\cdots,\tu_{r}$ after we input the tensor $\mathcal L$.
\subsubsection{Whitening Process}
Firstly we properly sample a unit vector $x\in\mathbb{R}^{r}$, which satisfies: $x\cdot \tu_{i} > 0~~\forall i\in[r]$. Next we slice the tensor $\mathcal L$ into $r$ matrices and calculate the weighted sum of these matrices:
\[\mathcal L_{x}=x_{1}\mathcal L[1,:,:]+x_{2}\mathcal L[2,:,:]+\cdots+x_{r}\mathcal L[r,:,:]=\sum_{i=1}^{r}(\tu_{i}\cdot x)\cdot\tu_{i}^{\otimes 2}=UEU^{T}\]
Here, $U=(\tu_{1},\tu_{2},\cdots,\tu_{r})$ and $E=diag(\langle\tu_{i}, x\rangle: i\in[r])$. Since the matrix $\mathcal L_{x}$ is symmetrical, there exists an orthogonal matrix $Q$ and diagonal matrix $D$, such that:
\[\mathcal L_{x}=QDQ^{T}\]
Here, we notice that since $x\cdot \tu_{i}>0$, all of the diagonal elements of $E$ are positive. Therefore, matrix $\mathcal L_{x}$ is positive definite which means all of the diagonal elements of $D$ are positive. Finally, we calculate the matrix $G=QD^{-1/2}$. 
\begin{Lemma}
This matrix $G$ satisfies that: $G^{T}\tu_{1},G^{T}\tu_{2},\cdots,G^{T}\tu_{r}$ are orthogonal with each other.
\end{Lemma}
\begin{proof}
\begin{align}
&\mathcal L_{x}=UEU^{T}=QDQ^{T}\Rightarrow D=Q^{T}UEU^{T}Q \Rightarrow I=Q^{T}UEU^{T}QD^{-1}\notag\\
&~~\Rightarrow I=EU^{T}QD^{-1}Q^{T}U\Rightarrow U^{T}QD^{-1}Q^{T}U=E^{-1}\notag
\end{align}
which means: $E^{-1}=(G^{T}U)^{T}\cdot(G^{T}U)$ is also a diagonal matrix and that proves the conclusion.
\end{proof}
Then: 
\begin{equation}
\mathcal L[G,G,G]=\sum_{i=1}^{r}(G^{T}\tu_{i})^{\otimes 3}
=\sum_{i=1}^{r} \lambda_{i}\left(\frac{1}{\tilde{\lambda}_{i}}G^{T}\tu_{i}\right)^{\otimes 3}
\end{equation}
Here, $\tilde{\lambda}_{i}=\|G^{T}\tu_{i}\|_{2}, \lambda_{i}=\tilde{\lambda}_{i}^{3}>0$, and the equation above is an orthogonal tensor decomposition.
\subsubsection{Orthogonal Tensor Decomposition for Tensors}
We re-write the vector $\frac{1}{\tilde{\lambda}}G^{T}\tu_{i}\triangleq \tv_{i}$, and then $\tv_{1},\tv_{2},\cdots,\tv_{r}$ are orthonormal vectors. Then we will solve the orthogonal tensor decomposition problem:
\[\mathcal T\triangleq\mathcal L[G,G,G]=\sum_{i=1}^{r}\lambda_{i}\tv_{i}^{\otimes 3}\]
Similarly, we randomly sample a $r$-length vector $y$ and calculate the weighted sum:
\[\mathcal T_{y}=y_{1}\mathcal T[1,:,:]+y_{2}\mathcal T[2,:,:]+\cdots+y_{r}\mathcal T[r,:,:]=\sum_{i=1}^{r}\lambda_{i}(\tv_{i}\cdot y)\cdot \tv_{i}^{\otimes 2}=VD_{1}V^{T}\]
We know that with probability 1 over the choice of $y$, the diagonal elements of $D_{1}$ are distinct. Therefore, we can uniquely determine the eigenvalue matrix $D_{1}$ and unit eigenvectors regardless of their sign. Assume that the eigenvalues of $\mathcal T_{y}$ are $d_{1}>d_{2}>\cdots>d_{r}$ and their corresponding unit eigenvectors are $\tw_{1},\tw_{2},\cdots,\tw_{r}$. Then let $D_{1}=diag(d_{1},d_{2},\cdots,d_{r})$ and then $\tv_{i}=\pm\tw_{i}$. Since $d_{i}=\lambda_{i}(\tv_{i}\cdot y)$, we can determine the sign by the fact that $\lambda_{i}>0$. Therefore, we can determine the $\lambda_{i}$ and $\tv_{i}$ uniquely. 
\subsubsection{Un-whitening Process}
Finally, we need to use the results of the previous step to calculate $\tu_{1},\tu_{2},\cdots,\tu_{r}$. According to the equations above,
\begin{equation}
\begin{aligned}
G^{T}\tu_{i}=\sqrt[3]{\lambda_{i}}\tv_{i}\Rightarrow \tu_{i}=\sqrt[3]{\lambda_{i}}G^{-T}\tv_{i}=\sqrt[3]{\lambda_{i}}QD^{1/2}\tv_{i}
\end{aligned}
\end{equation}
Till now, we can imply that tensor decomposition problem has a guaranteed unique solution.

%% file: matrix_perturbation.tex
\begin{Lemma}\label{pertub_inverse}
(Perturbation of Inverses) Given matrices $A,E\in\mathbb{R}^{k\times k}$. If $A$ is invertible, $\|A^{-1}E\|_{2}<1$, then $\tilde{A}=A+E$ is also invertible, and:
\[\|\tilde{A}^{-1}-A^{-1}\|_{2}\leqslant \frac{\|E\|_{2}\|A^{-1}\|_{2}^{2}}{1-\|A^{-1}E\|_{2}}\]
\end{Lemma}
The proof of this lemma can be found in \cite{stewart1990matrix}.

\begin{Lemma}\label{weyl}
[Weyl's Theorem] Let $A, E\in\mathbb{R}^{n\times n}$ be two symmetrical matrices. Denote $\lambda_{i}(X)$ be the $i$-th largest eigenvalue of matrix $X$, then it holds that for $\forall i\in[n]$:
\[|\lambda_{i}(A+E)-\lambda_{i}(A)|\leqslant \|E\|_{2}\]
\end{Lemma}

\begin{Lemma}\label{eigen-vector}
[Perturbation of Eigenvectors] Assume symmetrical matrix $X$ has distinct eigenvalues. $\lambda$ is one of these eigenvalues and $x$ is its corresponding eigenvector with its norm $\|x\|_{2}=1$. Then there exists an orthogonal matrix $U=(x,U_{2})$ such that:
\[U^{T}XU=D=\left(\begin{matrix}\begin{array}{cc}\lambda& 0\\0& D_{2}\end{array}\end{matrix}\right)\]
which is a diagonal matrix. Once we perturb the matrix $X$ into another symmetrical matrix $\tilde{X}=X+E$, then there exists an eigenvalue $\tilde{\lambda}$ of $\tilde{X}$ and its corresponding unit eigenvector $\tilde{x}$, such that:
\[|\tilde{\lambda}-\lambda|\leqslant \|E\|_{2}~~\text{and}~~\|\tilde{x}-\tilde{x}\|_{2}\leqslant \|\Sigma^{T}\|_{2}\cdot\|E\|_{2}+O(\|E\|_{2}^{2})\]
Here:
\[\Sigma = U(\lambda I-D)^{\dagger}U^{T}\]
and $X^{\dagger}$ stands for the Moore-Penrose inverse of matrix $X$.
\end{Lemma}

\begin{Lemma}\label{wedin}
[Wedin's Theorem] Let $A,E\in\mathbb{R}^{m\times n}$. Let $A$ have the singular value decomposition:
\[\left(\begin{matrix}U_{1}^{T}\\U_{2}^{T}\\U_{3}^{T}\end{matrix}\right)A(V_{1}~V_{2})=\left(\begin{matrix}\begin{array}{cc}\Sigma_{1}&0\\0&\Sigma_{2}\\0&0\end{array}\end{matrix}\right)\]
Let $\tilde{A}=A+E$, with analogous singular value decomposition $(\tilde{U}_{1},\tilde{U}_{2},\tilde{U}_{3},\tilde{\Sigma}_{1},\tilde{\Sigma}_{2}.\tilde{V}_{1},\tilde{V}_{2})$. Let $\Phi$ be the matrix of canonical angles between range$(U_{1})$ and range$(\tilde{U}_{1})$, and $\Theta$  be the matrix of canonical angles between range$(V_{1})$ and range$(\tilde{V}_{1})$. If there exists $\delta,\alpha>0$, such that 
\[\min_{i}\sigma_{i}(\tilde{\Sigma}_{1})\geqslant\alpha+\delta,~ \max_{i}\sigma_{i}(\Sigma_{2})\leqslant \alpha\]
then we have the stability property:
\[\max\left(\|\sin\Phi\|_{2},\|\sin\Theta\|_{2}\right)\leqslant \frac{\|E\|_{2}}{\delta}\]
\end{Lemma}

\begin{Lemma}\label{orthogonal_distance}
[Distance between Right-side Orthogonal Equivalence Classes] Given two matrices $A,A^*\in\mathbb{R}^{N\times d}~(N>d)$, which satisfy the following inequality.
\[\|AA^{T}-A^*(A^*)^{T}\|_{2}<\epsilon\]
Also, assume all the $d$ singular values of $A^*$ are larger than $\delta$. Then, it holds that:
\[\min_{G_{1}^{T}G_{1}=G_{2}^{T}G_{2}=I_{d}}\|AG_{1}-A^*G_{2}\|_{F}\leqslant\sqrt{d\epsilon} + O(\epsilon)\]
\end{Lemma}
\begin{Proof}
Consider the singular value decomposition of matrices $A,A^*$.
\[A=(U_{1},U_{2})\left(\begin{matrix}\Sigma_{1}\\0\end{matrix}\right)V_{1}^{T},~A^*=(U_{1}^*,U_{2}^*)\left(\begin{matrix}\Sigma_{1}^*\\0\end{matrix}\right)V_{1}^{*T}\]
Here, $\Sigma_{1},V_{1},\Sigma_{1}^*,V_{1}^*\in\mathbb{R}^{d\times d}$ and $U_{2},U_{2}^*\in\mathbb{R}^{N\times (N-d)}, U_{1},U_{1}^*\in\mathbb{R}^{N\times d}$. Then:
\[AA^{T}=(U_{1},U_{2})\left(\begin{matrix}\begin{array}{cc}\Sigma_{1}^{2}&0\\0&0\end{array}\end{matrix}\right)\left(\begin{matrix}U_{1}^{T}\\U_{2}^{T}\end{matrix}\right), A^*(A^*)^{T}=(U_{1}^*,U_{2}^*)\left(\begin{matrix}\begin{array}{cc}\Sigma_{1}^{*2}&0\\0&0\end{array}\end{matrix}\right)\left(\begin{matrix}U_{1}^{*T}\\U_{2}^{*T}\end{matrix}\right)\]
According to Weyl's Theorem (Lemma \ref{weyl}), for $\forall i\in[N]$, it holds that: $|\lambda_{i}(AA^{T})-\lambda_{i}(A^*(A^*)^{T})|\leqslant\|AA^{T}-A^*(A^*)^{T}\|_{2}<\epsilon$, which means:
\[\|\Sigma_{1}^{2}-\Sigma_{1}^{*2}\|_{2}<\epsilon~\Rightarrow \|\Sigma_{1}-\Sigma_{1}^*\|_{2}<\sqrt{\epsilon}\]
On the other hand, according to the Wedin's Theorem (Lemma \ref{wedin}) and the assumption that each singular value of $A^*$ is larger than $\delta$, we can properly choose the $U_{1},U_{1}^*$ above such that: for $\forall i\in[d]$:
\[\langle(U_{1})_{i},(U_{1}^*)_{i}\rangle \geqslant \sqrt{1-\frac{\epsilon^{2}}{(\delta-\sqrt{\epsilon})^{2}}}~\Rightarrow \|(U_{1})_{i}-(U_{1}^*)_{i}\|_{2}^{2}\leqslant 2\left(1-\sqrt{1-\frac{\epsilon^{2}}{(\delta-\sqrt{\epsilon})^{2}}}\right)=\frac{\epsilon^{2}}{\delta^{2}}+O(\epsilon^{5/2})\]
Here, we use the fact that all column vectors of $U_{1}, U_{1}^*$ are unit vectors. And from the inequality above, we know that: $\|U_{1}-U_{1}^*\|_{F}^{2}\leqslant d\cdot \frac{\epsilon^{2}}{\delta^ {2}}+O(\epsilon^{5/2})$
Therefore, we can get the final conclusion:
\begin{align}
&\min_{G_{1}^{T}G_{1}=G_{2}^{T}G_{2}=I_{d}}\|AG_{1}-BG_{2}\|_{2} \leqslant \|AV_{1}-A^*V_{1}^*\|_{2}=\|U_{1}\Sigma_{1}-U_{1}^*\Sigma_{1}^*\|_{2}\notag\\
\leqslant& \|U_{1}\|_{2}\cdot\|\Sigma_{1}-\Sigma_{1}^*\|_{2}+\|\Sigma_{1}^*\|_{2}\cdot \|U_{1}-U_{1}^*\|_{2} < \sqrt{d\epsilon}+\|A^*\|_{2}\cdot \sqrt{\frac{d\epsilon^{2}}{\delta^{2}}+O(\epsilon^{5/2})}\notag\\
=&\sqrt{d\epsilon} + O(\epsilon)\notag
\end{align}
which comes to our conclusion.
\end{Proof}

%% file: proof1_intro.tex
\begin{Proof}
For a given $k\in [D]$, we denote $x_{i}=\alpha_{k}^{(i)}, x_{i}^{*}=\alpha_{k}^{*(i)}$ for simplicity. Since it holds that $\tv_{k}^{(i)}~(i=1,2,\cdots,r)$ is an orthogonal vector group, according to Lemma \ref{gaussian-elementary} and Lemma \ref{elementary-1}, there exists an orthogonal matrix $Q$ such that $Q\tv_{k}^{(i)}=e_{i}$. Therefore,
\begin{equation}
\begin{aligned}
M^{i}(G_{k})&=\mathop{\mathbb{E}}\limits_{\omega\sim N(0,I)}\left[\sum_{j=1}^{r}x_{j}(\tv_{k}^{(i)}\cdot \omega)^{3}\right]^{i} = \mathop{\mathbb{E}}\limits_{\omega\sim N(0,I)}\left[\sum_{j=1}^{r}x_{j}((Q\tv_{k}^{(i)})\cdot \omega)^{3}\right]^{i}\\
&= \mathop{\mathbb{E}}\limits_{\omega\sim N(0,I)}\left[\sum_{j=1}^{r}x_{j}\omega_{j}^{3}\right]^{i} 
\end{aligned}
\end{equation}
According to the following property of Gaussian distribution:
\[\mathop{\mathbb{E}}\limits_{t\sim N(0,1)}t^{2k-1}=0,~\mathop{\mathbb{E}}\limits_{t\sim N(0,1)}t^{2k}=(2k-1)!!\]
We can completely extend the moment formula above. When $1\leqslant n\leqslant r$
\begin{align}\label{eqn-2}
&M^{2n}(G_{k})=\mathop{\mathbb{E}}\limits_{\omega\sim N(0,I)}\left[\sum_{j=1}^{r}x_{j}\omega_{j}^{3}\right]^{2n}\notag\\
=& \sum_{j=1}^{r}\sum_{\substack{a_{1}+\cdots+a_{j}=n\\1\leqslant a_{1},\cdots,a_{j}}}\binom{2n}{2a_{1}~~2a_{2}~~\cdots~~2a_{j}}(6a_{1}-1)!!\cdots(6a_{j}-1)!!\sum_{1\leqslant i_{1}<\cdots<i_{j}\leqslant r}x_{i_1}^{2a_{1}}x_{i_2}^{2a_{2}}\cdots x_{i_{j}}^{2a_{j}}\notag\\
=& \sum_{j=1}^{n}\sum_{\substack{a_{1}+\cdots+a_{j}=n\\1\leqslant a_{1}\leqslant\cdots\leqslant a_{j}}}(2n)!\left(\prod_{k=1}^{j}\frac{(6a_{k}-1)!!}{(2a_{k})!}\right)\frac{1}{(\#1)!(\#2)!\cdots(\#n)!}\sum_{\substack{1\leqslant i_{1},\cdots,i_{j}\leqslant r\\i_{1},\cdots,i_{j}\text{ are distinct}}}x_{i_1}^{2a_{1}}x_{i_2}^{2a_{2}}\cdots x_{i_{j}}^{2a_{j}}\notag\\
=& \sum_{j=1}^{n}\sum_{\substack{a_{1}+\cdots+a_{j}=n\\1\leqslant a_{1}\leqslant\cdots\leqslant a_{j}}}(2n)!\left(\prod_{k=1}^{j}\frac{(6a_{k}-1)!!}{(2a_{k})!}\right)\frac{1}{(\#1)!(\#2)!\cdots(\#n)!}\sum_{\substack{1\leqslant i_{1},\cdots,i_{j}\leqslant r\\i_{1},\cdots,i_{j}\text{ are distinct}}}y_{i_1}^{a_{1}}y_{i_2}^{a_{2}}\cdots y_{i_{j}}^{a_{j}}\notag\\
=& \sum_{j=1}^{n}\sum_{\substack{a_{1}+\cdots+a_{j}=n\\1\leqslant a_{1}\leqslant\cdots\leqslant a_{j}}}(2n)!\left(\prod_{k=1}^{j}\frac{(6a_{k}-1)!!}{(2a_{k})!}\right)\frac{1}{(\#1)!(\#2)!\cdots(\#n)!}F_{\ty}[a_{1},a_{2},\cdots,a_{j}]
\end{align}
Here, $\# t$ means the number of $t$s in $a_{1},a_{2},\cdots,a_{j}$. $y_{k}=x_{k}^{2},~k=1,2,\cdots,r$ and most importantly, we define:
\[F_{\ty}[a_{1},a_{2},\cdots,a_{j}]=\sum_{\substack{1\leqslant i_{1},\cdots,i_{j}\leqslant r\\i_{1},\cdots,i_{j}\text{ are distinct}}}y_{i_1}^{a_{1}}y_{i_2}^{a_{2}}\cdots y_{i_{j}}^{a_{j}}\]
which is an $r$-variable homogeneous symmetric interchangeable polynomial. We can omit the subscript $\ty$ if there is no ambiguity. Similarly, we can also define $F_{\ty}^{*}[a_1,a_2,\cdots,a_j]$ and $y_{k}^*=(x_{k}^*)^2$. \\
Specifically, $F_{\ty}[n]=y_{1}^{n}+y_{2}^{n}+\cdots+y_{r}^{n}$. Once we prove $F_{\ty}[l]=F_{\ty}^{*}[l]~(l=1,2,\cdots,r)$, we can conclude that $\{y_1,y_2,\cdots,y_r\}=\{y_{1}^*,y_{2}^*,\cdots,y_{r}^*\}$ by simply using Newton formula and Vieta theorem. Therefore, according to the conditions that $x_{i}, x_{i}^*> 0$ , we will have:
\[\{x_{1},x_{2},\cdots,x_{r}\}=\{x_{1}^*,x_{2}^*,\cdots,x_{r}^*\}\]
which comes to our conclusion. \\
In the following part, we will prove $F_{\ty}[l]=F_{\ty}^{*}[l]~(l=1,2,\cdots,r)$ by using induction on $l$.\\
(1) In the case where $l=1$, from Equation (\ref{eqn-2}), we know that:
\[M^{2}(G_k)= 15F_{\ty}[1],~ M^{2}(G_k^*)=15F_{\ty}^*[1]\]
Since $M^{2}(G_k)=M^{2}(G_k^*)$, we have $F_{\ty}[1]=F_{\ty}^*[1]$.\\
(2) If $F_{\ty}[l]=F_{\ty}^{*}[l]$ holds for $l=1,2,\cdots,n-1~(n\leqslant r)$, then we try to prove $F_{\ty}[n]=F_{\ty}^{*}[n]$. Before we do further proofs, we introduce an important lemma about $F_{\ty}$. 
\begin{Lemma}\label{poly-1}
Given $a_1,a_2,\cdots,a_j\geqslant 1$, let $n=a_1+a_2+\cdots+a_j$, then:
\[F_{\ty}[a_1,a_2,\cdots,a_j]=P\left(F_{\ty}[1],F_{\ty}[2],\cdots,F_{\ty}[n-1]\right)+(-1)^{j-1}(j-1)!\cdot F_{\ty}[n]\]
Here, $P(\cdot)$ is some polynomial.
\end{Lemma}
\begin{proof} 
We use induction on $j$. When $j=1$, the conclusion above is obvious and $P\equiv 0$. Assume the lemma holds for $j$, then for $j+1$, we have $j+1$ positive integers $a_1,a_2,\cdots,a_{j+1}$. Denote  $N=a_1+a_2+\cdots+a_{j+1}$.
\begin{align}\label{eqn-3}
&F_{\ty}[a_1,a_2,\cdots,a_{j+1}] = \sum_{\substack{1\leqslant i_{1},\cdots,i_{j+1}\leqslant r\\i_{1},\cdots,i_{j+1}\text{ are distinct}}}y_{i_1}^{a_{1}}y_{i_2}^{a_{2}}\cdots y_{i_{j+1}}^{a_{j+1}}\notag\\
=&\left(\sum_{i_1=1}^{r}y_{i_1}^{a_1}\right)\cdot\left(\sum_{\substack{1\leqslant i_{2},\cdots,i_{j+1}\leqslant r\\i_{2},\cdots,i_{j+1}\text{ are distinct}}}y_{i_2}^{a_2}y_{i_3}^{a_{3}}\cdots y_{i_{j+1}}^{a_{j+1}}\right)-\sum_{\substack{1\leqslant i_{1},\cdots,i_{j+1}\leqslant r\\i_{2},\cdots,i_{j+1}\text{ are distinct}\\i_1\in\{i_2,\cdots,i_{j+1}\}}}y_{i_1}^{a_1}y_{i_2}^{a_{2}}\cdots y_{i_{j+1}}^{a_{j+1}}\notag\\
=& F_{\ty}[a_1]\cdot F_{\ty}[a_2,\cdots,a_{j+1}]-\sum_{k=2}^{j+1}\sum_{\substack{1\leqslant i_{2},\cdots,i_{j+1}\leqslant r\\i_{2},\cdots,i_{j+1}\text{ are distinct}}}y_{i_k}^{a_{1}}y_{i_2}^{a_2}\cdots y_{i_{j+1}}^{a_{j+1}}\notag\\
=& F_{\ty}[a_1]\cdot F_{\ty}[a_2,\cdots,a_{j+1}]-\sum_{k=2}^{j+1}F_{\ty}[a_{1}+a_{k},a_{2},\cdots,a_{k-1},a_{k+1},\cdots,a_{j+1}]\notag\\
=& P_{1}\left(F_{\ty}[1],F_{\ty}[2],\cdots,F_{\ty}[n-1]\right)-\sum_{k=2}^{j+1}\left[P_{k}\left(F_{\ty}[1],F_{\ty}[2],\cdots,F_{\ty}[n-1]\right)+(-1)^{j-1}(j-1)!\cdot F_{\ty}[n]\right]\notag\\
=& P\left(F_{\ty}[1],F_{\ty}[2],\cdots,F_{\ty}[n-1]\right)+(-1)^{j}j!\cdot F_{\ty}[n]
\end{align}
which comes to our conclusion.
\end{proof}
After applying this lemma to Equation (\ref{eqn-2}), we have:
\begin{align}
M^{2n}(G_k)&=\sum_{j=1}^{n}\sum_{\substack{a_{1}+\cdots+a_{j}=n\\1\leqslant a_{1}\leqslant\cdots\leqslant a_{j}}}(2n)!\left(\prod_{k=1}^{j}\frac{(6a_{k}-1)!!}{(2a_{k})!}\right)\frac{1}{(\#1)!(\#2)!\cdots(\#n)!}\cdot\notag\\
&~~~~~~~~~~~~~~\left(P\left(F_{\ty}[1],F_{\ty}[2],\cdots,F_{\ty}[n-1]\right)+(-1)^{j-1}(j-1)!\cdot F_{\ty}[n]\right)\notag\\
&= P\left(F_{\ty}[1],F_{\ty}[2],\cdots,F_{\ty}[n-1]\right)+S_{n}\cdot F_{\ty}[n]
\end{align}
Here, $P(\cdot)$ means some polynomial and $P$s in different lines may stand for different polynomials, and:
\[S_{n}=\sum_{j=1}^{n}\sum_{\substack{a_{1}+\cdots+a_{j}=n\\1\leqslant a_{1}\leqslant\cdots\leqslant a_{j}}}(-1)^{j-1}(j-1)!\cdot(2n)!\left(\prod_{k=1}^{j}\frac{(6a_{k}-1)!!}{(2a_{k})!}\right)\frac{1}{(\#1)!(\#2)!\cdots(\#n)!}\]
Since $F_{\ty}[l]=F_{\ty}^*[l]~(l=1,2,\cdots,n-1)$ and $M^{2n}(G_k)=M^{2n}(G_k^*)$, once we prove that $S_{n}\neq 0$, we can conclude that $F_{\ty}[n]=F_{\ty}^*[n]$, which completes the induction and proves our theorem. Therefore, in the last part, we introduce a lemma to show that $S_{n}\neq 0$.
\begin{Lemma}
For any positive integer $n$,it holds that
\[S_{n}=\sum_{j=1}^{n}\sum_{\substack{a_{1}+\cdots+a_{j}=n\\1\leqslant a_{1}\leqslant\cdots\leqslant a_{j}}}(-1)^{j-1}(j-1)!(2n)!\cdot\left(\prod_{k=1}^{j}\frac{(6a_{k}-1)!!}{(2a_{k})!}\right)\frac{1}{(\#1)!(\#2)!\cdots(\#n)!}\neq 0\]
\end{Lemma}
\begin{proof}
When $n=1,2$, the inequality is easy to verify. We consider the situation where $n\geqslant 3$. In fact, it's not difficult for us to see that:
\begin{align}
\frac{S_{n}}{(2n)!}&= \sum_{j=1}^{n}\sum_{\substack{a_{1}+\cdots+a_{j}=n\\1\leqslant a_{1}\leqslant\cdots\leqslant a_{j}}}(-1)^{j-1}(j-1)!\cdot\left(\prod_{k=1}^{j}\frac{(6a_{k}-1)!!}{(2a_{k})!}\right)\frac{1}{(\#1)!(\#2)!\cdots(\#n)!}\notag\\
&= \frac{(6n-1)!!}{(2n)!}+\sum_{j=2}^{n}\sum_{\substack{a_{1}+\cdots+a_{j}=n\\1\leqslant a_{1}\leqslant\cdots\leqslant a_{j}}}(-1)^{j-1}(j-1)!\cdot\left(\prod_{k=1}^{j}\frac{(6a_{k}-1)!!}{(2a_{k})!}\right)\frac{1}{(\#1)!(\#2)!\cdots(\#n)!}\notag\\
&\triangleq T_{1}+\sum_{j=2}^{n}T_{j}
\end{align}
In order to have $\frac{S_{n}}{(2n)!}\neq 0$, we only need to prove that: $T_{1}>|T_{2}|+\cdots+|T_{n}|$. Before we do that, we analyze the following sequence.
\[F(k)=\frac{(6k-1)!!}{(2k)!}~~(k=1,2,\cdots)\]
Since $\frac{F_{k+1}}{F_{k}}=\frac{(6k+5)(6k+3)(6k+1)}{(2k+1)(2k+2)}$ which is increasing with $k$. Therefore, according to the properties of convex sequence, when $a_1+a_2+\cdots+a_j=n$:
\[F(a_1)F(a_2)\cdots F(a_j)\leqslant F(1)\cdot F(1)\cdots F(1)\cdot F(n-j+1)=\left(\frac{15}{2}\right)^{j-1}\cdot\frac{(6n-6j+5)!!}{(2n-2j+2)!}\]
Finally, we have:
\begin{align}
\sum_{j=2}^{n}\frac{|T_{j}|}{T_{1}} &= \sum_{j=2}^{n}\sum_{\substack{a_{1}+\cdots+a_{j}=n\\1\leqslant a_{1}\leqslant\cdots\leqslant a_{j}}}\frac{(j-1)!}{T_{1}}\cdot\left(\prod_{k=1}^{j}F(a_{k})\right)\frac{1}{(\#1)!(\#2)!\cdots(\#n)!}\notag\\
&\leqslant \sum_{j=2}^{n}\sum_{\substack{a_{1}+\cdots+a_{j}=n\\1\leqslant a_{1}\leqslant\cdots\leqslant a_{j}}}\frac{(j-1)!}{T_{1}}\cdot\left(\frac{15}{2}\right)^{j-1}\cdot\frac{(6n-6j+5)!!}{(2n-2j+2)!}\notag\\
&< \sum_{j=2}^{n}\frac{(j-1)!}{T_{1}}\cdot\left(\frac{15}{2}\right)^{j-1}\cdot\frac{(6n-6j+5)!!}{(2n-2j+2)!}\cdot \binom{n-1}{j-1}\notag\\
&= \sum_{j=2}^{n}\left(\frac{15}{2}\right)^{j-1}\frac{(2n)(2n-1)\cdots(2n-2j+3)}{(6n-1)(6n-3)\cdots(6n-6j+7)}\cdot (n-1)(n-2)\cdots (n-j+1)\notag\\
&= \sum_{j=0}^{n-2}\left(\frac{15}{2}\right)^{n-j-1}\frac{(2n)(2n-1)\cdots(2j+3)}{(6n-1)(6n-3)\cdots(6j+7)}\cdot (n-1)(n-2)\cdots (j+1)\notag\\
&= \sum_{j=0}^{n-2}\left(\frac{15}{2}\right)^{n-j-1}\left(\frac{2j+3}{6j+7}\cdot\frac{2j+4}{6j+9}\cdots\frac{2n}{2j+4n+1}\right)\cdot\left(\frac{j+1}{2j+4n+3}\cdots\frac{n-1}{6n-1}\right)\notag\\
&< \sum_{j=0}^{n-2}\left(\frac{15}{2}\right)^{n-j-1}\cdot \left(\frac{1}{2}\right)^{2n-2j-2}\cdot \left(\frac{1}{4}\right)^{n-j-1} < \sum_{j=0}^{n-2}\left(\frac{1}{2}\right)^{n-j-1} < 1
\end{align}
which comes to our conclusion. The lemma is proved.
\end{proof}
\end{Proof}

%% file: proof3_intro.tex
\begin{Proof}
This proof contains several individual parts below, and we use different lemmas to demonstrate. Since we assume that the target generator satisfies the $(\tau,\mathcal A)$-robustness, it holds that:
\[|\alpha_{k}^{*(i)}|>\tau, ~|\alpha_{k}^{*(i)}-\alpha_{k}^{*(j)}|>\tau,~ |\alpha_{k}^{*(i)}|<\mathcal A\]
% \begin{Lemma}
% With probability greater than $1-\delta$ over the choice of ground truth parameters $\alpha_{k}^{*(i)}$, we have:
% \begin{equation}
% \label{eqn-1}
% |\alpha_{k}^{*(i)}|>\tau, ~|\alpha_{k}^{*(i)}-\alpha_{k}^{*(j)}|>\tau,~ |\alpha_{k}^{*(i)}|<\mathcal A
% \end{equation}
% holds for $\forall k\in[D], i\neq j\in[r]$. Here, $\tau = \frac{\sigma\delta}{2Dr^{2}}$ and $\mathcal A=\sigma\sqrt{2\log(4rD/\delta)}$.
% \end{Lemma}
% \begin{proof}
% Since each $\alpha_{k}^{*(i)}$ is sampled from Gaussian distribution $\mathcal N(0, \sigma^{2})$, we can tell that:
% \[P(|\alpha_{k}^{*(i)}|\leqslant \tau)\leqslant 2\tau\frac{1}{\sqrt{2\pi\sigma^{2}}}<\frac{\tau}{\sigma},~~ P(|\alpha_{k}^{*(i)}|\geqslant\mathcal A)\leqslant 2\exp(-\mathcal A^{2}/(2\sigma^{2}))\]
% Also, for $\forall i\neq j\in[r]$, $\alpha_{k}^{*(i)}-\alpha_{k}^{*(j)}$ obeys the Gaussian distribution $\mathcal N(0, 2\sigma^{2})$. Therefore:
% \[P(|\alpha_{k}^{*(i)}-\alpha_{k}^{*(j)}|\leqslant \tau)\leqslant 2\tau\frac{1}{\sqrt{4\pi\sigma^{2}}}<\frac{2\tau}{\sigma}\]
% In order to make all the inequalities hold simultaneously with probability greater than $1-\delta$, we make $D(r\cdot\frac{\tau}{\sigma}+\frac{r(r-1)}{2}\cdot \frac{2\tau}{\sigma}) = \delta/2$, which means: $\tau=\frac{\sigma\delta}{2Dr^{2}}$, and $2rD\exp(-\mathcal A^{2}/(2\sigma^{2}))=\delta/2$, which means $\mathcal A =\sigma\sqrt{2\log(4rD/\delta)}$. 
% \end{proof}
In the following lemma, we use concentration inequality for polynomials to estimate the difference between the empirical mean 
\[M^{2n}(G_{k})=\frac{1}{N}\sum_{i=1}^{N}[G^*_{k}(\omega^{(i)})]^{2n}\]
and its expectation
\[M^{2n}(G_{k}^*) = \mathop{\mathbb{E}}\limits_{\omega\sim\mathcal N(0,I)}[G^*_{k}(\omega)]^{2n}\]
\begin{Lemma}\label{moment-gap}
With probability greater than $1-\delta$, the following inequalities hold simultaneously. 
\[|M^{2n}(G_{k})-M^{2n}(G_{k}^*)|< (C_{1}\mathcal A\cdot r^{3})^{3r}\cdot\sqrt{\frac{\log(re^{2}/\delta)}{N}}\triangleq \mu~~\forall n=1,2,\cdots,r\]
where $C_{1}$ is an absolute constant.
\end{Lemma}
\begin{proof}
According to Theorem \ref{thm-concent}, considering the polynomial:
\[f(\omega)=\frac{1}{N}\sum_{j=1}^{N}[\alpha_{k}^{*(1)}\omega_{j1}^{3}+\alpha_{k}^{*(2)}\omega_{j2}^{3}+\cdots+\alpha_{k}^{*(r)}\omega_{jr}^{3}]^{2n}\]
Since polynomial $f(\omega)$ is homogeneous with degree $q=6n$ and maximal variable power $\Gamma=6n$. It holds that: $\mu_{1}(w,\omega)=\mu_{2}(w,\omega)=\cdots=\mu_{6n-1}(w,\omega)=0$ since  each nonempty hyperedge $h$ satisfies $q(h)=6n$. Then we estimate the upper bound of $\mu_{0}(w,\omega)$ and $\mu_{6n}(w,\omega)$.
\begin{align}
\mu_{0}(w,\omega)&=\sum_{h}|w_{h}|\prod_{v\in\mathcal V(h)}\mathbb{E}(|\omega_{v}|^{\tau_{hv}}) = \mathbb{E}\left(\frac{1}{N}\sum_{j=1}^{N}\left[\alpha_{k}^{*(1)}|\omega_{j1}|^{3}+\alpha_{k}^{*(2)}|\omega_{j2}|^{3}+\cdots+\alpha_{k}^{*(r)}|\omega_{jr}|^{3}\right]^{2n}\right)\notag\\
&=\mathbb{E}\left(\left[\alpha_{k}^{*(1)}|\omega_{1}|^{3}+\alpha_{k}^{*(2)}|\omega_{2}|^{3}+\cdots+\alpha_{k}^{*(r)}|\omega_{r}|^{3}\right]^{2n}\right)\notag\\
&< r^{2n}\mathcal A^{2n}\cdot (6n-1)!!<(r\mathcal A)^{2n}\cdot (6n)^{3n}\notag\\
\mu_{6n}(w,\omega)&=\max_{h_{0}\in\mathcal H(H)} |w_{h_0}|<\frac{1}{N}\mathcal A^{2n}\cdot(2n)!<\frac{1}{N}(2n\mathcal A)^{2n}\notag
\end{align}
Here, we use the simple inequality that for $a_{1}+a_{2}+\cdots+a_{r}=3n,~ (2a_{1}-1)!!\cdot (2a_{2}-1)!!\cdots (2a_{r}-1)!!\leqslant (6n-1)!!$ and the obvious fact that $h_{0}$ itself is the only hyperedge $h$ that satisfies $h\succeq h_{0}$. According to Theorem \ref{gauss-L}, normal Gaussian distribution $\mathcal N(0,1)$ is moment bounded by parameter $L=1$. Then:
\begin{equation}
\begin{aligned}
&e^{-\frac{t^{2}}{\mu_{0}\mu_{6n}\cdot L^{6n}\Gamma^{6n}\cdot R^{6n}}}<\exp\left(-\frac{t^{2}N}{(2n\mathcal A)^{2n}\cdot(r\mathcal A)^{2n}\cdot (6n)^{3n}\cdot(6Rn)^{6n}}\right)\\
&e^{-\left(\frac{t}{\mu_{r}L^{6n}\Gamma^{6n}\cdot R^{6n}}\right)^{1/6n}}<\exp\left(-\left(\frac{tN}{(2n\mathcal A)^{2n}\cdot (6Rn)^{6n}}\right)^{1/6n}\right)
\end{aligned}
\end{equation}
Therefore, using Theorem \ref{thm-concent}, with probability at least $1-\delta$, it holds that:
\begin{equation}
\begin{aligned}
|f(\omega)-\mathbb{E}f(\omega)|&<\max\left((2nr\mathcal A^{2})^{n} (6n)^{3n/2}(6Rn)^{3n}\sqrt{\frac{\log(e^{2}/\delta)}{N}},(2n\mathcal A)^{2n}(6Rn)^{6n}\cdot\frac{(\log(e^{2}/\delta))^{6n}}{N}\right)\\
&<\max\left((C_{1}n^{2}\mathcal A\sqrt{r})^{3n}\sqrt{\frac{\log(e^{2}/\delta)}{N}}, (2n\mathcal A)^{2n}(6Rn)^{6n}\cdot\frac{(\log(e^{2}/\delta))^{6n}}{N}\right)
\end{aligned}
\end{equation}
Here, $C_{1}$ is also an absolute constant. Note that when $N$ is sufficiently large, the former one is bigger than the latter. Therefore, 
\begin{equation}
\begin{aligned}
|M^{2n}(G_{k})-M^{2n}(G_{k}^*)|&=|f(\omega)-\mathbb{E}f(\omega)|<(C_{1}n^{2}\mathcal A\sqrt{r})^{3n}\sqrt{\frac{\log(e^{2}/\delta)}{N}}\\
&<(C_{1}\mathcal A\cdot r^{3})^{3r}\sqrt{\frac{\log(e^{2}/\delta)}{N}}
\end{aligned}
\end{equation}
holds for $n=1,2,\cdots,r$ with probability higher than $1-\delta$. In order to make the $r$ inequalities hold simultaneously, we replace $\delta$ with $\delta/r$ and it comes to our conclusion.
\end{proof}
Next, we will learn how the sampling error influences the difference between $\alpha_{k}^{(i)}$ and $\alpha_{k}^{*(i)}$ with Equation (\ref{eqn-2}). Before that, we need to explicitly express the polynomial $P(\cdot)$ in Lemma \ref{poly-1} and furthermore show the relationship between moments $M^{2n}(G_{k})$ and homogeneous symmetric interchangeable polynomials $F_{\ty}[n]$. Here $y_{i}=(\alpha_{k}^{(i)})^{2}$.
\begin{Lemma}\label{partition}
\[F[a_{1},a_{2},\cdots,a_{t}] \equiv \sum_{\substack{T=T_{1}\cup T_{2}\cup\cdots\cup T_{k}\\\text{is a partition of}\\\{a_{1},\cdots,a_{t}\}}}(-1)^{t-k}(|T_{1}|-1)!\cdots(|T_{k}|-1)!\cdot F[s(T_{1})]\cdots F[s(T_{k})]\]
Here, $s(T)$ stands for the sum of all elements in set $T$, and it's easy to tell that Lemma \ref{poly-1} is a much simpler version of this one. 
\end{Lemma}
\begin{proof}
We use induction on $t$. When $t=1,2$, this conclusion is obvious. If this lemma holds for $t$, we will prove that this lemma also holds for $t+1$. According to Equation (\ref{eqn-3}),
\begin{align}
&~F[a_{1},a_{2},\cdots,a_{t+1}] = F[a_{1},a_{2},\cdots,a_{t}]\cdot F[a_{t+1}]-\sum_{j=1}^{t}F[a_{1},\cdots,a_{j}+a_{t+1},\cdots,a_{t}]\notag\\
=& \sum_{\substack{T=T_{1}\cup T_{2}\cup\cdots\cup T_{k}\\\text{is a partition of}\\\{a_{1},\cdots,a_{t}\}}}(-1)^{t-k}(|T_{1}|-1)!\cdots(|T_{k}|-1)!\cdot F[s(T_{1})]\cdots F[s(T_{k})]\cdot F[a_{t+1}]\notag\\
&~~~~~~~ - (-1)^{t-k}(|T_{1}|-1)!\cdots(|T_{k}|-1)!\sum_{j=1}^{t} F[S(T_{1})]\cdots F[S(T_{j})+a_{t+1}]\cdots F[S(T_{k})]\cdot |T_{j}|\notag\\
=& \sum_{\substack{T=T_{1}\cup T_{2}\cup\cdots\cup T_{k}\\\text{is a partition of}\\\{a_{1},\cdots,a_{t}\}}}(-1)^{t+1-(k+1)}(|T_{1}|-1)!\cdots(|T_{k}|-1)!(1-1)!\cdot F[s(T_{1})]\cdots F[s(T_{k})]\cdot F[s(\{a_{t+1}\})]\notag\\
&~~~~~~~ + (-1)^{t+1-k}\sum_{j=1}^{t}(|T_{1}|-1)! \cdots(|T_{j-1}|-1)!\cdot (|T_{j}\cup \{a_{t+1}\}|-1)!\cdots (|T_{k}|-1)!\cdot\notag\\
&~~~~~~~~~~~~~~~~~~~~~~~~~~~~~F[S(T_{1})]\cdots F[S(T_{j}\cup\{a_{t+1}\})]\cdots F[S(T_{k})]\notag\\
=& \sum_{\substack{T=T_{1}\cup T_{2}\cup\cdots\cup T_{k}\\\text{is a partition of}\\\{a_{1},\cdots,a_{t+1}\}\\a_{t+1}\text{ is alone.}}}(-1)^{t+1-k}(|T_{1}|-1)!\cdots(|T_{k}|-1)!\cdot F[s(T_{1})]\cdots F[s(T_{k})] \notag\\
&~~~~~~~+ \sum_{\substack{T=T_{1}\cup T_{2}\cup\cdots\cup T_{k}\\\text{is a partition of}\\\{a_{1},\cdots,a_{t+1}\}\\a_{t+1}\text{ is not alone.}}}(-1)^{t+1-k}(|T_{1}|-1)!\cdots(|T_{k}|-1)!\cdot F[s(T_{1})]\cdots F[s(T_{k})] \notag\\
=&\sum_{\substack{T=T_{1}\cup T_{2}\cup\cdots\cup T_{k}\\\text{is a partition of}\\\{a_{1},\cdots,a_{t+1}\}}}(-1)^{t+1-k}(|T_{1}|-1)!\cdots(|T_{k}|-1)!\cdot F[s(T_{1})]\cdots F[s(T_{k})]
\end{align}
which completes the induction and comes to our conclusion.
\end{proof}
After that, we use the following two lemmas to estimate the error between $\alpha_{k}^{(i)}$ and $\alpha_{k}^{*(i)}$, which is caused by the sampling gap between $M^{2n}(G_{k})$ and $M^{2n}(G_{k}^*)$. In the first lemma, we prove the gap between $F_{\ty}[n]$ and $F_{\ty^*}[n]$ while in the second lemma, we prove the gap between $\alpha_{k}^{(i)}$ and $\alpha_{k}^{*(i)}$.
\begin{Lemma}\label{moment-upper-bound}
Assume that $F_{\ty}[n]~(n=1,2,\cdots,r)$ are calculated through $M^{2n}(G_{k})~(n=1,2,\cdots,r)$. Under the condition of Lemma \ref{moment-gap}, we have:
\[|F_{\ty}[n]-F_{\ty^*}[n]|\leqslant ((6r)^{8r}\mathcal A^{2r})^{n-1}\mu+O(\mu^{2})\]
\end{Lemma}
\begin{proof}
Again, we use induction on $n$. When $n=1$, we know that:
\[M^{2}(G_{k})=\mathop{\mathbb{E}}\left(\alpha_{k}^{(1)}\omega_{1}^{3}+\alpha_{k}^{(2)}\omega_{2}^{3}+\cdots+\alpha_{k}^{(r)}\omega_{r}^{3}\right)^{2}=15\sum_{i=1}^{r}(\alpha_{k}^{(i)})^{2}=15F_{\ty}[1]\]
Similarly, $M^{2}(G_{k}^*)=15F_{\ty^*}[1]$. Therefore,
\[|F_{\ty}[1]-F_{\ty^*}[1]|=\frac{1}{15}|M^{2}(G_{k})-M^{2}(G_{k}^*)|\leqslant \frac{\mu}{15}<\mu\]
Assume the lemma holds for $n-1~(n\leqslant r)$, then we prove for $n$. According to Lemma \ref{partition}, we denote:
\[H[a_{1},a_{2},\cdots,a_{t}] =\sum_{\substack{T=T_{1}\cup T_{2}\cup\cdots\cup T_{k}\\\text{is a partition of}\\\{a_{1},\cdots,a_{t}\},k\geqslant 2}}(-1)^{t-k}(|T_{1}|-1)!\cdots(|T_{k}|-1)!\cdot F[s(T_{1})]\cdots F[s(T_{k})]\]
Then, $F_{\ty}[a_{1},\cdots,a_{t}]=H_{\ty}[a_{1},\cdots,a_{t}]+(-1)^{t-1}(t-1)!\cdot F_{\ty}[a_{1}+\cdots+a_{t}]$. By using Equation (\ref{eqn-2}):
\begin{equation}\label{eqn-moment}
M^{2n}(G_{k})=\sum_{j=1}^{r}\sum_{\substack{a_{1}+\cdots+a_{j}=n\\1\leqslant a_{1}\leqslant\cdots\leqslant a_{j}}}(2n)!\left(\prod_{k=1}^{j}\frac{(6a_{k}-1)!!}{(2a_{k})!}\right)\frac{1}{(\#1)!(\#2)!\cdots(\#n)!}H_{\ty}[a_{1},a_{2},\cdots,a_{j}]+S_{n}F_{\ty}[n]
\end{equation}
Denote $\epsilon = ((6r)^{8r}\mathcal A^{2r})^{n-2}\mu$, and then by induction assumption:
\[|F_{\ty}[i]-F_{\ty^*}[i]|\leqslant \epsilon + O(\mu^{2})~~\forall i=1,2,\cdots,n-1\]
Therefore: when $a_{1}+a_{2}+\cdots+a_{t}=n$ and $1\leqslant t\leqslant r$:
\[F_{\ty}[m]=\sum_{i=1}^{r}(\alpha_{k}^{(i)})^{2m}\leqslant r\mathcal A^{2m}\]
and then:
\begin{align}
&|H_{\ty}[a_{1},a_{2},\cdots,a_{t}]-H_{\ty^*}[a_{1},a_{2},\cdots,a_{t}]|\notag\\
\leqslant& \sum_{\substack{T=T_{1}\cup T_{2}\cup\cdots\cup T_{k}\\\text{is a partition of}\\\{a_{1},\cdots,a_{t}\},k\geqslant 2}}(|T_{1}|-1)!\cdots(|T_{k}|-1)!\cdot\left|\prod_{i=1}^{k}
F_{\ty}[s(T_{i})]-\prod_{i=1}^{k}F_{\ty^*}[s(T_{i})]\right|\notag\\
\leqslant& \sum_{\substack{T=T_{1}\cup T_{2}\cup\cdots\cup T_{k}\\\text{is a partition of}\\\{a_{1},\cdots,a_{t}\},k\geqslant 2}}(|T_{1}|-1)!\cdots(|T_{k}|-1)!\cdot\left(\epsilon \sum_{i=1}^{k}
\prod_{j\neq i}F_{\ty^*}[s(T_{j})]+O(\mu^{2})\right)\notag\\
\leqslant& \sum_{\substack{T=T_{1}\cup T_{2}\cup\cdots\cup T_{k}\\\text{is a partition of}\\\{a_{1},\cdots,a_{t}\},k\geqslant 2}}(|T_{1}|-1)!\cdots(|T_{k}|-1)!\cdot\left(k\epsilon\cdot r^{k-1}\mathcal A^{2(s(T_{1})+\cdots+s(T_{k}))}+O(\mu^{2})\right)\notag\\
<& \sum_{\substack{T=T_{1}\cup T_{2}\cup\cdots\cup T_{k}\\\text{is a partition of}\\\{a_{1},\cdots,a_{t}\},k\geqslant 2}}(t-k)!\cdot\left(k\epsilon\cdot r^{k-1}\mathcal A^{2n}+O(\mu^{2})\right)\notag\\
<& (r-1)!\cdot r\epsilon\cdot r^{r-1}\mathcal A^{2n}\cdot \#\text{(partitions of a t-element set)}+O(\mu^{2})\notag\\
<& r^{2r-1}\mathcal A^{2n}\epsilon\cdot t^{t}+O(\mu^{2})\leqslant r^{3r-1}\mathcal A^{2n}\epsilon+O(\mu^{2})
\end{align}
In the inequality above, we use the following simple fact:
\[n_{1}!n_{2}!\cdots n_{m}!\leqslant (n_{1}+n_{2}+\cdots+n_{m})!\]
Connecting with Equation (\ref{eqn-moment}), we know that:
\begin{align}
|S_{n}\cdot (F_{\ty}[n]-F_{\ty^*}[n])|&\leqslant |M^{2n}(G_{k})-M^{2n}(G_{k}^*)|+\sum_{j=1}^{r}\sum_{\substack{a_{1}+\cdots+a_{j}=n\\1\leqslant a_{1}\leqslant\cdots\leqslant a_{j}}}C(a_{1},\cdots,a_{j})\cdot\notag\\
&~~~~~~~~~~~~~~~~~~~~~~~~~\left|H_{\ty}[a_{1},a_{2},\cdots,a_{j}]-H_{\ty^*}[a_{1},a_{2},\cdots,a_{j}]\right|
\end{align}
Here:
\begin{align}
C(a_{1},a_{2},\cdots,a_{j})&=(2n)!\left(\prod_{k=1}^{j}\frac{(6a_{k}-1)!!}{(2a_{k})!}\right)\frac{1}{(\#1)!(\#2)!\cdots(\#n)!}\notag\\
&<(2r)!(6r-1)!!<(2r)^{2r}\cdot (6r)^{3r}<(6r)^{5r}
\end{align}
Also, there is a famous upper bound of the number of partitions of integer $n$ introduced in \cite{rosen2017handbook}:
\[p(n)\leqslant \exp\left(\sqrt{\frac{20}{3}n}\right)\]
Therefore:~~$|S_{n}\cdot(F_{\ty}[n]-F_{\ty^*}[n])|<\mu+p(n)(6r)^{5r}\cdot r^{3r-1}\mathcal A^{2r}\epsilon < (6r)^{8r}\mathcal A^{2r}\epsilon + O(\mu^{2})$.\\
Since $|S_{n}|\geqslant 1$ and $\epsilon = ((6r)^{8r}\mathcal A^{2r})^{n-2}\mu$, we get the conclusion that:
\[|F_{\ty}[n]-F_{\ty^*}[n]|\leqslant ((6r)^{8r}\mathcal A^{2r})^{n-1}\mu+O(\mu^{2})\]
which completes the induction.
\end{proof}
\begin{Lemma}\label{thm-3-final}
There exists a permutation $\pi:[r]\rightarrow[r]$, such that:
\[\sum_{i=1}^{r}|\alpha_{k}^{(\pi(i))}-\alpha_{k}^{*(i)}|^{2}<\left(\frac{D}{\sigma\delta}\cdot (6r)^{4r}\mathcal A^{r} \right)^{2r}\mu+O(\mu^{2})\]
\end{Lemma}
\begin{proof}
Consider the map $\gamma: \mathbb{R}_{+}^{r}\rightarrow\mathbb{R}^{r}$. 
\[\gamma: (x_{1},x_{2},\cdots,x_{r})^{T}~\mapsto~\left(\frac{1}{2}\sum_{i=1}^{r}x_{i}^{2},~\frac{1}{4}\sum_{i=1}^{r}x_{i}^{4},\cdots,~\frac{1}{2r}\sum_{i=1}^{r}x_{i}^{2r}\right)^{T}\]
is a bijection from the neighbourhood of $(\alpha_{k}^{*(1)},\alpha_{k}^{*(2)},\cdots,\alpha_{k}^{*(r)})^{T}$ to the neighbourhood of its imaging point $(F_{\ty^*}[1], F_{\ty^*}[2],\cdots,F_{\ty^*}[r])^{T}$ since the Jacobian matrix of this point:
\begin{equation}
J=\left(\begin{matrix}\alpha_{k}^{*(1)}&\alpha_{k}^{*(2)}&\cdots&\alpha_{k}^{*(r)}\\(\alpha_{k}^{*(1)})^{3}&(\alpha_{k}^{*(2)})^{3}&\cdots&(\alpha_{k}^{*(r)})^{3}\\\vdots&\vdots&\cdots&\vdots\\(\alpha_{k}^{*(1)})^{2r-1}&(\alpha_{k}^{*(2)})^{2r-1}&\cdots&(\alpha_{k}^{*(r)})^{2r-1}\end{matrix}\right)
\end{equation}
is invertible. Denote $\alpha = (\alpha_{k}^{(1)},\alpha_{k}^{(2)},\cdots,\alpha_{k}^{(r)})^{T}, \alpha^*=(\alpha_{k}^{*(1)},\alpha_{k}^{*(2)},\cdots,\alpha_{k}^{*(r)})^{T}$, then:
\[\gamma(\alpha^*)=(F_{\ty^*}[1],F_{\ty^*}[2],\cdots,F_{\ty^*}[r])^{T}, \gamma(\alpha)=(F_{\ty}[1],F_{\ty}[2],\cdots,F_{\ty}[r])^{T}\]
Since $\gamma(\alpha)$ is in the neighbourhood of $\gamma(\alpha^*)$ when $N$ is sufficiently large, there is a pre-image $\beta$ in the neighbourhood of $\alpha^*$ such that $\gamma(\beta)=\gamma(\alpha^*)$. By Newton formula, Vieta Theorem and positiveness of $\alpha^*$ as well as $\beta$, we know that $\beta$ is simply a permutation of $\alpha$. Without loss of generality, assume $\beta=\alpha$. Then we can further infer that: 
\[\gamma(\alpha)-\gamma(\alpha^*)=J(\alpha-\alpha^*)+O(\|\alpha-\alpha^*\|_{2}^{2})\]
Connecting with the result of Lemma \ref{moment-upper-bound} which demonstrates that: 
\[\|\gamma(\alpha)-\gamma(\alpha^*)\|_{2}<\sqrt{r}\cdot((6r)^{8r}\mathcal A^{2r})^{r-1}\mu+O(\mu^{2})\]
we get the following result:
\begin{align}\label{eqn-14}
\|\alpha-\alpha^*\|_{2}&=\|J^{-1}\cdot(\gamma(\alpha)-\gamma(\alpha^*))\|_{2}+O(\|\gamma(\alpha)-\gamma(\alpha^*)\|_{2}^{2})\leqslant \|J^{-1}\|_{2}\cdot\|\gamma(\alpha)-\gamma(\alpha^*)\|_{2}+O(\mu^{2})\notag\\
&\leqslant \|J^{-1}\|_{F}\cdot\sqrt{r}\cdot((6r)^{8r}\mathcal A^{2r})^{r-1}\mu+O(\mu^{2})
\end{align}
Next, we will give an upper bound of $\|J^{-1}\|_{F}$. Actually, there is an explicit expression of the inverse of Vandermonde matrix, by which we can express every element of matrix $J^{-1}$.
\[(J^{-1})_{ij}=\frac{(-1)^{j+1}}{\alpha_{k}^{*(i)}}\left(\sum_{\substack{1\leqslant p_{1}<p_{2}<\cdots<p_{r-j}\leqslant r\\p_{1},\cdots,p_{r-j}\neq i}}\left(\alpha_{k}^{*(p_{1})}\alpha_{k}^{*(p_{2})}\cdots\alpha_{k}^{*(p_{r-j})}\right)^{2}\right)\div\left(\prod_{1\leqslant t\leqslant r,t\neq i}((\alpha_{k}^{*(t)})^{2}-(\alpha_{k}^{*(i)})^{2})\right)\]
Therefore,
\[|(J^{-1})_{ij}|\leqslant \frac{\mathcal A^{2(r-j)}}{\tau^{2r}}\cdot\binom{r-1}{r-j}<\left(\frac{\mathcal A}{\tau}\right)^{2r}\cdot\binom{r-1}{r-j}\]
Connected with Equation (\ref{eqn-14}), we get the upper bound of $\|\alpha-\alpha^*\|_{2}$. 
\begin{align}
\|\alpha-\alpha^*\|_{2}&= \|J^{-1}\|_{F}\cdot\sqrt{r}\cdot((6r)^{8r}\mathcal A^{2r})^{r-1}\mu+O(\mu^{2})\notag\\
&< \left(\frac{\mathcal A}{\tau}\right)^{2r}\sqrt{\sum_{i,j=1}^{r}\binom{r-1}{r-j}^{2}}\sqrt{r}\cdot((6r)^{8r}\mathcal A^{2r})^{r-1}\mu+O(\mu^{2})\notag\\
&= \left(\frac{\mathcal A}{\tau}\right)^{2r}\sqrt{r\binom{2r-2}{r-1}}\cdot\sqrt{r}\cdot((6r)^{8r}\mathcal A^{2r})^{r-1}\mu+O(\mu^{2})\notag\\
&< \left(\frac{(6r)^{4r}\mathcal A^{r}}{\tau} \right)^{2r}\mu+O(\mu^{2})
\end{align}
which comes to our conclusion.
\end{proof}
Finally, combining Lemma \ref{moment-gap} and Lemma \ref{thm-3-final}, we can get the final conclusion: with probability at least $1-\delta$, it holds for each $k\in[D]$ that there exists a permutation $\sigma$, such that:
\[\left(\sum_{i=1}^{r}|\alpha_{k}^{(\sigma(i))}-\alpha_{k}^{*(i)}|^{2}\right)^{1/2}<\left(\frac{C_{1}(6r)^{3r}\cdot\mathcal A^{r+1}}{\tau}\right)^{3r}\cdot \sqrt{\frac{\log(Dre^{2}/\delta)}{N}}\]
Without loss of generality, assume $\alpha_{k}^{(1)}<\alpha_{k}^{(2)}<\cdots<\alpha_{k}^{(r)}, ~\alpha_{k}^{*(1)}<\alpha_{k}^{*(2)}<\cdots<\alpha_{k}^{*(r)}$ holds for $\forall k\in[D]$. Then the permutation $\sigma = id$, and till now, Theorem 3 is proved.
\end{Proof}

%% file: proof2_intro.tex
\begin{Proof}
Without loss of generality, assume that $i=1,j=2$. In order to simplify the notifications, we re-express $G_{1}$ and $G_{2}$ as:
\begin{equation}
\begin{aligned}
&G_{1}(\omega)= \sum_{i=1}^{r}\alpha^{(i)}(\tv^{(i)}\cdot\omega)^{3},~G_{2}(\omega)= \sum_{i=1}^{r}\beta^{(i)}(\tw^{(i)}\cdot\omega)^{3}\\
&G_{1}^*(\omega)= \sum_{i=1}^{r}\alpha^{(i)}(\tv^{*(i)}\cdot\omega)^{3},~G_{2}^*(\omega)= \sum_{i=1}^{r}\beta^{(i)}(\tw^{*(i)}\cdot\omega)^{3}
\end{aligned}
\end{equation}
Here, the vector groups $\tv^{(i)}, \tw^{(i)}, \tv^{*(i)}, \tw^{*(i)}$ are orthonormal. Denote $P_{ij}=\tv^{(i)}\cdot\tw^{(j)}, P_{ij}^*=\tv^{*(i)}\cdot\tw^{*(j)}$, and we only need to prove that $P_{ij}=P_{ij}^*$ holds for $\forall i,j\in[r]$. Firstly, we give a clear expression on the inter-component moment expectation $\mathbb{E}[G_{1}(\omega)^{k}G_{2}(\omega)]$. According to Lemma \ref{elementary-2}, there exists an orthogonal matrix $Q$, such that:
\[Q\tv^{(1)}=e_{1},\cdots,Q\tv^{(r)}=e_{r},~~ Q\tw^{(j)}=(P_{1j},P_{2j},\cdots,P_{rj},Q_{j},0,\cdots,0)^{T}\]
Here, $Q_{j}=\sqrt{1-P_{1j}^{2}-P_{2j}^{2}-\cdots-P_{rj}^{2}}$. 
\begin{align}
&~~~S_{k}\triangleq\mathop\mathbb{E}\limits_{\omega\sim\mathcal N(0,I)}\left(G_{1}(\omega)^{k}G_{2}(\omega)\right) = \mathbb{E}\sum_{j=1}^{r}\left(\sum_{i=1}^{r}\alpha^{(i)}(\tv^{(i)}\cdot\omega)^{3}\right)^{k}\cdot \beta^{(j)}(\tw^{(j)}\cdot\omega)^{3}\notag\\
&= \sum_{j=1}^{r}\mathbb{E}\left(\sum_{i=1}^{r}\alpha^{(i)}\omega_{i}^{3}\right)^{k}\cdot \beta^{(j)}(P_{1j}\omega_{1}+P_{2j}\omega_{2}+\cdots+P_{rj}\omega_{r}+Q_{j}\omega_{r+1})^{3}\notag\\
&= \mathbb{E}\left(\sum_{i=1}^{r}\alpha^{(i)}\omega_{i}^{3}\right)^{k}\cdot\sum_{j=1}^{r}\beta^{(j)}(P_{1j}\omega_{1}+P_{2j}\omega_{2}+\cdots+P_{rj}\omega_{r}+Q_{j}\omega_{r+1})^{3}\notag\\
&= \mathbb{E}\left(\sum_{i=1}^{r}\alpha^{(i)}\omega_{i}^{3}\right)^{k}\cdot \left(\sum_{j=1}^{r}\beta^{(j)}(P_{1j}\omega_{1}+P_{2j}\omega_{2}+\cdots+P_{rj}\omega_{r})^{3}+3Q_{j}^{2}(P_{1j}\omega_{1}+P_{2j}\omega_{2}+\cdots+P_{rj}\omega_{r})\right)\notag\\
&= \sum_{j=1}^{r}\sum_{x,y,z=1}^{r}\beta^{(j)}P_{xj}P_{yj}P_{zj}\cdot\mathbb{E}\left(\omega_{x}\omega_{y}\omega_{z}\left(\sum_{i=1}^{r}\alpha^{(i)}\omega_{i}^{3}\right)^{k}\right) +\sum_{j=1}^{r}\sum_{x=1}^{r}3\beta^{(j)}P_{xj}Q_{j}^{2}\cdot\mathbb{E}\left(\omega_{x}\left(\sum_{i=1}^{r}\alpha^{(i)}\omega_{i}^{3}\right)^{k}\right)\notag\\
&= \sum_{j=1}^{r}\sum_{x=1}^{r}\beta^{(j)}P_{xj}^{3}\cdot\mathbb{E}\left(\omega_{x}^{3}\left(\sum_{i=1}^{r}\alpha^{(i)}\omega_{i}^{3}\right)^{k}\right)+3\sum_{j=1}^{r}\sum_{x\neq y}\beta^{(j)}P_{xj}^{2}P_{yj}\cdot\mathbb{E}\left(\omega_{x}^{2}\omega_{y}\left(\sum_{i=1}^{r}\alpha^{(i)}\omega_{i}^{3}\right)^{k}\right)+\notag\\
&~6\sum_{j=1}^{r}\sum_{x<y<z}\beta^{(j)}P_{xj}P_{yj}P_{zj}\cdot\mathbb{E}\left(\omega_{x}\omega_{y}\omega_{z}\left(\sum_{i=1}^{r}\alpha^{(i)}\omega_{i}^{3}\right)^{k}\right)+3\sum_{j=1}^{r}\sum_{x=1}^{r}\beta^{(j)}P_{xj}Q_{j}^{2}\cdot\mathbb{E}\left(\omega_{x}\left(\sum_{i=1}^{r}\alpha^{(i)}\omega_{i}^{3}\right)^{k}\right)\notag
\end{align}
Let $k=1,3,\cdots, 2K_{r}-1$, we can rewrite the equations above as:
\[(S_{1},S_{3},\cdots,S_{2K_{r}-1})=\tp\cdot\left(CE[\alpha^{(1)},\alpha^{(2)},\cdots,\alpha^{(r)}]\right)\]
Here, $\tp$ is a vector with length $K_{r}$. Its elements are:
\begin{equation}
\begin{aligned}
&\tp=\Big(\sum_{j=1}^{r}\beta^{(j)}P_{1j}^{3},\cdots,\sum_{j=1}^{r}\beta^{(j)}P_{rj}^{3},~3\sum_{j=1}^{r}\beta^{(j)}P_{1j}^{2}P_{2j},\cdots,3\sum_{j=1}^{r}\beta^{(j)}P_{rj}^{2}P_{r-1,j},\\
&6\sum_{j=1}^{r}\beta^{(j)}P_{1j}P_{2j}P_{3j},\cdots,6\sum_{j=1}^{r}\beta^{(j)}P_{r-2,j}P_{r-1,j}P_{rj},~3\sum_{j=1}^{r}\beta^{(j)}P_{1j}Q_{j}^{2},\cdots, 3\sum_{j=1}^{r}\beta^{(j)}P_{rj}Q_{j}^{2}\Big)
\end{aligned}
\end{equation}
Similarly, the same relationship holds for the target generator.
\begin{equation}\label{eqn-31}
(S_{1}^*,S_{3}^*,\cdots,S_{2K_{r}-1}^*)=\tp^*\cdot\left(CE[\alpha^{(1)},\alpha^{(2)},\cdots,\alpha^{(r)}]\right)
\end{equation}
Here, $S_{k}^*=\mathbb{E}\left((G_{1}^*(\omega))^{k}G_{2}^*(\omega)\right)~(k=1,3,\cdots,2K_{r}-1)$ and 
\begin{equation}
\begin{aligned}
&\tp^*=\Big(\sum_{j=1}^{r}\beta^{(j)}P_{1j}^{*3},\cdots,\sum_{j=1}^{r}\beta^{(j)}P_{rj}^{*3},~3\sum_{j=1}^{r}\beta^{(j)}P_{1j}^{*2}P_{2j}^*,\cdots,3\sum_{j=1}^{r}\beta^{(j)}P_{rj}^{*2}P_{r-1,j}^*,\\
&6\sum_{j=1}^{r}\beta^{(j)}P_{1j}^*P_{2j}^*P_{3j}^*,\cdots,6\sum_{j=1}^{r}\beta^{(j)}P_{r-2,j}^*P_{r-1,j}^*P_{rj}^*,~3\sum_{j=1}^{r}\beta^{(j)}P_{1j}^*Q_{j}^{*2},\cdots, 3\sum_{j=1}^{r}\beta^{(j)}P_{rj}^*Q_{j}^{*2}\Big)
\end{aligned}
\end{equation}
According to the invertibility of CE-Matrix, $det(CE[\alpha^{(1)},\alpha^{(2)},\cdots,\alpha^{(r)}])\neq 0$, which leads to the fact that the following region 
\[\left\{\left(\alpha^{(1)},\alpha^{(2)},\cdots,\alpha^{(r)}\right):~det(CE[\alpha^{(1)},\alpha^{(2)},\cdots,\alpha^{(r)}])= 0\right\}\]
has zero measure. Therefore, with probability 1 over the choice of $\alpha^{(1)},\alpha^{(2)},\cdots,\alpha^{(r)}$, the Cubic Expectation Matrix $CE[\alpha^{(1)},\alpha^{(2)},\cdots,\alpha^{(r)}]$ is invertible. Since
\[(S_{1},S_{3},\cdots,S_{2K_{r}-1})=(S_{1}^*,S_{3}^*,\cdots,S_{2K_{r}-1}^*)\]
it holds that $\tp=\tp^*$. Then we have:
\begin{equation}\label{tensor-form}
\sum_{j=1}^{r}\beta^{(j)}P_{xj}P_{yj}P_{zj}=\sum_{j=1}^{r}\beta^{(j)}P_{xj}^*P_{yj}^*P_{zj}^*~~\forall x,y,z\in[r]
\end{equation}
Furthermore, since $Q_{j}^{2}=1-P_{1j}^{2}-\cdots-P_{rj}^{2}$, it also holds that for $\forall x\in[r]$:
\begin{equation}\label{tensor-extend}
\begin{aligned}
&~~\sum_{j=1}^{r}\beta^{(j)}P_{xj}Q_{j}^{2}=\sum_{j=1}^{r}\beta^{(j)}P_{xj}^*Q_{j}^{*2}\\
&\Rightarrow \sum_{j=1}^{r}\beta^{(j)}P_{xj}-\sum_{j=1}^{r}\beta^{(j)}\sum_{t=1}^{r}P_{xj}P_{tj}^{2}=\sum_{j=1}^{r}\beta^{(j)}P_{xj}^*-\sum_{j=1}^{r}\beta^{(j)}\sum_{t=1}^{r}P_{xj}^*P_{tj}^{*2}\\
&\Rightarrow \sum_{j=1}^{r}\beta^{(j)}P_{xj}=\sum_{j=1}^{r}\beta^{(j)}P_{xj}^*
\end{aligned}
\end{equation}
Notice that Equation (\ref{tensor-form}) is equivalent to the following tensor decomposition form.
\begin{equation}
\begin{aligned}
&\beta^{(1)}\left(\begin{matrix}P_{11}\\P_{21}\\\vdots\\P_{r1}\end{matrix}\right)^{\otimes 3}+\beta^{(2)}\left(\begin{matrix}P_{12}\\P_{22}\\\vdots\\P_{r2}\end{matrix}\right)^{\otimes 3}+\cdots+\beta^{(r)}\left(\begin{matrix}P_{1r}\\P_{2r}\\\vdots\\P_{rr}\end{matrix}\right)^{\otimes 3}\\
=~&\beta^{(1)}\left(\begin{matrix}P_{11}^*\\P_{21}^*\\\vdots\\P_{r1}^*\end{matrix}\right)^{\otimes 3}+\beta^{(2)}\left(\begin{matrix}P_{12}^*\\P_{22}^*\\\vdots\\P_{r2}^*\end{matrix}\right)^{\otimes 3}+\cdots+\beta^{(r)}\left(\begin{matrix}P_{1r}^*\\P_{2r}^*\\\vdots\\P_{rr}^*\end{matrix}\right)^{\otimes 3}
\end{aligned}
\end{equation}
Before using the property of tensor decomposition, we introduce a lemma to show that the following matrix $\mathbf{P}^*$ has full rank with probability 1 over the choice of $\tv^{*(i)}, \tw^{*(i)}$.
\[\mathbf{P}^*=\left(\begin{matrix}\begin{array}{cccc}P_{11}^*&P_{12}^*&\cdots&P_{1r}^*\\P_{21}^*&P_{22}^*&\cdots&P_{2r}^*\\\vdots&\vdots&~&\vdots\\P_{r1}^*&P_{r2}^*&\cdots&P_{rr}^*\end{array}\end{matrix}\right)\]
\begin{Lemma}
The matrix $\mathbf{P}^*$ has full rank with probability 1 over the choice of $\tv^{*(i)},\tw^{*(i)}~i\in[r]$. 
\end{Lemma}
\begin{proof}
If matrix $\mathbf{P}^*$ isn't invertible. Then there exists a non-zero vector $(x_{1},x_{2},\cdots,x_{r})$ such that:
\[\left(\begin{matrix}\begin{array}{cccc}P_{11}^*&P_{12}^*&\cdots&P_{1r}^*\\P_{21}^*&P_{22}^*&\cdots&P_{2r}^*\\\vdots&\vdots&~&\vdots\\P_{r1}^*&P_{r2}^*&\cdots&P_{rr}^*\end{array}\end{matrix}\right)\cdot \left(\begin{matrix}x_{1}\\x_{2}\\\vdots\\x_{r}\end{matrix}\right)=0\]
Since $P_{ij}^*=\tv^{*(i)}\cdot\tw^{*(j)}$, we know that:
\[\tv^{*(i)}\cdot (x_{1}\tw^{*(1)}+x_{2}\tw^{*(2)}+\cdots+x_{r}\tw^{*(r)})=0~~\forall i\in[r]\]
Therefore, the invertibility of matrix $\mathbf{P}^*$ is equivalent to:
\begin{equation}\label{eqn-73}
\left(span(\tv^{*(1)},\tv^{*(2)},\cdots,\tv^{*(r)})\right)^{\perp}\cap span(\tw^{*(1)},\tw^{*(2)},\cdots,\tw^{*(r)})\neq \emptyset
\end{equation}
Assume $\tv^{*(1)},\tv^{*(2)},\cdots,\tv^{*(r)},\tu^{*(1)},\cdots,\tu^{*(d-r)}$ is an orthonormal basis. Then the statement above is equivalent to:
\begin{equation}
\begin{aligned}
&span(\tu^{*(1)},\cdots,\tu^{*(d-r)})\cap span(\tw^{*(1)},\tw^{*(2)},\cdots,\tw^{*(r)})\neq \emptyset\\
\Leftrightarrow ~& \tu^{*(1)},\cdots,\tu^{*(d-r)},\tw^{*(1)},\tw^{*(2)},\cdots,\tw^{*(r)} \text{~is not a basis of }\mathbb{R}^{d}\\
\Leftrightarrow ~& det\left(\tu^{*(1)},\cdots,\tu^{*(d-r)},\tw^{*(1)},\tw^{*(2)},\cdots,\tw^{*(r)}\right)=0
\end{aligned}
\end{equation}
which probability is 0. 
\end{proof}
By the uniqueness of tensor decomposition (Section 7.3), we know that the $r$-vector group
\[\sqrt[3]{\beta^{(1)}}(P_{11},P_{21},\cdots,P_{r1})^{T},\sqrt[3]{\beta^{(2)}}(P_{12},P_{22},\cdots,P_{r2})^{T},\cdots, \sqrt[3]{\beta^{(r)}}(P_{1r},P_{2r},\cdots,P_{rr})^{T}\]
is a permutation of 
\[\sqrt[3]{\beta^{(1)}}(P_{11}^*,P_{21}^*,\cdots,P_{r1}^*)^{T},\sqrt[3]{\beta^{(2)}}(P_{12}^*,P_{22}^*,\cdots,P_{r2}^*)^{T},\cdots, \sqrt[3]{\beta^{(r)}}(P_{1r}^*,P_{2r}^*,\cdots,P_{rr}^*)^{T}\]
Denote $\alpha_{i}^*=\sqrt[3]{\beta^{(i)}}(P_{1i}^*,P_{2i}^*,\cdots,P_{ri}^*)^{T}$ and assume $\sqrt[3]{\beta^{(i)}}(P_{1i},P_{2i},\cdots,P_{ri})^{T} = \alpha^*_{\sigma(i)}$. Here, $\sigma:[r]\rightarrow[r]$ is a permutation. According to Equation (\ref{eqn-73}):
\begin{equation}
\begin{aligned}
&\beta^{(1)}(P_{11},P_{21},\cdots,P_{r1})^{T}+\beta^{(2)}(P_{12},P_{22},\cdots,P_{r2})^{T}+\cdots+\beta^{(r)}(P_{1r},P_{2r},\cdots,P_{rr})^{T}\\
=~& \beta^{(1)}(P_{11}^*,P_{21}^*,\cdots,P_{r1}^*)^{T}+\beta^{(2)}(P_{12}^*,P_{22}^*,\cdots,P_{r2}^*)^{T}+\cdots+\beta^{(r)}(P_{1r}^*,P_{2r}^*,\cdots,P_{rr}^*)^{T}
\end{aligned}
\end{equation}
which leads to:
\[(\beta^{(1)})^{2/3}\alpha_{1}^*+(\beta^{(2)})^{2/3}\alpha_{2}^*+\cdots+(\beta^{(r)})^{2/3}\alpha_{r}^*=(\beta^{(1)})^{2/3}\alpha_{\sigma(1)}^*+(\beta^{(2)})^{2/3}\alpha_{\sigma(2)}^*+\cdots+(\beta^{(r)})^{2/3}\alpha_{\sigma(r)}^*\]
Since $\beta^{(i)}~i\in[r]$ are distinct and vectors $\alpha_{1}^*,\alpha_{2}^*,\cdots,\alpha_{r}^*$ are linearly independent. We can conclude that: $\sigma = id$. Therefore, for $\forall i,j\in[r]$, it holds that $P_{ij}=P_{ij}^*$, which comes to our conclusion.
\end{Proof}

%% file: proof4_intro.tex
\begin{Proof}
Without loss of generality, we consider $G_{1}, G_{2}, G_{1}^*, G_{2}^*$. To simplify our notifications, we rewrite them as:
\begin{equation}
\begin{aligned}
&G_{1}(\omega)=\sum_{i=1}^{r}\alpha^{(i)}(\tv^{(i)}\cdot\omega)^{3}, ~G_{2}(\omega)=\sum_{i=1}^{r}\beta^{(i)}(\tw^{(i)}\cdot\omega)^{3}\\
&G_{1}^*(\omega)=\sum_{i=1}^{r}\alpha^{*(i)}(\tv^{*(i)}\cdot\omega)^{3}, ~G_{2}^*(\omega)=\sum_{i=1}^{r}\beta^{*(i)}(\tw^{*(i)}\cdot\omega)^{3}
\end{aligned}
\end{equation}

Similar to the proof of Theorem 2, we firstly estimate the difference between the expectation:
\begin{equation*}
\begin{aligned}
&S_{2k-1}(G_{1}^*,G_{2}^*)=\mathbb{E}\left(G_{1}^*(\omega)\right)^{2k-1}G_{2}^*(\omega)\\
=&\mathbb{E} \left(\sum_{i=1}^{r}\alpha^{*(i)}\omega_{i}^{3}\right)^{2k-1}\cdot \left(\sum_{j=1}^{r}\beta^{*(j)}(P_{1j}^*\omega_{1}+P_{2j}^*\omega_{2}+\cdots+P_{rj}^*\omega_{r}+Q_{j}^*\omega_{r+1})^{3}\right)
\end{aligned}
\end{equation*}
and the empirical mean:
\begin{equation*}
\begin{aligned}
&S_{2k-1}(G_{1},G_{2})=\mathbb{E}\left(G_{1}(\omega)\right)^{2k-1}G_{2}(\omega)\\
=~&\frac{1}{N}\sum_{t=1}^{N}\left(\sum_{i=1}^{r}\alpha^{*(i)}\omega_{ti}^{3}\right)^{2k-1}\cdot \left(\sum_{j=1}^{r}\beta^{*(j)}(P_{1j}^*\omega_{t1}+P_{2j}^*\omega_{t2}+\cdots+P_{rj}^*\omega_{tr}+Q_{j}^*\omega_{t,r+1})^{3}\right)
\end{aligned}
\end{equation*}
\begin{Lemma}\label{lemma-14}
When $N$ is sufficiently large, with probability greater than $1-\delta$, the following inequalities hold simultaneously for each $k=1,2,\cdots,K_{r}$.
\[|S_{2k-1}(G_{1},G_{2})-S_{2k-1}(G_{1}^*,G_{2}^*)|<(C_{2}r^{7}\mathcal A)^{6r^{3}}\cdot\sqrt{\frac{\log(2e^{2}r^{3}/\delta)}{N}}\]
Here $C_{2}$ is an absolute constant.
\end{Lemma}
\begin{proof}
According to Theorem \ref{moment-gap}, considering the following polynomial:
\begin{equation*}
\begin{aligned}
f(\omega)&=\frac{1}{N}\sum_{t=1}^{N}\left(\sum_{i=1}^{r}\alpha^{*(i)}\omega_{ti}^{3}\right)^{2k-1}\cdot \left(\sum_{j=1}^{r}\beta^{*(j)}(P_{1j}^*\omega_{t1}+P_{2j}^*\omega_{t2}+\cdots+P_{rj}^*\omega_{tr}+Q_{j}^*\omega_{t,r+1})^{3}\right)
\end{aligned}
\end{equation*}
This polynomial $f(\omega)$ is homogeneous with degree $q=6k$ and maximal variable power  $\Gamma = 6k$. Therefore:
\[\mu_{1}(w,\omega)=\mu_{2}(w,\omega)=\cdots=\mu_{6k-1}(w,\omega)=0\]\\
since each nonempty hyperedge $h$ has power $6k$. Then we estimate the upper bound of $\mu_{0}(w,\omega)$ and $\mu_{6k}(w,\omega)$.
\begin{equation*}
\begin{aligned}
\mu_{0}(w,\omega)&=\sum_{h}|w_{h}|\prod_{v\in \mathcal V(h)}\mathbb{E}(|\omega_{v}|^{\tau_{hv}})\\
&=\mathbb{E}\frac{1}{N}\sum_{t=1}^{N}\left(\sum_{i=1}^{r}\alpha^{*(i)}|\omega_{ti}|^{3}\right)^{2k-1}\cdot \left(\sum_{j=1}^{r}\beta^{*(j)}(P_{1j}^*|\omega_{t1}|+P_{2j}^*|\omega_{t2}|+\cdots+P_{rj}^*|\omega_{tr}|+Q_{j}^*|\omega_{t,r+1}|)^{3}\right)\\
&=\mathbb{E}\left(\sum_{i=1}^{r}\alpha^{*(i)}|\omega_{i}|^{3}\right)^{2k-1}\cdot \left(\sum_{j=1}^{r}\beta^{*(j)}(P_{1j}^*|\omega_{1}|+P_{2j}^*|\omega_{2}|+\cdots+P_{rj}^*|\omega_{r}|+Q_{j}^*|\omega_{r+1}|)^{3}\right)\\
&\leqslant r^{2k-1}\cdot r(r+1)^{3}\mathcal A^{2k}\cdot(6k-1)!! < (r+1)^{3}(r\mathcal A)^{2k}\cdot (6k)^{3k}\\
\mu_{6k}(w,\omega)&=\max_{h_{0}\in\mathcal H(H)}|w_{h_{0}}|<\frac{1}{N}\cdot(2k-1)!\mathcal A^{2k-1}\cdot 6A = \frac{6}{N}(2kA)^{2k}
\end{aligned}
\end{equation*}
By Theorem \ref{thm-concent}, when $N$ is sufficiently large, it holds that: with probability larger than $1-\delta$,
\begin{equation}
\begin{aligned}
|f(\omega)-\mathbb{E}f(\omega)|&< \max\left((5Rk^{2}r\mathcal A)^{3k}\cdot\sqrt{\frac{\log(e^{2}/\delta)}{N}}, 6(12Rk^{2}\mathcal A)^{6k}\cdot\frac{(\log(e^{2}/\delta))^{6k}}{N}\right)\\
&= (5Rk^{2}r\mathcal A)^{3k}\cdot\sqrt{\frac{\log(e^{2}/\delta)}{N}}= (C_{2}k^{2}r\mathcal A)^{3k}\cdot\sqrt{\frac{\log(e^{2}/\delta)}{N}}
\end{aligned}
\end{equation}
Here, $k=1,2,\cdots,K_{r}$. $K_{r}=\frac{1}{6}r(r^{2}+3r+8)\leqslant 2r^{3}$. In order to make the inequality holds simultaneously for $k=1,2,\cdots,K_{r}$, we use $\delta/K_{r}$ to replace $\delta$. Therefore, with probability larger than $1-\delta$, 
\[|S_{2k-1}(G_{1},G_{2})-S_{2k-1}(G_{1}^*,G_{2}^*)|<(C_{2}k^{2}r\mathcal A)^{3k}\cdot\sqrt{\frac{\log(e^{2}K_{r}/\delta)}{N}}<(C_{2}r^{7}\mathcal A)^{6r^{3}}\cdot\sqrt{\frac{\log(2e^{2}r^{3}/\delta)}{N}}\]
holds for each $k=1,2,\cdots,K_{r}$. Here, $C_{2}$ is an absolute constant. Denote $S_{k}=S_{k}(G_{1},G_{2}),S_{k}^*=S_{k}(G_{1}^*,G_{2}^*)$ and
\[\ts = (S_{1},S_{3},\cdots,S_{2K_{r}-1}), \ts^* = (S_{1}^*,S_{3}^*,\cdots,S_{2K_{r}-1}^*)\]
\end{proof}
From the lemma above, we know that with probability larger than $1-\delta$:
\[\|\ts-\ts^*\|_{2}<\sqrt{K_{r}}\cdot(C_{2}r^{7}\mathcal A)^{6r^{3}}\cdot\sqrt{\frac{\log(2e^{2}r^{3}/\delta)}{N}}<(C_{2}'r^{7}\mathcal A)^{6r^{3}}\cdot\sqrt{\frac{\log(2e^{2}r^{3}/\delta)}{N}}\]
Also, since:
\[S_{2k-1}(G_{1}^*,G_{2}^*)\leqslant\mu_{0}(w,\omega)<(r+1)^{3}(r\mathcal A)^{2k}\cdot (6k)^{3k}\]
we can estimate the upper bound of $\|\ts^*\|_{2}$:
\[\|\ts^*\|_{2}<\sqrt{K_{r}}\cdot (r+1)^{3}(r\mathcal A)^{2K_{r}}\cdot (6K_{r})^{3K_{r}}<12r^{5}(r\mathcal A)^{4r^{3}}\cdot (12r^{3})^{6r^{3}}<(24r^{6}\mathcal A)^{6r^{3}}\]
According to Equation (\ref{eqn-31}), 
\[\ts=\tp\cdot\left(CE[\alpha^{(1)},\alpha^{(2)},\cdots,\alpha^{(r)}]\right),~\ts^*=\tp^*\cdot\left(CE[\alpha^{*(1)},\alpha^{*(2)},\cdots,\alpha^{*(r)}]\right)\]
Here:
\begin{equation}
\begin{aligned}
&\tp=\Big(\sum_{j=1}^{r}\beta^{(j)}P_{1j}^{3},\cdots,\sum_{j=1}^{r}\beta^{(j)}P_{rj}^{3},~3\sum_{j=1}^{r}\beta^{(j)}P_{1j}^{2}P_{2j},\cdots,3\sum_{j=1}^{r}\beta^{(j)}P_{rj}^{2}P_{r-1,j},\\
&6\sum_{j=1}^{r}\beta^{(j)}P_{1j}P_{2j}P_{3j},\cdots,6\sum_{j=1}^{r}\beta^{(j)}P_{r-2,j}P_{r-1,j}P_{rj},~3\sum_{j=1}^{r}\beta^{(j)}P_{1j}Q_{j}^{2},\cdots, 3\sum_{j=1}^{r}\beta^{(j)}P_{rj}Q_{j}^{2}\Big)\\
&\tp^*=\Big(\sum_{j=1}^{r}\beta^{(j)}P_{1j}^{*3},\cdots,\sum_{j=1}^{r}\beta^{(j)}P_{rj}^{*3},~3\sum_{j=1}^{r}\beta^{(j)}P_{1j}^{*2}P_{2j}^*,\cdots,3\sum_{j=1}^{r}\beta^{(j)}P_{rj}^{*2}P_{r-1,j}^*,\\
&6\sum_{j=1}^{r}\beta^{(j)}P_{1j}^*P_{2j}^*P_{3j}^*,\cdots,6\sum_{j=1}^{r}\beta^{(j)}P_{r-2,j}^*P_{r-1,j}^*P_{rj}^*,~3\sum_{j=1}^{r}\beta^{(j)}P_{1j}^*Q_{j}^{*2},\cdots, 3\sum_{j=1}^{r}\beta^{(j)}P_{rj}^*Q_{j}^{*2}\Big)\\
\end{aligned}
\end{equation}
Next, we will estimate the upper bound of $\|\tp-\tp^*\|_{2}$.
\begin{Lemma}
Denote $T=CE[\alpha^{(1)},\alpha^{(2)},\cdots,\alpha^{(r)}],~ T^*=CE[\alpha^{*(1)},\alpha^{*(2)},\cdots,\alpha^{*(r)}]$. Then we estimate the upper bound of $\|(T^*)^{-1}\|_{F}$.  Assume that the ground truth weights $\alpha^{*(1)},\alpha^{*(2)},\cdots,\alpha^{*(r)}$ are chosen independently from Gaussian distribution with positive integer $M$ and standard deviation $\sigma$, then with probability larger than $1-\delta$ over the choice of $\alpha^{*(1)},\alpha^{*(2)},\cdots,\alpha^{*(r)}$, we have:
\[\|(T^*)^{-1}\|_{F}<\left(\frac{C_{3}\mathcal A^{3}r^{18}}{\delta\sqrt{\sigma}}\right)^{4r^{6}}\]
Here, $C_{3}$ is an absolute constant.
\end{Lemma}
\begin{proof}
Consider the matrix $T^*=CE[\alpha^{*(1)},\alpha^{*(2)},\cdots,\alpha^{*(r)}]\in\mathbb{R}^{K_{r}\times K_{r}}$. According to the definition of Cubic Expectation Matrix.
\begin{equation}
\begin{aligned}
T^*_{ij}&=\mathbb{E}\left(P_{i}(\omega_{1},\cdots,\omega_{r})\cdot(\alpha^{*(1)}\omega_{1}^{3}+\cdots+\alpha^{*(r)}\omega_{r}^{3})^{2j-1}\right)<(r\mathcal A)^{2j-1}(6j-1)!!\\
&< (r\mathcal A)^{2K_{r}}(6K_{r})^{3K_{r}}
\end{aligned}
\end{equation}
The inequality above holds because the expectation of each monomial is smaller than $(6j-1)!!$ and there are in all $r^{2j-1}$ monomials. Therefore, for each element of adjoint matrix of $\tilde{T}^*$. Its element:
\[|\tilde{T}^*_{ij}|=|det(W^*_{ji})|\leqslant (K_{r}-1)!\cdot \left((r\mathcal A)^{2K_{r}}(6K_{r})^{3K_{r}}\right)^{K_{r}-1}<(2r^{3}\cdot(12r^{4}\mathcal A)^{6r^{3}})^{2r^{3}}<(16r^{4}\mathcal A)^{12r^{6}}\]
Here, $W^*_{ji}$ is the $ji$-th algebraic cofactor of $T^*$. On the other hand, since $det(T)$ is an integer-coefficient homogeneous polynomial with $r$ variables $\alpha^{*(1)},\alpha^{*(2)},\cdots,\alpha^{*(r)}$ and degree $1+3+5+\cdots+(2K_{r}-1)=K_{r}^{2}$.

According to the anti-concentration property of Gaussian polynomials introduced by \cite{lovett2010elementary}, assume that the ground truth parameters $\alpha^{*(1)},\alpha^{*(2)},\cdots,\alpha^{*(r)}$ are chosen from a Gaussian distribution with positive integer mean $M$ and standard deviation $\sigma$, then: 
\[Pr\left[|det(T^*)|\leqslant \epsilon\cdot \sqrt{\mathcal C(det(T^*))}\cdot\sigma^{K_{r}^{2}/2}\right]\leqslant C_{3}K_{r}^{2}\cdot\epsilon^{1/K_{r}^{2}}\]
Here, $\mathcal C(det(T^*))$ is the sum of squares of all the monomial coefficients of $CE[M+x_{1},M+x_{2},\cdots,M+x_{r}]$. Since this is a non-zero integer-coefficient polynomial, it's safe to say $C(det(T^*))\geqslant 1$. Therefore, with probability at least $1-\delta$, it holds that:
\[|det(T^*)|> \left(\frac{\delta\sqrt{\sigma}}{C_{3}K_{r}^{2}}\right)^{K_{r}^2}\]
Finally, we can estimate the upper bound of the Frobenius norm of $T^*$. With probability larger than $1-\delta$ over the choice of $\alpha^{*(1)},\alpha^{*(2)},\cdots,\alpha^{*(r)}$, 
\begin{align}
\|(T^*)^{-1}\|_{F}&=\left\|\frac{1}{det(T^*)}\tilde{T}^*\right\|_{F}<K_{r}\left(\frac{C_{3}K_{r}^{2}}{\delta\sqrt{\sigma}}\right)^{K_{r}^2}\cdot (16r^{4}\mathcal A)^{12r^{6}}\notag\\
&<\left(\frac{2C_{3}\cdot 4r^{6}\cdot 16^{3}r^{12}\mathcal A^{3}}{\delta\sqrt{\sigma}}\right)^{4r^{6}}
=\left(\frac{C_{3}'\mathcal A^{3}r^{18}}{\delta\sqrt{\sigma}}\right)^{4r^{6}}\notag
\end{align}
Here, $C_{3},C_{3}'$ are absolute constants. 
\end{proof}
Next, we use the lemma above to estimate the gap $\|\tp-\tp^*\|_{2}$.
\begin{Lemma}
With probability larger than $1-3\delta$ over the choice of $\alpha^{*(1)},\alpha^{*(2)},\cdots,\alpha^{*(r)}$:
\[\|\tp-\tp^*\|_{2}<\left(\frac{Cr^{23}\mathcal A^{6}}{\tau\delta\sqrt{\sigma}}\right)^{8r^6}\cdot\sqrt{\frac{\log(e^{2}(Dr+2r^{3})/\delta)}{N}}+O\left(\frac{1}{N}\right)\triangleq \mu_{2}\]
Here, $C$ is an absolute constant.
\end{Lemma}
\begin{proof}
Denote $E=T-T^*$. According to Theorem 3, we have:
\[|\alpha^{(i)}-\alpha^{*(i)}|<\mu_{1}=O\left(\frac{1}{\sqrt{N}}\right)\]
Therefore,
\begin{align}
E_{ij}&=\mathbb{E}P_{i}\cdot \left((\alpha^{(1)}\omega_{1}^{3}+\alpha^{(2)}\omega_{2}^{3}+\cdots+\alpha^{(r)}\omega_{r}^{3})^{2j-1}-(\alpha^{*(1)}\omega_{1}^{3}+\alpha^{*(2)}\omega_{2}^{3}+\cdots+\alpha^{*(r)}\omega_{r}^{3})^{2j-1}\right)\notag\\
&<\mathbb{E}P_{i}(|\omega_{1}|,|\omega_{2}|,\cdots,|\omega_{r}|)\cdot (2j-1)(\mu_{1}|\omega_{1}|^{3}+\mu_{1}|\omega_{2}|^{3}+\cdots+\mu_{1}|\omega_{r}|^{3})\cdot\notag\\
&~~~~~~(\alpha^{*(1)}|\omega_{1}|^{3}+\cdots+\alpha^{*(r)}|\omega_{r}|^{3})^{2j-2}+O(\mu_{1}^{2})\notag\\
&< \mu_{1}(r\mathcal A)^{2j}(6j-1)!!+O(\mu_{1}^{2})< \mu_{1}(r\mathcal A)^{2K_{r}}\cdot(6K_{r})^{3K_{r}}+O(\mu_{1}^{2})
\end{align}
with probability larger than $1-\delta$, 
\[\mu_{1}=\left(\frac{C_{1}(6r)^{3r}\cdot\mathcal A^{r+1}}{\tau}\right)^{3r}\cdot \sqrt{\frac{\log(Dre^{2}/\delta)}{N}}\]
Therefore:
\[\|E\|_{F}<K_{r}\cdot \mu_{1}(r\mathcal A)^{2K_{r}}\cdot(6K_{r})^{3K_{r}}+O(\mu_{1}^{2})<\left(\frac{C_{4}r^{4}\mathcal A^{2}}{\tau}\right)^{6r^{3}}\cdot\sqrt{\frac{\log(Dre^{2}/\delta)}{N}}+O\left(\frac{1}{N}\right)\]
According to Theorem \ref{pertub_inverse}: with probability at least $1-2\delta$, 
\begin{equation}
\begin{aligned}
\|T^{-1}-(T^*)^{-1}\|_{2}&\leqslant\frac{\|E\|_{2}\|(T^*)^{-1}\|_{2}^{2}}{1-\|(T^*)^{-1}E\|_{2}}= \|E\|_{2}\|(T^*)^{-1}\|_{2}^{2}+O\left(\frac{1}{N}\right)\\
&\leqslant\|E\|_{F}\|(T^*)^{-1}\|_{F}^{2}+O\left(\frac{1}{N}\right)\\
&< \left(\frac{C_{5}r^{19}A^{5}}{\tau\delta\sqrt{\sigma}}\right)^{8r^{6}}\cdot\sqrt{\frac{\log(Dre^{2}/\delta)}{N}}+O\left(\frac{1}{N}\right)
\end{aligned}
\end{equation}
Combining the equation above with Lemma \ref{lemma-14}, with probability larger than $1-3\delta$:
\begin{equation*}
\begin{aligned}
&\|\tp-\tp^*\|_{2}=\|\ts\cdot T^{-1}-\ts^*\cdot (T^*)^{-1}\|_{2}<\|\ts-\ts^*\|_{2}\cdot\|(T^*)^{-1}\|_{F}+\|T^{-1}-(T^*)^{-1}\|_{2}\cdot\|\ts^*\|_{2}+O\left(\frac{1}{N}\right)\\
<& \left(\frac{C_{5}r^{19}\mathcal A^{5}}{\tau\delta\sqrt{\sigma}}\right)^{8r^6}(24r^{6}\mathcal A)^{6r^3}\sqrt{\frac{\log(Dre^{2}/\delta)}{N}} +\left(\frac{C_{3}'r^{18}\mathcal A^{3}}{\delta\sqrt{\sigma}}\right)^{4r^6}(C_{2}'r^{7}\mathcal A)^{6r^3}\sqrt{\frac{\log(2e^{2}r^{3}/\delta)}{N}}+O\left(\frac{1}{N}\right)\\
<& \left(\frac{Cr^{23}\mathcal A^{6}}{\tau\delta\sqrt{\sigma}}\right)^{8r^6}\cdot\sqrt{\frac{\log(e^{2}(Dr+2r^{3})/\delta)}{N}}+O\left(\frac{1}{N}\right)\triangleq \mu_{2}
\end{aligned}
\end{equation*}
Here, $C$ is an absolute constant.
\end{proof}
Since the form of $\tp,\tp^*$ is equivalent to tensor decomposition. Denote:
\begin{equation}
\begin{aligned}
L &= \beta^{(1)}\left(\begin{matrix}P_{11}\\P_{21}\\\vdots\\P_{r1}\end{matrix}\right)^{\otimes 3}+\beta^{(2)}\left(\begin{matrix}P_{12}\\P_{22}\\\vdots\\P_{r2}\end{matrix}\right)^{\otimes 3}+\cdots+\beta^{(r)}\left(\begin{matrix}P_{1r}\\P_{2r}\\\vdots\\P_{rr}\end{matrix}\right)^{\otimes 3}\\
&= (u^{(1)})^{\otimes 3} + (u^{(2)})^{\otimes 3} + \cdots+(u^{(r)})^{\otimes 3}\\
L^* &= \beta^{*(1)}\left(\begin{matrix}P_{11}^*\\P_{21}^*\\\vdots\\P_{r1}^*\end{matrix}\right)^{\otimes 3}+\beta^{*(2)}\left(\begin{matrix}P_{12}^*\\P_{22}^*\\\vdots\\P_{r2}^*\end{matrix}\right)^{\otimes 3}+\cdots+\beta^{*(r)}\left(\begin{matrix}P_{1r}^*\\P_{2r}^*\\\vdots\\P_{rr}^*\end{matrix}\right)^{\otimes 3}\\
&= (u^{*(1)})^{\otimes 3} + (u^{*(2)})^{\otimes 3} + \cdots+(u^{*(r)})^{\otimes 3}
\end{aligned}
\end{equation}
Then $\|L-L^*\|_{F}\leqslant \|\tp-\tp^*\|_{2}<\mu_{2}$. Here, $u^{(i)}=\sqrt[3]{\beta^{(i)}}(P_{1i},P_{2i},\cdots,P_{ri})^{T}, u^{*(i)}=\sqrt[3]{\beta^{*(i)}}(P_{1i}^*,P_{2i}^*,\cdots,P_{ri}^*)^{T}$. In the next lemma, we will use the uniqueness and perturbation property of tensor decomposition to estimate the solution gap $\sum_{i=1}^{r}\|u^{(i)}-u^{*(i)}\|_{2}$. According to the Jenrich's Algorithm of tensor decomposition, firstly we find a matrix $G\in\mathbb{R}^{r\times r}$ such that: $G^{T}u^{*(1)}, G^{T}u^{*(2)},\cdots,G^{T}u^{*(r)}$ are orthonormal vectors. Denote $U=(u^{(1)},u^{(2)},\cdots,u^{(r)})$ and $U^*=(u^{*(1)},u^{*(2)},\cdots,u^{*(r)})$. Then $G^{T}U^*$ is an orthogonal matrix and we conduct the whitening process, using the same $G$ to the whitening. 
\[\mathcal U\triangleq L[G,G,G]=\sum_{i=1}^{r}(G^{T}u^{(i)})^{\otimes 3},~\mathcal U^*\triangleq L^*[G,G,G]=\sum_{i=1}^{r}(G^{T}u^{*(i)})^{\otimes 3}\]
Then we can properly sample an $r$-length unit vector $x$ and calculate the weighted sum of $r$ slices of sub-matrices of $\mathcal U$ and $\mathcal U^*$. Denote $v^{(i)}=G^{T}u^{(i)}, v^{*(i)}=G^{T}u^{*(i)}$.
\begin{align}
&\mathcal U_{x}=x_{1}\mathcal U[1,:,:]+x_{2}\mathcal U[2,:,:]+\cdots+x_{r}\mathcal U[r,:,:]=\sum_{i=1}^{r}(v^{(i)}\cdot x)\cdot (v^{(i)})^{\otimes 2}=VDV^{T}\notag\\
&\mathcal U^*_{x}=x_{1}\mathcal U^*[1,:,:]+x_{2}\mathcal U^*[2,:,:]+\cdots+x_{r}\mathcal U^*[r,:,:]=\sum_{i=1}^{r}(v^{*(i)}\cdot x)\cdot (v^{*(i)})^{\otimes 2}=V^*D^*(V^*)^{T}\notag
\end{align}
Notice that $V^*$ is the orthogonal eigenvector matrix of $\mathcal U_{x}^*$. We calculate the orthogonal eigenvector matrix $V,V^*$ of both $\mathcal U_{x}$ and $\mathcal U_{x}^*$, and we can give an upper bound of their difference with theoretical guarantee, which leads us to the stability of $U^*$. Firstly, we analyze the stability property of $V,V^*$ by using Lemma \ref{eigen-vector}.
\begin{Lemma}
Given the whitening matrix $G$, we can estimate the upper bound of the difference of $V$ and $V^*$.
\[\|V-V^*\|_{F}\leqslant r^{2}\mu_{2}\cdot\|G\|_{F}^{3}+O\left(\frac{1}{N}\right)\]
\end{Lemma}
\begin{proof}
Notice that $V^*$ is an orthogonal matrix, which means $v^{*(1)},v^{*(2)},\cdots,v^{*(r)}$ are orthonormal vectors. We can properly choose a unit vector $x$:
\[x=\alpha v^{*(1)}+(2\alpha)v^{*(2)}+\cdots+(r\alpha)v^{*(r)}\]
where $\alpha=1/\sqrt{1^{2}+2^{2}+\cdots+r^{2}}=\sqrt{6/r(r+1)(2r+1)}$. Then the diagonal matrix $D^*=diag(\alpha,2\alpha,\cdots,r\alpha)$. According to Lemma \ref{eigen-vector}:
\[\|v^{(i)}-v^{*(i)}\|_{2}\leqslant \frac{\|\mathcal U_{x}-\mathcal U_{x}^*\|_{2}}{\alpha}+O(\|\mathcal U_{x}-\mathcal U_{x}^*\|_{2}^{2})\]
Then we estimate the upper bound of $\|\mathcal U_{x}-\mathcal U_{x}^*\|_{2}$:
\begin{align}
&\|\mathcal U-\mathcal U^*\|_{F}^{2} = \|L[G,G,G]-L^*[G,G,G]\|_{F}^{2} =\sum_{i,j,k=1}^{r}\left(\sum_{a,b,c=1}^{r}(L_{abc}-L_{abc}^*)G_{ia}G_{jb}G_{kc}\right)^{2}\notag\\
\leqslant& \sum_{i,j,k=1}^{r}\left(\sum_{a,b,c=1}^{r}(L_{abc}-L_{abc}^*)^{2}\right)\cdot\left(\sum_{a,b,c=1}^{r}(G_{ia}G_{jb}G_{kc})^{2}\right)=\|L-L^*\|_{F}^{2}\cdot\|G\|_{F}^{6}\notag\\
\Rightarrow& \|\mathcal U-\mathcal U^*\|_{F}\leqslant \|L-L^*\|_{F}\cdot\|G\|_{F}^{3}<\mu_{2}\|G\|_{F}^{3}\notag\\
\Rightarrow& \|\mathcal U_{x}-\mathcal U_{x}^*\|_{F}\leqslant \sum_{i=1}^{r}|x_{i}|\cdot\|\mathcal U[i,:,:]-\mathcal U^*[i,:,:]\|_{F}\leqslant  \sqrt{\sum_{i=1}^{r}|x_{i}|^{2}\cdot\sum_{i=1}^{r}\|\mathcal U[i,:,:]-\mathcal U^*[i,:,:]\|_{F}^{2}}\notag\\
&~~~~~~~~~~~~~~~~= \|\mathcal U-\mathcal U^*\|_{F} <\mu_{2}\|G\|_{F}^{3}\notag
\end{align}
Therefore:
\[\|V-V^*\|_{F}\leqslant \sqrt{r}\left(\frac{\mu_{2}\|G\|_{F}^{3}}{\alpha}\right)+O\left(\frac{1}{N}\right)<r^{2}\mu_{2}\cdot\|G\|_{F}^{3}+O\left(\frac{1}{N}\right)\]
\end{proof}
After we estimate the gap between $V$ and $V^*$, there is also one step away from estimating the gap between $U$ and $U^*$, which is estimating the upper bound of $\|G\|_{F}$ and $\|G^{-1}\|_{2}$. The whitening process above shows how to get this $G$. We firstly properly choose a unit vector $y$ and calculate the weighted sum.
\[L_{y}^*=y_{1}L^*[1,:,:]+y_{2}L^*[2,:,:]+\cdots+y_{r}L^*[r,:,:]=U^*EU^{*T}\]
Here, $E=diag(\langle u^{(i)},y\rangle:i\in [r])$. Since $L_{y}^*$ is a symmetrical matrix, we can conduct the orthogonal decomposition as: $L_{y}^*=QFQ^{T}$, and our whitening matrix $G=QF^{-1/2}$. Then we estimate the upper bound of $\|G^{-1}\|_{2}$.
\begin{Lemma}
For the whitening matrix $G$ mentioned above, we can estimate the upper bound of its 2-norm.
\begin{align}
\|G^{-1}\|_{2}\leqslant\sqrt{r\mathcal A}
\end{align}
\end{Lemma}
\begin{proof}
In fact, it's not difficult to notice that:
\begin{align}
\|G^{-1}\|_{2}&= \|F^{1/2}\|_{2}\leqslant \sqrt{\|F\|_{2}} =\sqrt{\|L_{y}^*\|_{2}}\leqslant \sqrt{\|L_{y}^*\|_{F}}\leqslant \sqrt{\sum_{i=1}^{r}|y_{i}|\cdot\|L^*[i,:,:]\|_{F}}\notag\\
&\leqslant \sqrt[4]{\sum_{i=1}^{r}|y_{i}|^{2}\cdot\sum_{i=1}^{r}\|L^*[i,:,:]\|_{F}^{2}}=\sqrt{\|L^*\|_{F}}\leqslant \sqrt{\sum_{i=1}^{r}|\beta^{*(i)}|\cdot\|(P_{1i}^*,P_{2i}^*,\cdots,P_{ri}^*)^{T}\|_{2}^{3}}\notag\\
&\leqslant \sqrt{r\mathcal A}
\end{align}
Here, we use the fact that: $P_{1i}^{*2}+P_{2i}^{*2}+\cdots+P_{ri}^{*2}\leqslant 1$.
\end{proof}
On the other hand, we will estimate the upper bound of $\|G\|_{F}$. 
\begin{Lemma}
With probability larger than $1-2\delta$ over the choice of $\tv^{*(i)}, \tw^{*(i)}~i\in[r]$, we have:
\[\|G\|_{F}\leqslant \frac{2r^{2}\sqrt{d}}{\delta^{5/2}\cdot\sqrt[6]{\tau}}\]
\end{Lemma}
\begin{proof}
It's obvious that: $\|G\|_{F}\leqslant \sqrt{r}\cdot\|G\|_{2}=\sqrt{r}\cdot\|F^{-1/2}\|_{2}=\frac{\sqrt{r}}{\sqrt{\lambda_{\min}(L_{y}^*)}}$. Here, $\lambda_{\min}(X)$ is the absolute value of the eigenvalue of $X$ which is closest to 0. Next,
\[\lambda_{\min}(L_{y}^*)=\lambda_{\min}(EU^{*T}U^*)\geqslant \lambda_{\min}(E)\cdot\sigma_{\min}^{2}(U^*)\]
Since $U^*\in\mathbb{R}^{r\times r}$ has the following form:
\[U^*_{i}=\sqrt[3]{\beta^{*(i)}}(P_{1i}^*,P_{2i}^*,\cdots,P_{ri}^*)^{T}=\sqrt[3]{\beta^{*(i)}}(\tv^{*(1)}\cdot\tw^{*(i)},\tv^{*(2)}\cdot\tw^{*(i)},\cdots,\tv^{*(r)}\cdot\tw^{*(i)})^{T}=\sqrt[3]{\beta^{*(i)}}A^{*T}B^*_{i}\]
Here, $A^*=(\tv^{*(1)},\tv^{*(2)},\cdots,\tv^{*(r)})\in\mathbb{R}^{d\times r},~ B^*=(\tw^{*(1)},\tw^{*(2)},\cdots,\tw^{*(r)})\in\mathbb{R}^{d\times r}$. Therefore,  $U^*=A^{*T}B^*$. According to the conditions of $\tv^{*(i)}, \tw^{*(i)}$, matrices $A^*, B^*$ can be treated as the first $r$ columns of two random $d\times d$ orthogonal matrices (with regard to Haar Measure). Firstly, we analyze the lower bound of $\lambda_{\min}(E)$. According to the rotational symmetrical property and Lemma \ref{inner_product}, we can estimate the lower bound of $\|U_{i}^*\|_{2}$ with high probability.
\begin{align}
P(\|U_{i}^*\|_{2}\leqslant t)&\leqslant P\left(|\tv^{*(1)}\cdot\tw^{*(i)}|\leqslant \frac{t}{\sqrt[3]{\tau}}\right) \leqslant \frac{\pi}{2}d\cdot\frac{t}{\sqrt[3]{\tau}}<\frac{2dt}{\sqrt[3]{\tau}}
\end{align}
Therefore, with probability at least $1-\delta$, it holds that:
$\|U_{i}^*\|_{2}\geqslant \frac{\delta\sqrt[3]{\tau}}{2d}$. After combining all these $i=1,2,\cdots,r$ together, we can conclude that, with probability at least $1-\delta$, we have:
\[\|U_{i}^*\|_{2}\geqslant \frac{\delta\sqrt[3]{\tau}}{2rd}~~(i=1,2,\cdots,r)\]
Also, by using Lemma \ref{inner_product}, we can properly choose a unit vector $y$, such that:
\[|\hat{U}_{i}^*\cdot y|\geqslant \frac{2}{\pi r^{2}}>\frac{1}{2r^{2}} \Rightarrow |U_{i}^*\cdot y|\geqslant \frac{\delta\sqrt[3]{\tau}}{4r^{3}d}\Rightarrow \lambda_{\min}(E)\geqslant \frac{\delta\sqrt[3]{\tau}}{4r^{3}d}\]
Then, we estimate the lower bound of $\sigma_{\min}(U^*)\geqslant\lambda_{\min}(U^*)$. Since matrix $U^*$ can be treated as the top left $r\times r$ submatrix of a random orthogonal matrix (under Haar Measure), we use an important equation about the relationship between lower bound of the smallest eigenvalue and distance between subspaces.
\[\sum_{i=1}^{r}\lambda_{i}(U^*)^{-2}=\sum_{i=1}^{r}D_{i}^{-2}\]
Here, $D_{i}$ denotes the distance between the $i$-th column vector of $U^*$ and the subspace spanned by other column vectors. We analyze $D_{1} $ first, which is the distance between $\alpha=U^*_{1}$ and $r-1$ dimensional hyperplane $\Gamma = span(U^*_{2},U^*_{3},\cdots,U^*_{r})$. Since $U^*_{1},\cdots,U^*_{r}$ can be treated as the first $r$ elements of $r$ orthonormal vectors $\tilde{U}^*_{1},\cdots,\tilde{U}^*_{r}$. Therefore, for any given hyperplane $\Gamma$ in $\mathbb{R}^{r}$, $\alpha$ can be the first $r$ elements of any random unit vector in the complementary space in $span(\tilde{U}^*_{2},\cdots,\tilde{U}^*_{r})^{\perp}$. Then we can estimate the lower bound of $D_{1}$ with high probability.
\begin{align}
P(D_{1}<\epsilon)&\leqslant P(\|\alpha\|_{2}<\sqrt{\epsilon})+P(|\langle\hat{\alpha},\Gamma^{\perp} \rangle|<\sqrt{\epsilon})\notag\\
&\leqslant P(\|\tilde{U}^*_{1}, \mathbf{n}\|_{2}<\sqrt{\epsilon}) + P(|\langle\hat{\alpha},\Gamma^{\perp} \rangle|<\sqrt{\epsilon})\notag\\
&< \frac{\pi}{2}d\sqrt{\epsilon} + \frac{\pi}{2}d\sqrt{\epsilon}=\pi d\sqrt{\epsilon}\notag
\end{align}
Here, $\mathbf{n}=(1/\sqrt{r},\cdots,1/\sqrt{r},0,\cdots,0)^{T}$ is a unit vector. Therefore,
\[P(\lambda_{\min}(U^*)\leqslant t)\leqslant P(\sum_{\lambda\in\lambda(U^*)}\lambda^{-2}\geqslant t^{-2})=P(\sum_{i=1}^{r}D_{i}^{-2}\geqslant t^{-2})\geqslant rP(D_{1}^{-2}\geqslant 1/rt^{2})=rP(D_{1}\leqslant \sqrt{r}t)\]
Let $\sqrt{r}t=\epsilon$, we know that:
\[P(\lambda_{\min}(U^*)\leqslant \frac{\epsilon}{\sqrt{r}})\leqslant rP(D_{1}< \epsilon)=\pi rd\sqrt{\epsilon}\]
In other words, with probability greater than $1-\delta$, it holds that:
\[\lambda{\min}(U^*)\geqslant \frac{\delta^{2}}{r^{5/2}d^{2}}\]
After summing up all the inequalities above, we can conclude that, with probability larger than $1-2\delta$ over the choice of $\tv^{*(i)}, \tw^{*(i)}~i\in[r]$,
\[\|G\|_{F}\leqslant \frac{\sqrt{r}}{\sqrt{\frac{\delta\sqrt[3]{\tau}}{4r^{3}d}}\cdot\frac{\delta^{2}}{r^{5/2}d^{2}}}=\frac{2r^{2}\sqrt{d}}{\delta^{5/2}\cdot\sqrt[6]{\tau}}\]
\end{proof}
Finally, we are able to estimate the gap between $U$ and $U^*$ and furthermore, the gap between $P_{ij}$ and $P_{ij}^*$ and finish the proof of Theorem 4.
\begin{Lemma}
For $\forall i,j\in[r]$, it holds that:
\[|P_{ij}-P_{ij}^*|<\left(\frac{C'r^{23}\mathcal A^{6}}{\tau\delta\sqrt{\sigma}}\right)^{9r^6}\cdot d^{8}\cdot\sqrt{\frac{\log(e^{2}(Dr+2r^{3})/\delta)}{N}}+O\left(\frac{1}{N}\right)\]
where $C'$ is an absolute constant.
\end{Lemma}
\begin{proof}
Since $V=G^{T}U\Rightarrow U=G^{-T}V$. It holds that:
\begin{align}
\|U-U^*\|_{2}&\leqslant \|G^{-1}\|_{2}\cdot\|V-V^*\|_{2} \leqslant \sqrt{r\mathcal A}\cdot r^{2}\mu_{2}\cdot\frac{8r^{27/2}d^{15/2}}{\delta^{21/2}\sqrt{\tau}}+O\left(\frac{1}{N}\right)\notag\\
&=\left(\frac{Cr^{23}\mathcal A^{6}}{\tau\delta\sqrt{\sigma}}\right)^{8r^6}\cdot\sqrt{\frac{\log(e^{2}(Dr+2r^{3})/\delta)}{N}}\cdot\frac{8r^{14}d^{8}}{\delta^{11}}\sqrt{\frac{\mathcal A}{\tau}}+O\left(\frac{1}{N}\right)\notag
\end{align}
Therefore, for $\forall i,j\in[r]$, it holds that:
\[|P_{ij}-P_{ij}^*|<\frac{1}{\sqrt[3]{\tau}}\|U-U^*\|_{2}<\left(\frac{C'r^{23}\mathcal A^{6}}{\tau\delta\sqrt{\sigma}}\right)^{9r^6}\cdot d^{8}\sqrt{\frac{\log(e^{2}(Dr+2r^{3})/\delta)}{N}}+O\left(\frac{1}{N}\right)\]
where $C'$ is an absolute constant. Till now, we've finished the proof of Theorem 4.
\end{proof}
\end{Proof}

%% file: proof5_intro.tex
\begin{Proof}
Denote:
\begin{align}
\mathcal V&=(\tv_{1}^{(1)},\tv_{1}^{(2)},\cdots,\tv_{1}^{(r)},\tv_{2}^{(1)},\cdots,\tv_{D}^{(r)})^{T}\in\mathbb{R}^{Dr\times d}\notag\\
\mathcal V^*&=(\tv_{1}^{*(1)},\tv_{1}^{*(2)},\cdots,\tv_{1}^{*(r)},\tv_{2}^{*(1)},\cdots,\tv_{D}^{*(r)})^{T}\in\mathbb{R}^{Dr\times d}\notag
\end{align}
Then, according to Theorem 4, with probability larger than $1-\frac{5D^{2}}{2}\delta$, it holds that:
\[\|\mathcal V\mathcal V^{T}-\mathcal V^*\mathcal V^{*T}\|_{\infty,\infty}\leqslant \mu_{3}\]
By Lemma \ref{orthogonal_distance}, there exists orthogonal matrices $G_{1}, G_{2}\in\mathbb{R}^{d\times d}$, such that:
\[\|\mathcal VG_{1}-\mathcal V^*G_{2}\|_{F}\leqslant 
\sqrt{Ddr\mu_{3}}+O(\mu_{3})\]
Denote $(\tv_{k}^{(i)})^{T}G_{1}=\tw_{k}^{(i)}, (\tv_{k}^{*(i)})^{T}G_{2}=\tw_{k}^{*(i)}$. Finally, we are able to estimate the Wasserstein distance. 
\begin{align}
&W_{1}(G_{\#}\mathcal N(0,I), G^*_{\#}\mathcal N(0,I))=\sup_{Lip(D)\leqslant 1}\left[\mathop\mathbb{E}_{\omega\sim\mathcal N(0,I)}D(G(\omega))-\mathop\mathbb{E}_{\omega\sim\mathcal N(0,I)}D(G^*(\omega))\right]\notag\\
=& \sup_{Lip(D)\leqslant 1}\left[\mathop\mathbb{E}_{\omega\sim\mathcal N(0,I)}D(G(G_{1}^{T}\omega))-\mathop\mathbb{E}_{\omega\sim\mathcal N(0,I)}D(G^*(G_{2}^{T}\omega))\right]\notag\\
\leqslant & \mathbb{E}_{\omega}\|G(G_{1}^T\omega)-G^*(G_{2}^T\omega)\|_{2} \leqslant \mathbb{E}_{\omega}\left(\sum_{k=1}^{D}\left(\sum_{i=1}^{r}\alpha_{k}^{(i)}(\tw_{k}^{(i)}\cdot\omega)^{3}-\sum_{i=1}^{r}\alpha_{k}^{*(i)}(\tw_{k}^{*(i)}\cdot\omega)^{3}\right)^{2}\right)^{1/2}\notag\\
\leqslant & \left(\sum_{k=1}^{D}\mathbb{E}_{\omega}\left(\sum_{i=1}^{r}\alpha_{k}^{(i)}(\tw_{k}^{(i)}\cdot\omega)^{3}-\sum_{i=1}^{r}\alpha_{k}^{*(i)}(\tw_{k}^{*(i)}\cdot\omega)^{3}\right)^{2}\right)^{1/2}\notag
\end{align}
Since:
\begin{align}
&\mathbb{E}_{\omega}\left(\sum_{i=1}^{r}\alpha_{k}^{(i)}(\tw_{k}^{(i)}\cdot\omega)^{3}-\sum_{i=1}^{r}\alpha_{k}^{*(i)}(\tw_{k}^{*(i)}\cdot\omega)^{3}\right)^{2}\notag\\
\leqslant& (2r)\left(\sum_{i=1}^{r}\mathbb{E}_{\omega}\left(\alpha_{k}^{(i)}(\tw_{k}^{(i)}\cdot\omega)^{3}-\alpha_{k}^{*(i)}(\tw_{k}^{(i)}\cdot\omega)^{3}\right)^{2}+\sum_{i=1}^{r}\mathbb{E}_{\omega}\left(\alpha_{k}^{*(i)}(\tw_{k}^{(i)}\cdot\omega)^{3}-\alpha_{k}^{*(i)}(\tw_{k}^{*(i)}\cdot\omega)^{3}\right)^{2}\right)\notag\\
\leqslant& (2r)\left(15\sum_{i=1}^{r}|\alpha_{k}^{(i)}-\alpha_{k}^{*(i)}|^{2}+\mathcal A^{2}\sum_{i=1}^{r}\mathbb{E}_{\omega}\left(3((\tw_{k}^{(i)}-\tw_{k}^{*(i)})\cdot\omega)(\tw_{k}^{*(i)}\cdot\omega)^{2}\right)^{2}+O\left(\frac{1}{N^{3/4}}\right)\right)\notag\\
\leqslant& 2r\cdot 9\mathcal A^{2}\sum_{i=1}^{r}\sqrt{\mathbb{E}((\tw_{k}^{(i)}-\tw_{k}^{*(i)})\cdot\omega)^{4}\cdot\mathbb{E}(\tw_{k}^{*(i)}\cdot\omega)^{8}}+O\left(\frac{1}{N^{3/4}}\right)\notag\\
=&18r\mathcal A^{2}\sum_{i=1}^{r}\sqrt{105\cdot 15\|\tw_{k}^{(i)}-\tw_{k}^{*(i)}\|_{2}^{4}}+O\left(\frac{1}{N^{3/4}}\right)<720r\mathcal A^{2}\sum_{i=1}^{r}\|\tw_{k}^{(i)}-\tw_{k}^{*(i)}\|_{2}^{2}+O\left(\frac{1}{N^{3/4}}\right)\notag
\end{align}
Therefore,
\begin{align}
&W_{1}(G_{\#}\mathcal N(0,I), G^*_{\#}\mathcal N(0,I))<\left(720r\mathcal A^{2}\sum_{k=1}^{D}\sum_{i=1}^{r}\|\tw_{k}^{(i)}-\tw_{k}^{*(i)}\|_{2}^{2}+O\left(\frac{1}{N^{3/4}}\right)\right)^{1/2}\notag\\
=& \left(720r\mathcal A^{2}\cdot\|\mathcal VG_{1}-\mathcal V^*G_{2}\|_{F}^{2}+O\left(\frac{1}{N^{3/4}}\right)\right)^{1/2}
<30\sqrt{r}\mathcal A\cdot\sqrt{Ddr\mu_{3}}+o(\sqrt{\mu_{3}})\notag\\
<& \sqrt{D}\cdot d^{4}\left(\frac{Cr^{23}\mathcal A^{6}}{\tau\delta\sqrt{\sigma}}\right)^{5r^6}\sqrt[4]{\frac{\log(e^{2}(Dr+2r^{3})/\delta)}{N}}+o\left(\frac{1}{N^{1/4}}\right)
\end{align}
Till now, the main theorem has been proved.
\end{Proof}

%% file: algorithm.tex
In this section, we summarize the sections above, and introduce an efficient algorithm (Algorithm 1) to determine the learner generator $G$ only by moment analysis.

\begin{figure}[!ht]
\centering
\includegraphics[width=1.1\textwidth]{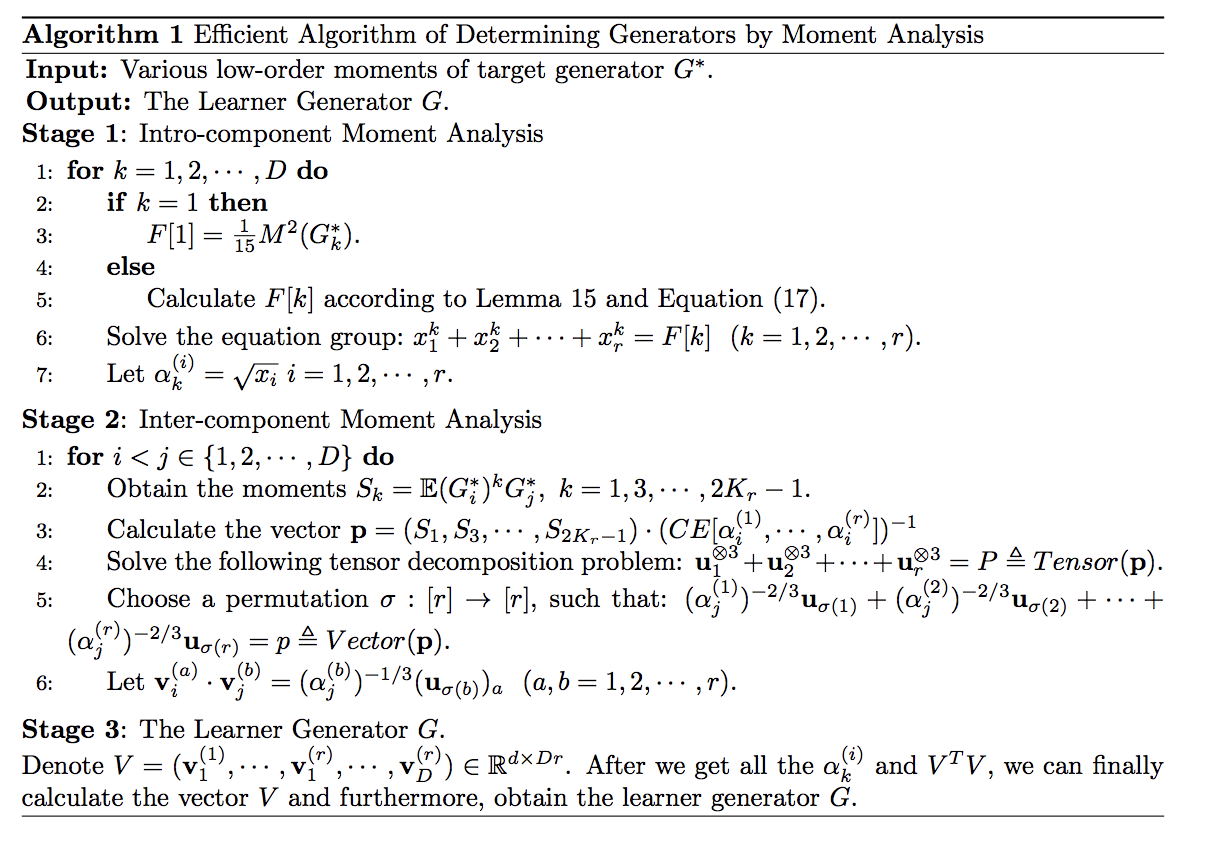}
\end{figure}

Here, as we expressed in Proof of Theorem 2:
\begin{equation}
\begin{aligned}
&\tp=\Big(\sum_{j=1}^{r}\beta^{(j)}P_{1j}^{*3},\cdots,\sum_{j=1}^{r}\beta^{(j)}P_{rj}^{*3},~3\sum_{j=1}^{r}\beta^{(j)}P_{1j}^{*2}P_{2j}^*,\cdots,3\sum_{j=1}^{r}\beta^{(j)}P_{rj}^{*2}P_{r-1,j}^*,\\
&6\sum_{j=1}^{r}\beta^{(j)}P_{1j}^*P_{2j}^*P_{3j}^*,\cdots,6\sum_{j=1}^{r}\beta^{(j)}P_{r-2,j}^*P_{r-1,j}^*P_{rj}^*,~3\sum_{j=1}^{r}\beta^{(j)}P_{1j}^*Q_{j}^{*2},\cdots, 3\sum_{j=1}^{r}\beta^{(j)}P_{rj}^*Q_{j}^{*2}\Big)
\end{aligned}
\end{equation}
$P\triangleq Tensor(\tp)\in\mathbb{R}^{r\times r\times r}$ is the tensor with elements: 
\[P_{abc}=\sum_{i=1}^{r}\beta^{(j)}P_{aj}^*P_{bj}^*P_{cj}^*~~(a,b,c,=1,2,\cdots,r)\]
$p\triangleq Vector(\tp)\in\mathbb{R}^{r}$ denotes the vector with elements:
\[p_{k}=\sum_{j=1}^{r}\beta^{(j)}P_{kj}^*~~(k=1,2,\cdots,r)\]